\documentclass[12pt]{article}

\usepackage[colorlinks, linkcolor=blue, citecolor=blue]{hyperref}           
\usepackage{color}
\usepackage{graphicx,subfigure,amsmath,amssymb,amsfonts,bm,epsfig,epsf,url,dsfont}
\usepackage{amsthm,mathrsfs}

\newtheorem{thm}{Theorem}[section]
\newtheorem{lem}[thm]{Lemma}
\newtheorem{assum}[thm]{Assumption}
\newtheorem{proposition}[thm]{Proposition}

\newtheorem{definition}[thm]{Definition}

\usepackage{tikz}
\usepackage{bbm}
\usepackage[small,bf]{caption}
\usepackage[top=1in,bottom=1in,left=1in,right=1in]{geometry}
\usepackage{fancybox}
\usepackage{hyperref}
\usepackage{algorithm}
\usepackage{verbatim}
\usepackage[titletoc,toc]{appendix}

\usepackage{mathtools}

\usepackage{scalerel,stackengine}
\stackMath
\newcommand\reallywidehat[1]{%
\savestack{\tmpbox}{\stretchto{%
  \scaleto{%
    \scalerel*[\widthof{\ensuremath{#1}}]{\kern-.6pt\bigwedge\kern-.6pt}%
    {\rule[-\textheight/2]{1ex}{\textheight}}
  }{\textheight}%
}{0.5ex}}%
\stackon[1pt]{#1}{\tmpbox}%
}

\newcommand*{\rom}[1]{\expandafter\@slowromancap\romannumeral #1@}

\newcommand{\abs}[1]{\left|#1\right|}

\DeclareMathOperator*{\argmax}{arg\,max}
\DeclareMathOperator*{\argmin}{arg\,min}

\newcommand {\pr} {\mathbb{P}}

\newcommand{\calA}{{\cal A}}

\newcommand{\calD}{{\cal D}}

\newcommand{\calL}{{\cal L}}
\newcommand{\calM}{{\cal M}}
\newcommand{\calN}{{\cal N}}

\newcommand{\calP}{{\cal P}}

\newcommand{\calV}{{\cal V}}
\newcommand{\calW}{{\cal W}}
\newcommand{\calX}{{\cal X}}

\newcommand{\be}{\begin{equation}}
\newcommand{\ee}{\end{equation}}
\newcommand{\beqna}{\begin{eqnarray}}
\newcommand{\eeqna}{\end{eqnarray}}

\newcommand{\p}[1]{\left(#1\right)}
\newcommand{\pp}[1]{\left[#1\right]}
\newcommand{\ppp}[1]{\left\{#1\right\}}
\newcommand{\norm}[1]{\left\|#1\right\|}

\newcommand{\s}[1]{\mathsf{#1}}
\usepackage{xspace}

\numberwithin{equation}{section}

\makeatletter
\renewcommand{\paragraph}{%
  \@startsection{paragraph}{4}%
  {\z@}{1.25ex \@plus 1ex \@minus .2ex}{-1em}%
  {\normalfont\normalsize\bfseries}%
}
\makeatother

\allowdisplaybreaks

\usepackage{authblk}

\begin{document}

\title{Confirmation Bias in Gaussian Mixture Models}

\author[1]{Amnon Balanov \thanks{Corresponding author: \url{amnonba15@gmail.com}}}
\author[1]{Tamir Bendory}
\author[1]{Wasim Huleihel}
\affil[1]{School of Electrical and Computer Engineering, Tel Aviv University}

\maketitle

\begin{abstract}

Confirmation bias, the tendency to interpret information in a way that aligns with one's preconceptions, can profoundly impact scientific research, leading to conclusions that reflect the researcher's hypotheses even when the observational data do not support them. 
This issue is especially critical in scientific fields involving highly noisy observations, such as cryo-electron microscopy.

This study investigates confirmation bias in Gaussian mixture models.
We consider the following experiment: A team of scientists assumes they are analyzing data drawn from a Gaussian mixture model with known signals (hypotheses) as centroids. However, in reality, the observations consist entirely of noise without any informative structure. The researchers use a single iteration of the $K$-means or expectation-maximization algorithms, two popular algorithms to estimate the centroids. 
Despite the observations being pure noise, we show that these algorithms yield biased estimates that resemble the initial hypotheses, contradicting the unbiased expectation that averaging these noise observations would converge to zero.
Namely, the algorithms generate estimates that mirror the postulated model, although the hypotheses (the presumed centroids of the Gaussian mixture) are not evident in the observations.

Specifically, among other results, we prove a positive correlation between the estimates produced by the algorithms and the corresponding hypotheses. We also derive explicit closed-form expressions of the estimates for a finite and infinite number of hypotheses. Furthermore, we provide theoretical and empirical results for multi-iteration $K$-means and expectation-maximization, showing that the bias is persistent even after hundreds of iterations of these algorithms.
This study underscores the risks of confirmation bias in low signal-to-noise environments, provides insights into potential pitfalls in scientific methodologies, and highlights the importance of prudent data interpretation.

\end{abstract}

\section{Introduction}
Confirmation bias refers to the cognitive tendency to interpret information that aligns with our beliefs or presumptions, disregarding evidence that contradicts these beliefs \cite{klayman1995varieties, kassin2013forensic}. This bias can distort perceptions and lead to flawed decision-making.
Examples of confirmation bias are common in both everyday life and scientific practice. In medical diagnosis, for example, a doctor might diagnose a patient based on an initial impression and subsequently give more weight to symptoms that confirm this diagnosis while overlooking contradictory evidence~\cite{khadilkar2020bias}. In legal settings, confirmation bias might influence how evidence is interpreted, with investigators or jurors giving excessive credibility to information that supports their initial beliefs about a case, leading to potential miscarriages of justice~\cite{rassin2010let,rassin2020context}.

\paragraph{Confirmation bias in science.}
In scientific research, confirmation bias can significantly influence the experimental design and result interpretation, potentially leading to conclusions that reflect the researcher’s expectations rather than objective findings. This bias can cause researchers to unconsciously favor data that supports their hypotheses, overlooking or dismissing contradictory evidence. Recognizing and addressing confirmation bias is crucial to ensure that scientific research remains rigorous, objective, and evidence-based~\cite{mynatt1977confirmation}. 

\paragraph{Motivating example: Einstein from noise.}  The \textit{Einstein from noise}  experiment is a prototype example of confirmation bias in science. 
Consider a scenario where scientists acquire a set of images (observational data) and genuinely believe their observations contain noisy, shifted copies of a known template image (say, an image of Einstein). However, in reality, their data consists of pure noise with no image present. To estimate the (absent) image, the scientists follow the ubiquitous procedure of template matching: align each observation by cross-correlating it with the template (e.g., Einstein's image) and then average the aligned observations. Remarkably, empirical evidence has shown, multiple times, that the reconstructed structure from this process is structurally similar to the template, even when all the observations are pure noise~\cite{shatsky2009method, sigworth1998maximum}. 
This example of generating a structured image from pure noise images garnered significant attention as it was at the heart of a crucial scientific debate about the structure of an HIV molecule~\cite{mao2013molecular,henderson2013avoiding,van2013finding,subramaniam2013structure,mao2013reply}; see~\cite{balanov2024einstein} for a detailed description and statistical analysis. 
This debate serves as one of the prime motivations for this work.

\paragraph{Conformation bias in statistical models.} A statistical model that characterizes confirmation bias should incorporate three main components: the true statistics of the observations (e.g., independent and identically distributed samples drawn from an isotropic Gaussian distribution), the postulated statistical model of the observations (i.e., the hypotheses, for example, shifted, noisy copies of Einstein's image), and an estimation process (algorithm) based on the observations and the initial hypotheses. 
Confirmation bias arises when the output of the estimation process is positively correlated with the postulated statistical model of the observations, although the observations do not support this correlation. 
In this work, we focus on the ubiquitous Gaussian mixture models (GMMs) as the postulated model and a single iteration of the widely used $K$-means and expectation-maximization (EM) algorithms as estimation procedures. We first introduce these ingredients and then describe the conformation bias experiment.

\paragraph{Gaussian mixture models.} 
The primary aim of this work is to statistically model and analyze confirmation bias. 
As a focal point of our discussion, we turn our attention to GMMs---ubiquitous statistical models when the data is thought to arise from multiple underlying distributions. There is a vast literature on Gaussian mixtures, see, e.g., \cite{Lindsay,HeinrichKahn,Sylvia06}, and references therein.
GMMs have broad applicability across various domains in statistics, signal processing, and machine learning, including clustering \cite{yang2012robust, maugis2009variable}, density estimation \cite{villani2009regression, glodek2013ensemble}, image segmentation \cite{gupta1998gaussian, nguyen2012fast}, anomaly detection \cite{li2016anomaly}, and conformation variability analysis of proteins~\cite{chen2021deep}, 
 to name just a few.   
In this paper, we would like to address the following question:
\begin{quote}
   \emph{Do standard GMM algorithms exhibit a bias towards postulated hypotheses, even when the empirical observations do not support them?}
\end{quote}

\paragraph{Algorithms.}
Two highly popular methods to estimate the means in GMMs are the $K$-means algorithm (a.k.a  Lloyd's algorithm) and the expectation-maximization (EM) algorithm~\cite{dempster1977maximum}. 
Both algorithms iteratively refine initial hypotheses of the means of the GMM. 
These algorithms are known to be sensitive to the choice of the initial hypotheses, which are typically either randomly drawn or determined based on prior knowledge, as will be discussed further below. This work, as well as the Einstein from noise phenomenon~\cite{balanov2024einstein}, emphasizes the potential pitfalls of choosing inappropriate initial templates (hypotheses).  

$K$-means iteratively clusters observations and identifies their centers, which correspond to the means in a GMM \cite{Lloyd82}. The EM algorithm offers a more general framework for optimizing the likelihood function of a statistical model. When estimating the means of a GMM, the EM algorithm functions as a soft version of $K$-means. Rather than strictly grouping observations into clusters, it assigns probabilities to all observations across all clusters and then updates the estimates using a weighted average~\cite{dempster1977maximum, Redner82}.
The statistical properties of the EM algorithm have been extensively studied across various contexts and from multiple perspectives; see, for instance, \cite{Jordan01,Hsu16,pmlr-v65-daskalakis17b} and references therein.
In this work, we primarily analyze theoretical results for a single iteration of $K$-means (the \emph{hard assignment algorithm}), and a single iteration of the EM algorithm (the \emph{soft assignment algorithm}), applied to GMMs~\cite{kearns1998information}. We then extend our analysis to the multi-iteration case. The details of both algorithms, along with the probabilistic model, are presented in Section~\ref{sec:problemFormulation}.

\begin{figure}[h]
	\centering
	\includegraphics[width=0.95\textwidth]{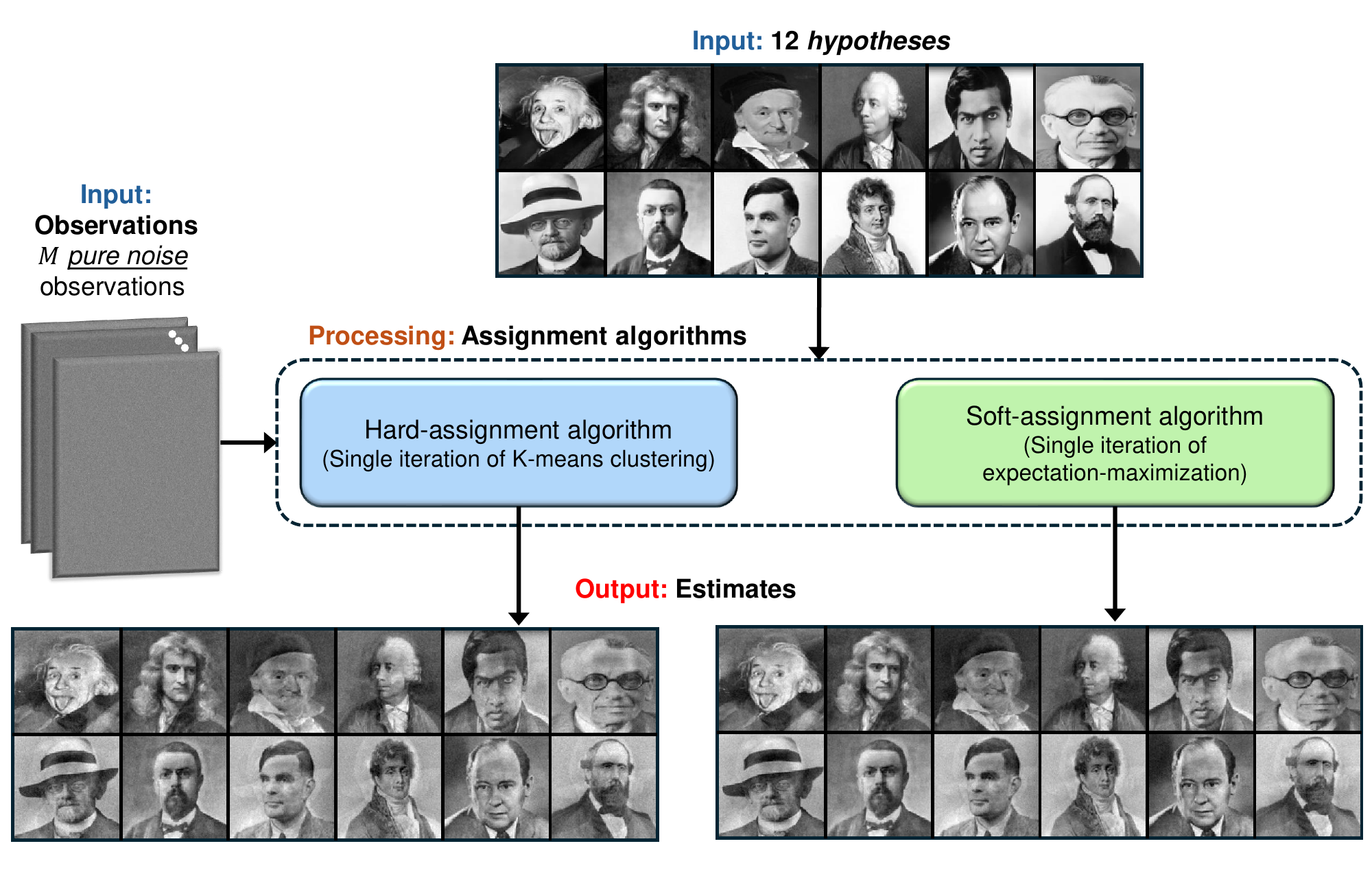}
	\caption{\textbf{Confirmation Bias in GMMs.} 
 A group of scientists believes they have collected multiple low signal-to-noise ratio (SNR) observations of images of notable mathematicians. In reality, however, the data is purely random noise. 
 Based on a GMM and the presumed hypotheses, they estimate the Gaussian centroids using the hard-assignment procedure (a single iteration of $K$-means) and the soft-assignment algorithm (a single iteration of EM). The resulting estimates structurally resemble the 12 hypotheses, although the observations do not support these hypotheses. This experiment was conducted with $M = 2 \times 10^6$ observations and images of size of $100 \times 100$ pixels.} 
	\label{fig:0}
\end{figure}

\paragraph{Statistical properties of $K$-means and EM.} 
The statistical properties of the $K$-means and EM algorithms have received considerable attention over the past decade, particularly in the context of two-component GMMs and well-specified settings where the population likelihood is locally strongly concave around its maximum (see, for example, \cite{xu2016global, daskalakis2017ten, wu2021randomly, dwivedi2020singularity, kuncheva2006evaluation}). In general, since the likelihood function is not necessarily strongly concave, both $K$-means and EM may converge to a local optimum rather than the global maximum of the likelihood function. Consequently, their performance is sensitive to initialization~\cite{celebi2013comparative, celebi2011improving, bubeck2012initialization}. However, only a few initialization methods come with theoretical accuracy guarantees, such as $K$-means++ \cite{arthur2006k} and local search approaches \cite{kanungo2002local, gupta2017local, lattanzi2019better}. As a result, a common practice in many applications is to initialize $K$-means and EM with a point that closely aligns with the researcher's prior beliefs. We discuss this further in Section \ref{sec:outlook}.

\paragraph{Confirmation bias in GMMs.}
The following experiment illustrates confirmation bias in GMMs. A team of scientists assumes they have collected multiple observations of known templates (hypotheses) swamped by high noise levels; in Figure~\ref{fig:0}, the templates correspond to images of 12 notable mathematicians. The scientists hypothesize that these observations were generated by a Gaussian mixture model, with the 12 mathematicians (hypotheses) serving as the centroids of the Gaussian mixtures.
However, the observations are purely random noise with no underlying structure. This discrepancy between the actual statistics of the observations and the assumed statistical model forms the core of the confirmation bias phenomenon.

To estimate the Gaussian centroids from these observations, the scientists apply both the hard-assignment procedure (a single iteration of $K$-means) and the soft-assignment procedure (a single iteration of EM). The results of such an experiment are shown in Figure~\ref{fig:0}.
Remarkably, the outputs of both assignment algorithms closely resemble the 12 hypotheses, reflecting the scientists' initial assumptions. This outcome contrasts with the unbiased expectation that averaging pure noise images would converge to zero. The objective of this paper is to explore this phenomenon and to characterize its statistical properties.

\paragraph{Main results.}
This study characterizes the relationship between the initial templates and their corresponding hard and soft assignment estimators. The main results are presented in Section~\ref{sec:mainResults} and proved in the appendices.
We begin by showing in Theorem \ref{thm:hardAssignmentPositiveCorrelation} that there is a positive correlation between each assignment estimator and its corresponding template, which explains the structural similarity between the reconstructed structures and their associated templates. Additionally, we demonstrate an inverse relationship between the cross-correlations among the templates and the correlations between the assignment estimators and their respective templates (Proposition \ref{thm:meanInverseDependency}). This inverse relationship supports the intuitive expectation that as the hypotheses become more distinct, the confirmation bias tends to increase.
We then provide analytical results for different numbers of hypotheses. In the case of two hypotheses, Theorems \ref{thm:hardAssignmentTwoHypoteses} and \ref{thm:softAssignmentTwoTemplates} provide explicit expressions for the hard-assignment and soft-assignment estimators, respectively. For a finite number of hypotheses, Theorem \ref{thm:hardAssignmentLinearCombination} shows that the assignment estimator is a linear combination of the templates. Finally, Theorems \ref{thm:hardAssignmentAsymptoticLandAsymptoticD} and \ref{thm:softAssignmentAsymptoticL} demonstrate that as the number of hypotheses and the dimension of the signals increase, the assignment estimator converges to the corresponding template. In Section \ref{sec:multiIteration}, we extend the theoretical framework and simulation results to the multi-iteration regime for the hard-assignment and soft-assignment methods. Specifically, we demonstrate that confirmation bias persists even after hundreds of iterations of $K$-means and EM.
Section~\ref{sec:outlook} concludes the paper and delineates several potential directions for future research.

\section{Problem formulation} \label{sec:problemFormulation}
This section introduces the probabilistic framework and presents our main mathematical objectives. Although the empirical demonstrations in this work are shown in the context of images, we will formulate and analyze the problem using 1-D signals for simplicity of notation. Since our results and proofs rely on the cross-correlations between the templates, the extension of the results to higher dimensions is straightforward.
We begin by outlining the general Gaussian mixture model. Following this, we introduce the hard-assignment and soft-assignment methods used to estimate the means of the mixture model components and discuss the relationship between these methods.

\subsection{Gaussian mixture models}\label{subsec:GMMandour}

A GMM with $L$ components in $d$ dimensions can be represented by $\boldsymbol{w} = (w_0,w_1,\ldots,w_{L-1})$, $\boldsymbol{\mu} = (\mu_0,\mu_1,\ldots,\mu_{L-1})$, and $\boldsymbol{\Sigma} = (\Sigma_0,\Sigma_1,\ldots,\Sigma_{L-1})$, where $w_\ell$ is a mixing weight such that $w_\ell\geq0$ and $\sum_{\ell=0}^{L-1}w_\ell=1$, $\mu_\ell\in\mathbb{R}^d$ is the $\ell$th component mean, and $\Sigma_\ell$ is the $\ell$th component covariance matrix. To draw a random instance from this GMM, one first samples an index $\ell\in[L]$, with probability $w_\ell$, and then returns a random sample from the Gaussian distribution $\calN(\mu_\ell,\Sigma_\ell)$. Stated differently, $M$ GMM samples are generated as follows,
\begin{align}
    y_0,y_1,\ldots,y_{M-1}\stackrel{\s{i.i.d.}}{\sim}\sum_{\ell=0}^{L-1}w_\ell\cdot\calN(\mu_\ell,\Sigma_\ell).
\end{align}
We denote the probability density function of the $\ell$th component by $\phi_d(\mu_\ell,\Sigma_\ell)$, and the GMM density by,
\begin{align}
    \calM_{L,d}(\boldsymbol{w},\boldsymbol{\mu},\boldsymbol{\Sigma})\triangleq\sum_{\ell=0}^{L-1}w_\ell\cdot\phi_d(\mu_\ell,\Sigma_\ell).\label{eqn:GMMpdf}
\end{align}
 
We explore the effects and potential biases that arise from erroneous assumptions in GMMs. To wit, we consider the following experiment. Let the $M$ underlying observations $(y_0,y_1,\ldots,y_{M-1})$ be distributed as  isotropic Gaussian random vectors with variance $\sigma^2$, i.e., $y_i\sim \mathcal{N}(0,\sigma^2 I_{d \times d})$. In terms of GMMs, this can also be written equivalently as,
    \begin{align} \label{eqn:calPModel}
        (\s{Underlying}\;\s{statistics})    \quad y_0,y_1,\ldots,y_{M-1}\stackrel{\s{i.i.d.}}{\sim}\mathcal{M}_{L,d}(\mathbf{1}/L,\mathbf{0},\sigma^2I_{d \times d}),
    \end{align}
    where $\mathbf{0}$ and $\mathbf{1}$ are all-zeros and all-ones vectors, respectively. A researcher believes, on the other hand, that these observations are generated from a GMM with $L$ distinct components and with different means (for example, the 12 mathematicians in Figure 1) and the same covariance matrix $\sigma^2 I_{d \times d}$, namely, 
    \begin{align} \label{eqn:calQModel}
          (\s{Postulated}\;\s{statistics})    \quad   y_0,y_1,\ldots,y_{M-1}\stackrel{\s{i.i.d.}}{\sim}\mathcal{M}_{L,d}(\mathbf{1}/L,\boldsymbol{\mu},\sigma^2I_{d \times d}). 
    \end{align}
Since the underlying statistics of the observations correspond to pure noise, as given in \eqref{eqn:calPModel}, we explicitly denote them as $\ppp{y_i}_{i=0}^{M-1} = \ppp{n_i}_{i=0}^{M-1} $ for clarity. Without loss of generality, we assume $\sigma^2 = 1$ throughout our analysis. Extending our model to the more general case in \eqref{eqn:calPModel}, where the underlying statistics have a nonzero mean and $\sigma \neq 1$ is straightforward. However, for notational simplicity, we focus on the case of zero-mean noise with $\sigma^2=1$. 

To estimate these means, the researcher applies a certain estimation procedure, coupled with a given (biasing) side information of $L$ different initial templates, denoted by $(x_0,x_1,\ldots,x_{L-1})$, which she suspects are close to the actual means. These initial templates embody the researcher’s initial assumptions about the data generation model. If the estimation process is unbiased, that is, it remains unaffected by these initial templates, we anticipate that it would converge towards $\mathbf{0}$ (the true means) as the number of observations grows. However, as we demonstrate in this study, this may not necessarily be the result. 

\begin{figure}[h]
    \centering
    \includegraphics[width=0.95\textwidth]{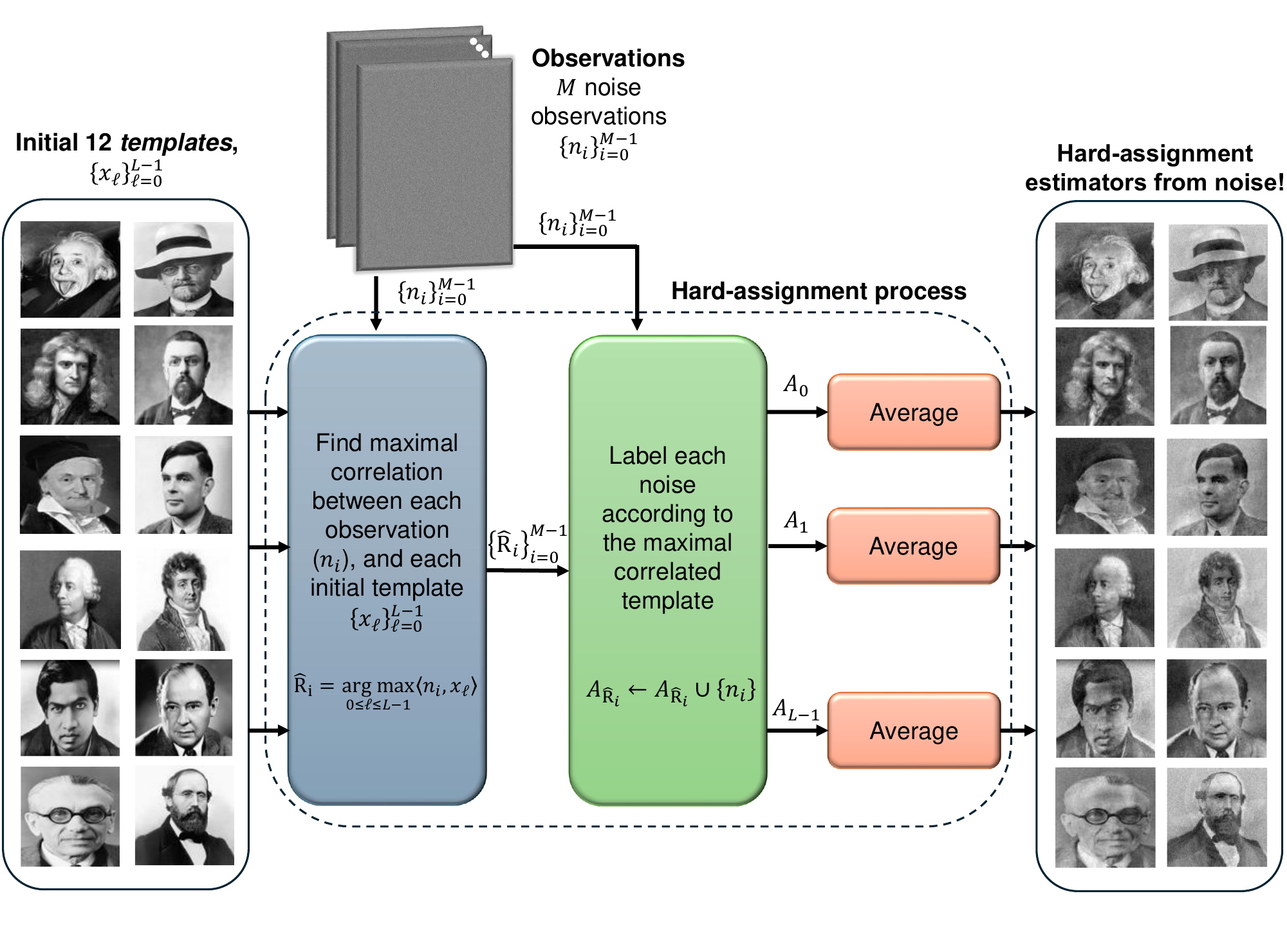}
    \caption{ 
  Confirmation bias with the hard-assignment algorithm: The model consists of $L=12$ distinct templates, $\ppp{x_\ell}_{\ell=0}^{L-1}$, each representing an initial hypothesis for the means of the components in the GMM. There are $M$ data observations, $\ppp{n_i}_{i=0}^{M-1}$, each assumed to correspond to one of the templates. However, in reality, these observations are purely random noise. During the hard-assignment process, each observation is labeled with the highest-correlated template. Observations associated with the same template are then averaged. The resulting hard-assignment estimators closely resemble the initial templates, revealing that the estimation process is biased toward the initial hypotheses.}
    \label{fig:1}
\end{figure}

 We analyze two estimation processes: \emph{single-iteration} hard-assignment and soft-assignment algorithms, which we define in subsequent sections. 
Our main goal is to assess the correlation between the estimation of the means produced by the above methodologies and the corresponding templates (i.e., the initial hypotheses).
{In the sequel, we assume that all templates have the same norm, that is, $\norm{x_\ell}_2 \triangleq\norm{x}_2$, for all $\ell\in[L]$. }

\subsection{Assignment algorithms}\label{sec:2_2}

\paragraph{The hard-assignment process.}
The hard assignment procedure labels each observation with the hypothesis that achieves the highest correlation among the $L$ possible template hypotheses. Then, the algorithm computes the average of all the observations that best align with the $\ell$-th hypothesis relative to the other hypotheses. This averaging is performed for each hypothesis, resulting in $L$ different assignment estimators, denoted by $\ppp{\hat{x}_\ell}_{\ell=0}^{L-1}$, each corresponding to a different template. A pseudo-code for this procedure is provided in Algorithm~\ref{alg:generalizedEfNhard}. The procedure and results of the hard-assignment process are illustrated empirically in Figure~\ref{fig:1}.

\begin{algorithm}[t!]
  \caption{\texttt{Hard-assignment} \label{alg:generalizedEfNhard}}
\textbf{Input:} Initial templates $\ppp{x_\ell}_{\ell=0}^{L-1}$ and observations $\ppp{n_i}_{i=0}^{M-1}\stackrel{\s{i.i.d.}}{\sim} \calN \p{0, I_{d \times d}}$.\\
\textbf{Output:} Hard assignments estimates $\ppp{\hat{x}_\ell}_{\ell=0}^{L-1}$.
\begin{enumerate}
    \item Initialize: $\calA_\ell\leftarrow\emptyset$, for $\ell = 0,1,\ldots,L-1$. 
    \item For $i=0,\ldots,M-1$, compute
     \begin{align}
        \s{\hat{R}}_i\triangleq \underset{0 \leq \ell \leq L-1}{\argmax} {\,   \langle{n_i}, x_\ell\rangle},
    \label{eqn:OptShiftRealSpace}
    \end{align} 
    and add to the set $\calA_{\s{\hat{R}}_i}$ the noise observation $n_i$: $\calA_{\s{\hat{R}}_i} \leftarrow \calA_{\s{\hat{R}}_i} \cup \ppp{n_i}$.
\item  Compute for $0 \leq \ell \leq L-1$: 
    \begin{align}
        \hat{x}_\ell\triangleq \frac{1}{\abs{\calA_\ell}} \sum_{n_i \in \calA_\ell} n_i \label{eqn:generalizedEfNEqn}.
    \end{align}
\end{enumerate}
\end{algorithm}

\paragraph{The soft-assignment process.}
Each iteration of the  classical EM algorithm consists of two steps: the E-step, which calculates the expected value of latent variables given the observed data and current parameter estimates, and the M-step, which updates the parameters to maximize the expected likelihood determined in the E-step~\cite{dempster1977maximum}.
To estimate the means in GMMs, a single iteration of the EM algorithm is given by,
\begin{align} 
       \hat{x}_\ell = \frac{\sum_{i=0}^{M-1}{ p_{i}^{(\ell)}n_i}}{\sum_{i=0}^{M-1}{p_{i}^{(\ell)}}},\label{eqn:softAssignmentUpdateStep} 
\end{align}
where
\begin{align}
    p_{i}^{(\ell)} \triangleq \frac{\exp \p{n_i^Tx_\ell}}{\sum_{r=0}^{L-1} \exp \p{n_i^Tx_r}}, \label{eqn:softMaxProbabilityDist}
\end{align}
where $\ppp{{x}_\ell}_{\ell=0}^{L-1}$ are the hypotheses, and $\ppp{\hat{x}_\ell}_{\ell=0}^{L-1}$ are the estimations.
See Appendix~\ref{sec:softAssignmentUpdateStep} for the proof. Note that in contrast to the hard-assignment process, each observation is not assigned to a single template. Instead, we compute the probability that each observation is associated with each template, thus the name soft assignment. We then average all observations, weighted by the probabilities.  A pseudo-code for the soft assignment procedure is given in Algorithm~\ref{alg:generalizedEfNsoft}. 

We are now ready to formally define confirmation bias in the context of this work.
\begin{definition}
    Let $x$ be an initial hypothesis/template, and $\hat{x}$ be the output of an estimation procedure (e.g., Algorithms \ref{alg:generalizedEfNhard}--\ref{alg:generalizedEfNsoft}). Confirmation bias is quantified by the inner product $\langle \hat{x}, x\rangle$. If $\langle \hat{x}, x\rangle > 0$, then we say that $\hat{x}$ is biased towards $x$.
\end{definition} 

We note that although most of the conditions and results in this work are expressed in terms of inner products, either between templates or between a template and its corresponding estimator, these can equivalently be expressed using the Kullback--Leibler (KL) divergence. In particular, under the isotropic Gaussian noise model with unit-norm templates, the KL divergence between $\mathcal{N}(x_k, I_{d \times d})$ and $\mathcal{N}(x_\ell, I_{d \times d})$ simplifies to
\[
    D_{\mathrm{KL}}(\mathcal{N}(x_k, I_{d \times d}) \,\|\, \mathcal{N}(x_\ell, I_{d \times d})) = 1 - \langle x_k, x_\ell \rangle.
\]
As a result, conditions involving inner products, such as the degree of confirmation bias or similarity between templates, can be equivalently expressed through the lens of information-theoretic divergence.

\begin{algorithm}[t!]
  \caption{\texttt{Soft-assignment} \label{alg:generalizedEfNsoft}}
\textbf{Input:} Initial templates $\ppp{x_\ell}_{\ell=0}^{L-1}$ and observations $\ppp{n_i}_{i=0}^{M-1}\stackrel{\s{i.i.d.}}{\sim} \calN \p{0, I_{d \times d}}$.\\
\textbf{Output:} Soft assignment estimates $\ppp{\hat{x}_\ell}_{\ell=0}^{L-1}$.
\begin{enumerate}
    \item  Compute $p_i^{(\ell)}$ \eqref{eqn:softMaxProbabilityDist}, for all $i=0,\ldots,M-1$ and $\ell=0,\ldots,L-1$.
\item  Compute $\hat{x}_\ell$ in \eqref{eqn:softAssignmentUpdateStep}, for all $\ell=0,\ldots,L-1$.
\end{enumerate}
\end{algorithm}

\subsection{Hard and soft assignments boundaries}

Before presenting the main results of this paper, the bias introduced by the hard-assignment and soft-assignment procedures, we present the tight relationship between these two estimation processes in the extreme signal-to-noise ratio (SNR) levels.
Let us define $\hat{x}_\ell^{(\beta)}$ as the estimator in \eqref{eqn:softAssignmentUpdateStep} but with $p_{i}^{(\ell)}$ replaced by,
\begin{align}
    p_{i,\beta}^{(\ell)} \triangleq \frac{\exp \p{\beta n_i^Tx_\ell}}{\sum_{r=0}^{L-1} \exp \p{\beta n_i^Tx_r}}. \label{eqn:softMaxProbabilityDistWithBeta}
\end{align}
We will refer to $\hat{x}_\ell^{(\beta)}$ as the $\beta$-soft-assignment.
The parameter $\beta$ can be interpreted as if all templates are multiplied by the same constant factor $\beta \in \mathbb{R}, \beta \neq 0$, where~\eqref{eqn:softMaxProbabilityDist} corresponds to $\beta=1$.  
We thus refer to $\beta$ as the SNR parameter.

Proposition~\ref{thm:softAssignmentExtremeNormalization} describes the extreme cases of low and high SNRs. 
When $\beta \to \infty$, the soft-assignment and the hard-assignment estimators converge to the same value. Conversely, in the low SNR regime, as $\beta \to 0$, the soft-assignment  estimator can be expressed as a linear combination of the templates {that converges to zero.}

\begin{proposition}
\label{thm:softAssignmentExtremeNormalization}
Fix $L\geq 2$ and denote by $\{\hat{x}_\ell^{(\beta)}\}_{\ell=0}^{L-1}$ the output of the $\beta$-soft-assignment estimator described above. 
    \begin{enumerate}
         \item For every $0 \leq \ell \leq L-1$, we have,
        \begin{align}
            \lim_{\beta \to \infty} \hat{x}_\ell^{(\beta)} =  \frac{1}{\abs{\calA_\ell}} \sum_{n_i \in \calA_\ell} n_i,
        \end{align}
        almost surely, where $\calA_\ell$ is defined in Algorithm \ref{alg:generalizedEfNhard}. In other words, the soft-assignment procedure in Algorithm \ref{alg:generalizedEfNsoft} and the hard-assignment procedure in Algorithm \ref{alg:generalizedEfNhard} coincide for $\beta\to\infty$.
     \item For every $0 \leq \ell \leq L-1$, we have,
              \begin{align}
                \lim_{\beta \to 0} \lim_{M \to \infty} \frac{ \hat{x}_\ell^{(\beta)}}{\beta } = {x_\ell - \frac{1}{L}\sum_{r = 0}^{L-1} x_r}, \label{eqn:asymptoticNormXtoZeros}
            \end{align}
        almost surely. 
    \end{enumerate}
\end{proposition}


\section{Main results} \label{sec:mainResults}

We begin by proving that there exists a positive correlation between the assignment estimators and their corresponding templates. Next, we show that this correlation increases as the cross-correlation between the templates decreases. This inverse relation indicates that selecting initial templates with lower correlation results in greater model bias.
We then examine various scenarios depending on the number of templates, $L$. Specifically, we derive exact analytical results for the case of two templates ($L=2$) and investigate the behavior for a finite number of templates. Finally, we consider the scenario where both the number of templates and the signal dimension grow unbounded.
The appendix contains the proofs of the results.

\subsection{Fundamental properties}
The following result highlights three fundamental properties of the hard-assignment and soft-assignment estimators. First, both estimators converge to a non-zero signal, contrasting the unbiased model's prediction that averaging zero-mean pure noise signals would converge to zero. Second, the estimators exhibit a positive correlation with their respective template signals. We further establish a lower bound on the average correlation, which depends on the separation between the most distinct templates. Third, we prove a consistency property for the hard-assignment estimator: the underlying template achieves the maximum correlation with a given estimator among all possible templates.

\begin{thm}\label{thm:hardAssignmentPositiveCorrelation}
Fix $L\geq 2$, and assume that $x_{\ell_1} \neq x_{\ell_2}$ for every $\ell_1 \neq \ell_2$. 
\begin{enumerate}
    \item (Non-vanishing estimators.) Let $\ppp{\hat{x}_\ell}_{\ell=0}^{L-1}$ be the output of either Algorithm~\ref{alg:generalizedEfNhard} or Algorithm~\ref{alg:generalizedEfNsoft}. Then, for every $0 \leq \ell \leq L-1$,
    \begin{align}
        \lim_{M\to\infty} \hat{x}_\ell \neq 0 \label{eqn:nonVanishingEstimator},
    \end{align}
    almost surely.
    \item (Positive correlation.) Let $\ppp{\hat{x}_\ell}_{\ell=0}^{L-1}$ be the output of either Algorithm~\ref{alg:generalizedEfNhard} or Algorithm~\ref{alg:generalizedEfNsoft}. Then, for every $0 \leq \ell \leq L-1$,
    \begin{align}
        \lim_{M\to\infty} \mathbb{E} \pp{{\langle{\hat{x}_\ell}, x_\ell\rangle}} > 0 \label{eqn:positiveCorrelation}.
    \end{align}
    \item (Lower bound for  average correlation.) Denote by $\ppp{\hat{x}_\ell}_{\ell=0}^{L-1}$ the output of Algorithm~\ref{alg:generalizedEfNhard}, and $\s{\hat{R}}$ be defined as in \eqref{eqn:OptShiftRealSpace}, with the templates $\ppp{x_\ell}_{\ell=0}^{L-1}$, satisfying $\|x_\ell\| = 1$, for every $\ell \in \pp{L}$. Let $\rho = \underset{i \neq j} \min \pp{{\langle{{x}_{i}}, x_{j}\rangle}}$. Then, 
    \begin{align}
        \lim_{M\to\infty} \sum_{\ell=0}^{L-1} \mathbb{E}{\langle{\hat{x}_\ell}, x_\ell\rangle} \mathbb{P}[\s{\hat{R}} = \ell] \geq \sqrt{\frac{1-\rho}{\pi}} > 0 \label{eqn:positiveCorrelationSum},
    \end{align}
    almost surely. 
    \item (Consistency.) Let $\ppp{\hat{x}_\ell}_{\ell=0}^{L-1}$ be the output of Algorithm~\ref{alg:generalizedEfNhard}. Then, for every $k \neq \ell$,
    \begin{align}
        \lim_{M\to\infty} \mathbb{E} \pp{{\langle{\hat{x}_\ell}, x_\ell \rangle}} - \mathbb{E} \pp{{\langle{\hat{x}_\ell}, x_k \rangle}} > 0 \label{eqn:closetstTemplateIsCorresponding}.
    \end{align}

\end{enumerate}

\end{thm}

Note that the consistency relation~\eqref{eqn:closetstTemplateIsCorresponding} is not necessarily true if we take the absolute values of the correlations. This means that, in principle, there could be a template $x_k$, different from the underlying one $x_\ell$, whose \emph{negative correlation} with $\hat{x}_\ell$ is \emph{larger} than the \emph{positive} correlation with $x_\ell$, namely, $-\mathbb{E} \pp{{\langle{\hat{x}_\ell}, x_k \rangle}}>\mathbb{E} \pp{{\langle{\hat{x}_\ell}, x_\ell \rangle}}$. 

The second property, which states that the correlation between the estimator and the associated templates is strictly positive, is detailed in Appendices \ref{sec:A5} and \ref{sec:appD}. The third property establishes a lower bound on the \textit{average correlation} between the estimators $\ppp{\hat{x}_\ell}_{\ell =0}^{L-1}$ and their associated templates $\ppp{x_\ell}_{\ell=0}^{L-1}$, demonstrating that the average correlation is lower bounded by the most distinct templates.

Next, for the important case where the templates exhibit a cyclic symmetry relation, as defined below, we derive an explicit lower bound on the expected correlation of \textit{each pair} of estimator and its corresponding template.

\begin{assum}[Cyclic symmetry]\label{assump:0}
We say that $\ppp{x_\ell}_{\ell=0}^{L-1}$ have a cyclic symmetric structure, if there exist a sequence $\{\rho_\ell\}_\ell$ of real-valued numbers, such that,
    \begin{align}
        {\langle{{x}_{\ell_1}}, x_{\ell_2}\rangle} = \rho_{\abs{\ell_1-\ell_2}\s{mod}L},
    \end{align}
for every $0\leq\ell_1,\ell_2\leq L-1$.
\end{assum}

Assumption~\ref{assump:0} specifies that the correlations between templates have a cyclic dependence. A typical example of templates that meet this assumption is a set of signals that includes a reference template and all its cyclic translations, akin to the model studied in~\cite{balanov2024einstein}. This scenario is also likely to occur in applications with intrinsic symmetries, such as cryo-electron microscopy~\cite{bendory2020single} and multi-reference alignment~\cite{bendory2017bispectrum,perry2019sample,bandeira2023estimation,bendory2024sample}. The connection between these applications and the results of this work will be further discussed in Section \ref{sec:outlook}.
Under this assumption, we obtain the following result.
\begin{proposition}
[Positive correlation for cyclic symmetry] \label{thm:positiveCorrForCyclic}
Fix $L\geq 2$. Consider normalized template hypotheses $\|x_\ell\| =1$, for every $\ell \in \pp{L}$, satisfying Assumption \ref{assump:0}. Denote by $\rho = \underset{i \neq j} \min \pp{{\langle{{x}_{i}}, x_{j}\rangle}}$, and by $\ppp{\hat{x}_\ell}_{\ell=0}^{L-1}$ the output of Algorithm~\ref{alg:generalizedEfNhard}. Then, for every $0 \leq \ell \leq L-1$,
    \begin{align}
        \lim_{M\to\infty} \mathbb{E} \pp{{\langle{\hat{x}_\ell}, x_\ell\rangle}} \geq \sqrt{\frac{1-\rho}{\pi}} > 0 \label{eqn:positiveCorrelationCyclic}.
    \end{align}
\end{proposition}
Proposition~\ref{thm:positiveCorrForCyclic} shows that \textit{each} estimator exhibits a guaranteed level of confirmation bias—its expected correlation with the corresponding template is lower bounded by a function of the minimum similarity between any two templates. In particular, the more distinct the templates are (i.e., the smaller $\rho$ is), the larger this lower bound becomes.

\subsection{Inverse dependency}
Next, we demonstrate an intriguing finding regarding the dependency of the correlations between the estimators and the templates on the cross-correlations between different template pairs. Specifically, these correlations increase as the cross-correlations between different templates decrease. This, in turn, implies that choosing initial templates that are less correlated would lead to a higher model bias. 

\begin{proposition}
    [Average inverse dependency] \label{thm:meanInverseDependency}
    Fix $L\geq 2$. Consider two different normalized template hypotheses sets $\ppp{{x}_\ell}_{\ell=0}^{L-1}$ and $\ppp{{y}_\ell}_{\ell=0}^{L-1}$, such that $\|x_\ell\| = \|y_\ell\| = 1$, for every $\ell \in \pp{L}$, and  ${\langle{{x}_{\ell_1}}, x_{\ell_2}\rangle} \leq {\langle{{y}_{\ell_1}}, y_{\ell_2}\rangle}$, for every $0 \leq \ell_1,\ell_2 \leq L-1$. Denote by $\ppp{\hat{x}_\ell}_{\ell=0}^{L-1}$ and $\ppp{\hat{y}_\ell}_{\ell=0}^{L-1}$ the output of Algorithm~\ref{alg:generalizedEfNhard}. Let $\s{\hat{R}}^{(x)}$ and $\s{\hat{R}}^{(y)}$ be defined as in \eqref{eqn:OptShiftRealSpace}, corresponding to the templates sets $\ppp{x_\ell}_{\ell=0}^{L-1}$ and $\ppp{y_\ell}_{\ell=0}^{L-1}$, respectively. Then, as $M\to\infty$, 
        \begin{align}
            \sum_{\ell=0}^{L-1} {\langle{\hat{x}_\ell}, x_{\ell}\rangle} \mathbb{P}[\s{\hat{R}}^{(x)} = \ell] \geq  \sum_{\ell=0}^{L-1} {\langle{\hat{y}_\ell}, y_{\ell}\rangle} \mathbb{P}[\s{\hat{R}}^{(y)} = \ell],
            \label{eqn:hardAssignMeanInverse}
        \end{align}
        almost surely.
\end{proposition}

Proposition \ref{thm:meanInverseDependency} implies that if the correlation between different templates decreases, then the \textit{weighted average} of the correlations between the estimators and the corresponding templates increases. Note that, however, it is not true that each individual correlation between an estimator and its corresponding template increases. For this to happen, additional conditions on the templates should hold; Proposition \ref{prop:softAssignmentGreaterCorrelation} in the Appendix formulates some necessary conditions. In particular, if the templates satisfy the cyclic symmetry property, as specified in Assumption \ref{assump:0}, the inverse property in Proposition~\ref{thm:meanInverseDependency} holds for each pair of an estimator and its corresponding template.

\begin{proposition}
[Individual inverse dependency] \label{thm:largerCorrelationForCycloCorrelations}
Fix $L\geq 2$. Consider two different normalized template hypotheses $\ppp{{x}_\ell}_{\ell=0}^{L-1}$ and $\ppp{{y}_\ell}_{\ell=0}^{L-1}$, both satisfying Assumption \ref{assump:0}, such that ${\langle{{x}_{\ell_1}}, x_{\ell_2}\rangle} \leq {\langle{{y}_{\ell_1}}, y_{\ell_2}\rangle}$, for every $0 \leq \ell_1,\ell_2 \leq L-1$. Denote by $\ppp{\hat{x}_\ell}_{\ell=0}^{L-1}$ and $\ppp{\hat{y}_\ell}_{\ell=0}^{L-1}$ the output of  Algorithm~\ref{alg:generalizedEfNhard}. Then, for every $0 \leq \ell \leq L-1$, as $M\to\infty$,
    \begin{align}
        \mathbb{E} \pp{{\langle{\hat{x}_\ell}, x_\ell\rangle}} \geq  \mathbb{E} \pp{{\langle{\hat{y}_\ell}, y_\ell\rangle}} \label{eqn:hardAssignmentGreaterCorrelation},
    \end{align}
    almost surely.
\end{proposition}

A related concept to the inverse dependency property is the monotonicity property: as additional templates are incorporated into the model, the confirmation bias increases. This relationship is formalized in the following result.

\begin{proposition}[Monotonicity of the confirmation bias] \label{thm:monotonicity}  
Fix $L_X, L_Y \geq 2$. Consider two sets of normalized template hypotheses, $\mathcal{X} = \{x_\ell\}_{\ell=0}^{L_X - 1}$ and $\mathcal{Y} = \{y_\ell\}_{\ell=0}^{L_Y - 1}$, such that $L_X \geq L_Y$, and $\mathcal{X} \supseteq \mathcal{Y}$; that is, $\mathcal{X}$ is an extension of $\mathcal{Y}$.  
Let $\{\hat{x}_\ell\}_{\ell=0}^{L_X - 1}$ and $\{\hat{y}_\ell\}_{\ell=0}^{L_Y - 1}$ denote the outputs of Algorithm~\ref{alg:generalizedEfNhard}, and let $\hat{\s{R}}^{(x)}$ and $\hat{\s{R}}^{(y)}$ be defined as in \eqref{eqn:OptShiftRealSpace}, corresponding to the template sets $\{x_\ell\}_{\ell=0}^{L_X - 1}$ and $\{y_\ell\}_{\ell=0}^{L_Y - 1}$, respectively. Then, as $M \to \infty$, 
\begin{align}
    \sum_{\ell=0}^{L_X - 1} \langle \hat{x}_\ell, x_\ell \rangle \, \mathbb{P}[\hat{\s{R}}^{(x)} = \ell] 
    \geq 
    \sum_{\ell=0}^{L_Y - 1} \langle \hat{y}_\ell, y_\ell \rangle \, \mathbb{P}[\hat{\s{R}}^{(y)} = \ell],
    \label{eqn:hardAssignMeanMonotone}
\end{align}
almost surely.
\end{proposition}

\subsection{Two templates}
We now turn to analyze the behavior of the algorithms as a function of the number of templates. We begin with two templates $L=2$. In this case, we derive closed-form expressions for the structure of the hard and soft assignment estimators. Specifically, we show that the estimators $\ppp{\hat{x}_\ell}_{\ell=0}^{1}$ in Algorithms \ref{alg:generalizedEfNhard} and \ref{alg:generalizedEfNsoft} can be represented as specific linear combinations of the template signals $\ppp{{x}_\ell}_{\ell=0}^{1}$, as $M\to\infty$. 
In both algorithms, since $p_{i,\beta}^{(0)} + p_{i,\beta}^{(1)} = 1$ for every $\beta \in \mathbb{R}$ (where $p_{i,\beta}^{(\ell)}$ is defined in \eqref{eqn:softMaxProbabilityDistWithBeta}), and $\mathbb{E}\pp{n_i} = 0$, it follows that $\hat{x}_0 + \hat{x}_1 \to 0$ as $M\to\infty$. That is, the two estimators are the contrasting signals of each other. Furthermore, the estimator is a linear combination of the two templates, where the linear coefficients depend explicitly on the cross-correlation between the two templates.

We start with the hard-assignment estimator. 

\begin{thm}[Hard-assignment for $L=2$]\label{thm:hardAssignmentTwoHypoteses}
Denote by $\hat{x}_0$ and $\hat{x}_1$ the output of Algorithm~\ref{alg:generalizedEfNhard}. Let $\rho \triangleq  {{\langle{{x}_{0}}, x_{1}\rangle}}/\norm{x}_2^2$, and assume that $\rho<1$. Then,
    \begin{align}
        \hat{x}_0 \rightarrow \sqrt{\frac{1}{\pi\p{1-\rho} \norm{x}_2^2}} \p{x_0 - x_1}, \label{eqn:hardAssignTwoHypoteses1}
    \end{align}
and
    \begin{align}
        \hat{x}_1 \rightarrow \sqrt{\frac{1}{\pi\p{1-\rho}\norm{x}_2^2}} \p{x_1 - x_0}, \label{eqn:hardAssignTwoHypoteses2}
    \end{align}
almost surely, as $M\to\infty$.
\end{thm}

Figure \ref{fig:2} illustrates Theorem \ref{thm:hardAssignmentTwoHypoteses} by presenting three extreme cases where $\rho = \{-1, 0, 0.99\}$. When $\rho = -1$ (the images are the contrast of each other), we observe an accurate reconstruction of the Einstein template and its contrast. 
When $\rho = 0$, the estimator appears as a linear combination of the Einstein and cameraman templates. For $\rho = 0.99$, the image appears to be filled with noise, and the correlation with the Einstein template is barely noticeable.
These results are predicted by Theorem \ref{thm:hardAssignmentTwoHypoteses}, and the ``contrast" image of $\hat{x}_1\approx -\hat{x}_0$ is clearly visible in all cases. 

\begin{figure}[t!]
	\centering
	\includegraphics[width=0.95\linewidth]{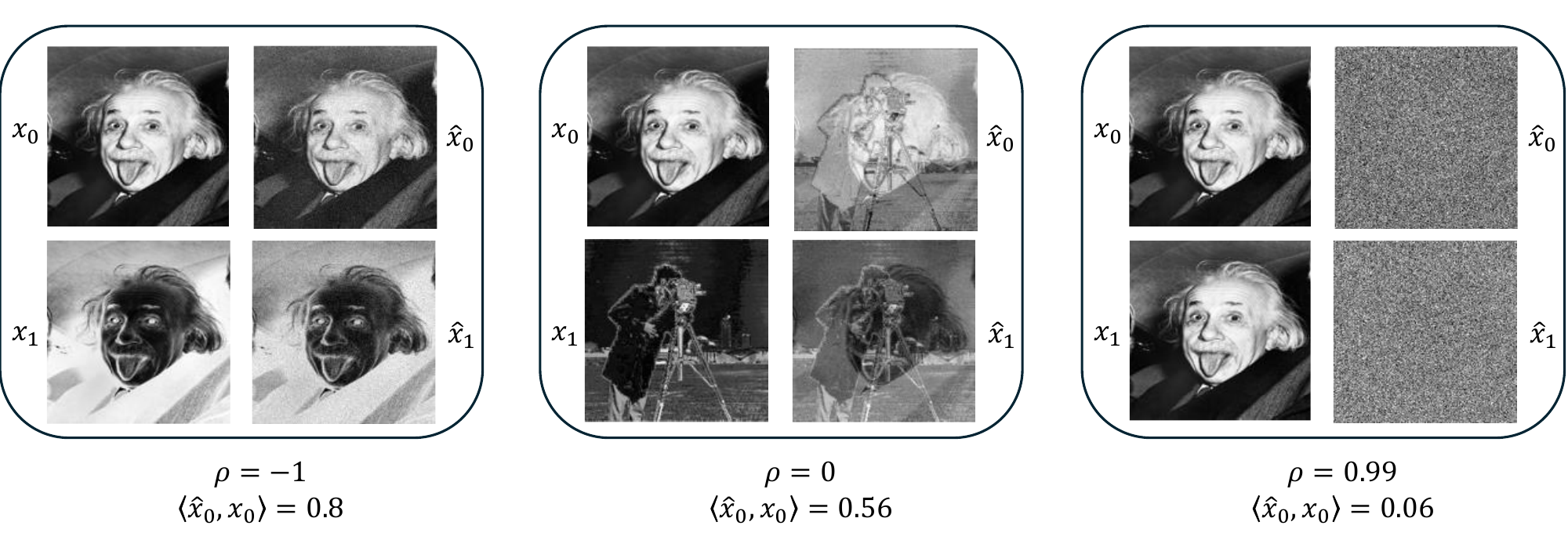}
	\caption{The hard-assignment estimator of Algorithm \ref{alg:generalizedEfNhard} for two templates $L=2$.  The right panel ($\rho = 0.99$) was generated by adding a small additive noise to the Einstein image. All template images are normalized to have unit norm. The inner product between the assignment estimators and the templates, $\langle{\hat{x}_{0}}, x_{0}\rangle$, is accurately predicted by Theorem~\ref{thm:hardAssignmentTwoHypoteses}. The experiments were conducted with $M = 5 \times 10^6$ observations, and dimension $d = 150 \times 150$.}
	\label{fig:2}
\end{figure}

Next, we move to the soft-assignment case. We prove the following result.

\begin{thm}[Soft-assignment for $L=2$] \label{thm:softAssignmentTwoTemplates}
Denote by $\hat{x}_0$ and $\hat{x}_1$ the output of Algorithm~\ref{alg:generalizedEfNsoft}. 
Then,
   \begin{align}
        \hat{x}_0 \to2\cdot\mathbb{E}\pp{\frac{n_1}{1+\exp\p{{\langle{n_1, x_1 - x_0\rangle}}}}},\label{eqn:softBefApp1}
    \end{align}
and
    \begin{align}
        \hat{x}_1 \rightarrow -2\cdot\mathbb{E}\pp{\frac{n_1}{1+\exp\p{{\langle{n_1, x_1 - x_0\rangle}}}}},\label{eqn:softBefApp2}
    \end{align}
almost surely, as $M\to\infty$.
\end{thm}

While the consequences of Theorem~\ref{thm:softAssignmentTwoTemplates} are less discernible than Theorem~\ref{thm:hardAssignmentTwoHypoteses}, in Appendix~\ref{eqn:softTwoAppro}, based on a standard approximation of the logistic function, we show that with the notation of ${\rho} \triangleq  {{\langle{{x}_{0}}, x_{1} \rangle} / \norm{x}_2^2}$, and the assumption that ${\rho} < 1$, we have, 
\begin{align}
       \hat{x}_0 \approx \frac{1}{2} \frac{x_0 - x_1}{\sqrt{1+\frac{\pi}{4}\p{1-\rho}\norm{x}_2^2}} \label{eqn:x0hatTwoTemplates},
    \end{align}
    and
    \begin{align}
        \hat{x}_1 \approx \frac{1}{2} \frac{x_1 - x_0}{\sqrt{1+\frac{\pi}{4}\p{1-\rho}\norm{x}_2^2}} \label{eqn:x1hatTwoTemplates},
    \end{align}
as $M\to\infty$. The absolute maximum error of this approximation is less than 0.02 \cite{crooks2009logistic}, and the approximation error vanishes as $\norm{x} \to \infty$.

 We note that Theorem \ref{thm:softAssignmentTwoTemplates} aligns with the results we obtained for the asymptotic cases of low and high SNR regimes in Proposition~\ref{thm:softAssignmentExtremeNormalization}. Specifically, we see that in the high SNR regime, when $\norm{x}_2^2\gg\frac{4}{\pi(1-\rho)}$, then \eqref{eqn:x0hatTwoTemplates}--\eqref{eqn:x1hatTwoTemplates} coincide with the hard-assignment result of Theorem~\ref{thm:hardAssignmentTwoHypoteses}, and when $\norm{x}_2 \to 0$, then \eqref{eqn:x0hatTwoTemplates}--\eqref{eqn:x1hatTwoTemplates} matches Proposition~\ref{thm:softAssignmentExtremeNormalization}.

\subsection{Finite number of templates}

We next consider the case of a finite number of templates $L<\infty$, and prove that the estimators are given by a linear combination of the $L$ templates, as $M\to\infty$. In contrast to the case of $L=2$, we do not derive a closed-form expression for the coefficients; instead, we provide an approximation.

\begin{thm}[Hard and Soft-assignment for finite $L$]\label{thm:hardAssignmentLinearCombination}
Fix $L \geq 2$. Denote by $\ppp{\hat{x}_\ell}_{\ell=0}^{L-1}$ the output of either Algorithm~\ref{alg:generalizedEfNhard} or Algorithm~\ref{alg:generalizedEfNsoft}. Assume that $x_{\ell_1} \neq x_{\ell_2}$ for every $\ell_1 \neq \ell_2$. Then, for every $0\leq\ell\leq L-1$,
        \begin{align}
           \hat{x}_\ell \rightarrow \sum_{k=0}^{L-1} \alpha_{\ell k} x_k,
        \end{align}
almost surely, as $M\to\infty$, for some coefficients $\ppp{\alpha_{\ell k}}_{\ell,k = 0}^{L-1} \in \mathbb{R}$.
\end{thm}

In Appendix~\ref{sec:expressionForLinearCoefficents}, we derive an  analytic expression for the coefficients $\ppp{\alpha_{\ell k}}_{\ell, k = 0}^{L-1}$, and show that,
\begin{align}
    \alpha_{k \ell} = \beta \p{\delta_{k \ell} - \frac{\mathbb{E}\pp{p_{1,\beta}^{(\ell)} \cdot p_{1,\beta}^{(k)}}}{\mathbb{E}\pp{p_{1,\beta}^{(\ell)}}} },
\end{align}
where $p_{1,\beta}^{(\ell)}$ is defined in \eqref{eqn:softMaxProbabilityDistWithBeta}. The soft-assignment estimators correspond to $\beta = 1$, while the hard-assignment estimators correspond to the limit $\beta \to \infty$.
Furthermore, in Appendix~\ref{sec:approximationForFiniteNumberOfTemplates}, we use a standard approximation of the expected value of the ratio between two random variables to show that the soft-assignment estimator can be approximated by an explicit expression,
    \begin{align}
            \hat{x}_\ell \approx x_\ell - \frac{1}{C_\ell} \sum_{r=0}^{L-1} x_r e^{{\,\langle{x_\ell}, x_r \rangle}},
            \label{eqn:approximationSoftForFiniteL}
        \end{align}
        as $M\to\infty$, where $C_\ell\triangleq L - \sum_{r=0}^{L-1}  e^{{\,\langle{x_\ell}, x_r \rangle}} + \frac{1}{L}\sum_{r_1,r_2=0}^{L-1}  e^{{\,\langle{x_{r_1}}, x_{r_2} \rangle}}$. 
This approximation shows that templates that exhibit a higher correlation with the $\ell$th template will tend to have a more significant contribution through the weight $e^{{\langle{{x}_{\ell}}, x_{r}\rangle}}$.

\subsection{Growing number of templates and dimension}

We now explore the behavior of the assignment estimators, when $d,L\to\infty$.
We show that the soft-assignment estimator converges to the corresponding template, while the hard-assignment estimator converges to the corresponding template, up to a scaling factor. 

For the hard-assignment procedure, we assume that the correlation between the various templates decays faster than a logarithmic factor. This is formulated as follows.
\begin{assum}\label{assump:1}
We say that the template signals $\ppp{x_\ell}_{\ell=0}^{L-1}$ satisfy Assumption~\ref{assump:1} if 
\begin{align}
    \max_{\ell_1\neq\ell_2\in[L]}{{{\langle{{x}_{\ell_1}}, x_{\ell_2}\rangle}} \cdot \log\p{\abs{\ell_1 - \ell_2}}} \to 0,\label{eqn:assumEQ}
\end{align}
as $d,L\to\infty$.
\end{assum}
While implicit in Assumption~\ref{assump:1}, it should be noted that for \eqref{eqn:assumEQ} to hold, the dimension $d$ must diverge as well. Intuitively speaking, if $d$ is fixed and, as an example, the templates are spread uniformly over the $d$-dimensional hypersphere, then it is clear that we cannot have growing number of templates that will also appear ``almost" orthogonal at the same time. Using the Johnson–Lindenstrauss lemma \cite{JohnsonLindenstrauss}, it is not difficult to argue that \eqref{eqn:assumEQ} induces an asymptotic relation between $d$ and $L$; indeed, \eqref{eqn:assumEQ} can hold when $d\gg\log L$. We have the following result.

\begin{thm}[Hard-assignment for $d,L\to\infty$]\label{thm:hardAssignmentAsymptoticLandAsymptoticD} Let $\ppp{\hat{x}_\ell}_{\ell=0}^{L-1}$ be the output of Algorithm~\ref{alg:generalizedEfNhard}, and $\s{\hat{R}}$ be defined as in \eqref{eqn:OptShiftRealSpace}. Then, under Assumption \ref{assump:1}, we have,
    \begin{align}
        \lim_{d,L \to\infty} \lim_{M \to\infty} \frac{1}{\sqrt{2 \log L}}  \sum_{\ell=0}^{L-1}\langle{\hat{x}_\ell},{x_\ell}\rangle \mathbb{P}[\hat{\s{R}} = \ell]  = 1, \label{eqn:hardAssignmentAsymptoticLandAsymptoticD1}
    \end{align}
almost surely. If, in addition, Assumption \ref{assump:0} holds, then,
    \begin{align}
        \lim_{d,L \to\infty} \lim_{M \to\infty} \frac{1}{\sqrt{2 \log L}}  \langle{\hat{x}_\ell},{x_\ell}\rangle  = 1, \label{eqn:hardAssignmentAsymptoticLandAsymptoticD2}
    \end{align}
almost surely, for every $\ell \in \mathbb{N}$.
\end{thm}

The proof of Theorem \ref{thm:hardAssignmentAsymptoticLandAsymptoticD} relies on certain results from the theory of extrema of Gaussian processes, in particular, the convergence of the maximum of a Gaussian process to the Gumbel distribution, see, e.g., \cite{leadbetter2012extremes}. To demonstrate Theorem~\ref{thm:hardAssignmentAsymptoticLandAsymptoticD}, we conducted the following experiment. We generated template signals (hypotheses) by
    \begin{align}
        {x}_\ell \triangleq U_\ell x_0, \label{eqn:generationOfXell}
    \end{align} 
    for $\ell\in[L]$, a fixed $x_0 \in \mathbb{R}^d$, and $L$ orthogonal matrices  $\ppp{U_\ell}_{\ell=0}^{L-1}$, drawn from a uniform (Haar) distribution. 
     Figure~\ref{fig:4} shows the convergence of the hard-assignment estimator, for large values of $d$ and $L$, in the regime where $\log L\ll d$. It can be seen that
     as $L$ and $d$ increase, the Pearson cross-correlation gets closer to unity, as our results predict. 

\begin{figure}[t!]
    \centering
    \includegraphics[width=0.9\linewidth]{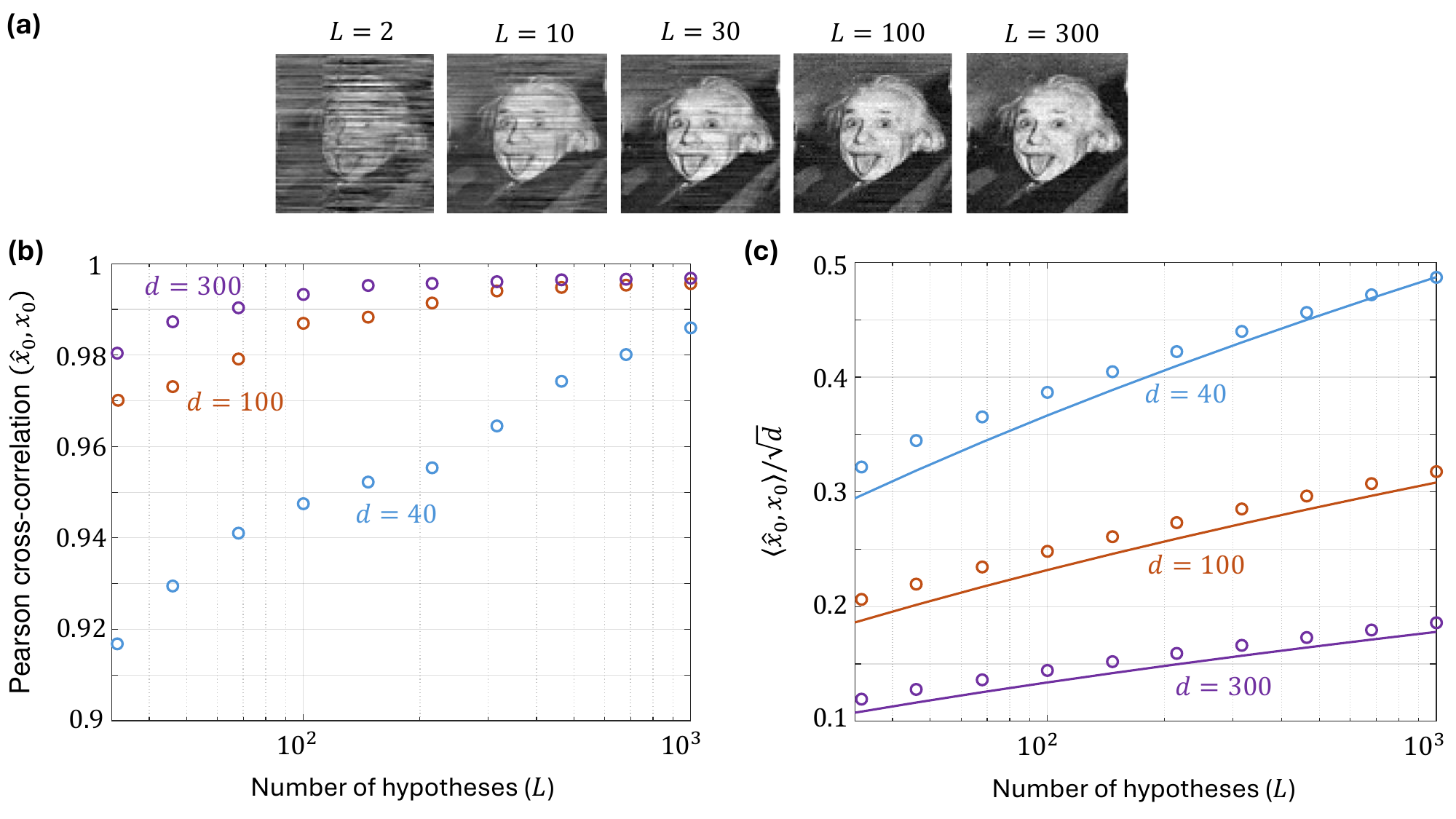}
    \caption{The convergence of the hard-assignment estimator in the high-dimensional regime $\p{L,d\gg1}$, where $\log L\ll d$. (a) A visual illustration shows the correlation between the initial template and its estimator, as $L$ increases. The initial template $x_0$ is an image of Einstein of dimensions $d = 75 \times 75$, while the other templates $\ppp{x_\ell}_{\ell\geq1}$ were generated according to the procedure in \eqref{eqn:generationOfXell}. (b)+(c) The Pearson cross-correlation and the inner product between the hard-assignment estimator $\hat{x}_0$ and the template $x_0$, which is modeled as a normalized exponential vector $x_0 \pp{m} = \exp \ppp{-\alpha \cdot m}$ for $0 \leq m \leq d-1$, and $\alpha = 1/30$, with varying dimension sizes $\p{d = 40, 100, 300}$. The additional templates $x_\ell$, for $1 \leq \ell \leq L-1$, were generated according to~\eqref{eqn:generationOfXell}. (b) The Pearson cross-correlation between $x_0$ and the estimator $\hat{x}_0$ is plotted as a function of $d$ and $L$, showing that the correlation approaches one as the number of hypotheses increases. Similar trends are observed for the correlation between other templates $\{x_\ell\}_{\ell\geq1}$ and their corresponding estimators $\{\hat{x}_\ell\}_{\ell\geq1}$. (c) The inner product between the hard-assignment estimator and the template is shown as a function of $d$ and $L$. Dots represent Monte Carlo simulations with $5 \times 10^6$ trials at each point, while the lines are given by $b_L/\sqrt{d}$, where $b_L$ is defined by \eqref{eqn:b_n}, which is asymptotically equivalent to \eqref{eqn:hardAssignmentAsymptoticLandAsymptoticD2} under Assumption \ref{assump:0}. Although the template vectors $\ppp{x_\ell}_{\ell=0}^{L-1}$ do not strictly satisfy Assumption \ref{assump:0}, they are empirically weakly correlated, leading to a covariance matrix close to the identity matrix. }
    \label{fig:4} 
\end{figure}

Next, we move to the soft-assignment procedure.
\begin{thm}[Soft-assignment for $d,L\to\infty$] \label{thm:softAssignmentAsymptoticL}
Denote by $\ppp{\hat{x}_\ell}_{\ell=0}^{L-1}$ the output of Algorithm~\ref{alg:generalizedEfNsoft}. Assume that ${{\langle{{x}_{\ell_1}}, x_{\ell_2}\rangle}} \to 0$, for $\abs{\ell_1 - \ell_2} \to \infty$, as $d,L \to \infty$. Then, 
        \begin{align}
           \hat{x}_\ell - x_\ell \rightarrow 0,
        \end{align}
in probability, for every $\ell \in \mathbb{N}$, as $M, d, L \to \infty$.
\end{thm}

The proof of Theorem \ref{thm:softAssignmentAsymptoticL} relies on Bernstein's law of large numbers for correlated sequences (see, for example, \cite{durrett2019probability}). 
Note that for the soft-assignment case, we only require that the cross-correlations vanish, without any restriction on the decay rate, whereas for the hard-assignment algorithm, the cross-correlation is required to decay faster than $\log \p{\abs{\ell_1 - \ell_2}}$. 

\section{Multi-iteration hard-assignment and soft-assignment} \label{sec:multiIteration}

So far, we have analyzed a single iteration of the $K$-means and EM algorithms. A key question that naturally arises is how the confirmation bias evolves as the number of iterations increases. This is particularly relevant in practical applications, where these algorithms typically run for many iterations. The fundamental question is whether these algorithms eventually forget their initializations (i.e., templates). Intuitively, the initial templates exert the greatest influence during the early iterations, with their impact expected to fade over time. However, as we demonstrate in this section, both theoretically and through numerical simulations, the bias persists even after numerous iterations, underscoring the enduring role of confirmation bias throughout the iterative process.

In the multi-iteration versions of the $K$-means and EM algorithms, the means estimators at iteration $t+1$, denoted by $\hat{x}_{\ell}^{(t+1)}$ for $\ell = 0,1,\dots, L-1$, are computed by updating the estimates from the previous iteration, $\hat{x}_{\ell}^{(t)}$, using either the hard-assignment approach (as outlined in Algorithm \ref{alg:generalizedEfNhardMultiIteration}) or the soft-assignment approach (Algorithm \ref{alg:generalizedEfNSoftMultiIteration}). We denote the total number of iterations by $T$. As in the single-iteration case, multi-iteration algorithms start with an initial set of template estimates, $\hat{x}_\ell^{(0)}$. A key question is how the final estimates after $T$ iterations, $\hat{x}_{\ell}^{(T)}$, relate to the initial hypotheses, $\hat{x}_\ell^{(0)}$ for $\ell \in \pp{L}$, when the observed data consists of noise only. As we demonstrate in this section, the bias becomes more pronounced as the number of observations $M$ and hypotheses $L$ increase. It should be emphasized here that in the single-iteration analysis (Algorithms \ref{alg:generalizedEfNhard} and \ref{alg:generalizedEfNsoft}), we have assumed that all templates are normalized. Under this assumption, minimizing the Euclidean distance $\norm{n_i - x_\ell}^2$ is equivalent to maximizing the inner product $\langle n_i, x_\ell \rangle$. However, in the general case, the estimates after the first iteration are not necessarily normalized, except in some special cases (e.g., when the cyclic symmetry in Assumption~\ref{assump:0} holds). As a result, the update rules in Algorithms \ref{alg:generalizedEfNhardMultiIteration} and \ref{alg:generalizedEfNSoftMultiIteration} consider the more general Frobenius distances instead of inner products.

\begin{algorithm}[t!]
  \caption{\texttt{$K$-means (Multi-iteration hard-assignment)} \label{alg:generalizedEfNhardMultiIteration}}
\textbf{Input:} Initial templates $\ppp{\hat{x}_\ell^{(0)}}_{\ell=0}^{L-1}$, observations $\ppp{n_i}_{i=0}^{M-1}\stackrel{\s{i.i.d.}}{\sim} \calN \p{0, I_{d \times d}}$, and number of iterations $T$.\\
\textbf{Output:} K-means estimates $\ppp{\hat{x}_\ell^{(T)}}_{\ell=0}^{L-1}$.\\
\textbf{Each Iteration, for $t=0\ldots,T-1$:}
\begin{enumerate}
    \item Initialize: $\calA_\ell^{(t)} \leftarrow\emptyset$, for $\ell = 0,1,\ldots,L-1$. 
    \item For $i=0,\ldots,M-1$, compute
     \begin{align}
        \s{\hat{R}}_i^{(t)}\triangleq \underset{0 \leq \ell \leq L-1}{\argmin} {\,   \norm{n_i - \hat{x}_\ell^{(t)}}_F^2},
    \label{eqn:OptShiftRealSpaceMulti}
    \end{align} 
    and add to the set $\calA_{\s{\hat{R}}_i}^{(t)}$ the noise observation $n_i$: $\calA_{\s{\hat{R}}_i}^{(t)} \leftarrow \calA_{\s{\hat{R}}_i}^{(t)} \cup \ppp{n_i}$.
\item  Compute for $0 \leq \ell \leq L-1$: 
    \begin{align}
        \hat{x}_\ell^{(t+1)}\triangleq \frac{1}{\abs{\calA_\ell^{(t)}}} \sum_{n_i \in \calA_\ell^{(t)}} n_i \label{eqn:generalizedEfNEqnMulti}.
    \end{align}
\end{enumerate}
\end{algorithm}

\begin{algorithm}[t!]
  \caption{\texttt{Expectation-maximization (Multi-iteration soft-assignment)} \label{alg:generalizedEfNSoftMultiIteration}}
\textbf{Input:} Initial templates $\ppp{\hat{x}_\ell^{(0)}}_{\ell=0}^{L-1}$, observations $\ppp{n_i}_{i=0}^{M-1}\stackrel{\s{i.i.d.}}{\sim} \calN \p{0, I_{d \times d}}$, and number of iterations $T$.\\
\textbf{Output:} Expectation-maximization estimates $\ppp{\hat{x}_\ell^{(T)}}_{\ell=0}^{L-1}$.\\
\textbf{Each Iteration, for $t=0\ldots,T-1$:}
\begin{enumerate}
    \item For $i=0,\ldots,M-1$, and $\ell = 0, \dots, L-1$, compute
    \begin{align}
        p_{i,t}^{(\ell)} \triangleq \frac{\exp \p{- \frac{1}{2} \norm{n_i - \hat{x}_\ell^{(t)}}_F^2 }}{\sum_{r=0}^{L-1} \exp \p{- \frac{1}{2} \norm{n_i - \hat{x}_r^{(t)}}_F^2 }}, \label{eqn:softMaxProbabilityDistNorm}
    \end{align}
\item  Compute for $0 \leq \ell \leq L-1$: 
    \begin{align} 
       \hat{x}_\ell^{(t+1)} = \frac{\sum_{i=0}^{M-1}{ p_{i,t}^{(\ell)}n_i}}{\sum_{i=0}^{M-1}{p_{i,t}^{(\ell)}}},\label{eqn:softAssignmentUpdateStepMulti} 
\end{align}
\end{enumerate}
\end{algorithm}

\subsection{Theoretical results}

In the previous sections, we analyzed the correlation between two consecutive iterations of the $K$-means and EM algorithms. To extend these results to multiple iterations, we analyze correlations across multiple iterations.
In general, correlation is not transitive, that is, given normalized vectors $u, v, w$ such that $\langle u,v \rangle > 0$ and $\langle v,w \rangle > 0$, it does not necessarily follow that $\langle u,w \rangle > 0 $. Nonetheless, we show that for strongly correlated vectors, as defined below, a certain property can be proved. The following result holds, in fact, for any set of iterative estimators $\{\hat{x}_\ell^{(t)}\}_{\ell\in[L],t\in\mathbb{N}}$ of templates $\{x_\ell\}_{\ell\in[L]}$.

\begin{proposition}
[Lower bound on confirmation bias] \label{thm:lowerBoundAfterTiterations}
Let $\{x_\ell\}_{\ell\in[L]}$ be a set of templates, and $\{\hat{x}_\ell^{(t)}\}_{\ell\in[L]}$ denote a set of estimators, for each iteration index $t\in[T]$. We assume that at $t=0$, the estimation procedure is initialized at $\hat{x}_\ell^{(0)} = x_\ell$, for $\ell\in[L]$. Assume $\| \hat{x}_\ell^{(t)}\|_2 = 1$, for all $\ell\in[L]$ and $t\in[T]$. We further assume that for all $t\in[T]$,
\begin{equation}
    \langle \hat{x}_\ell^{(t+1)}, \hat{x}_\ell^{(t)} \rangle \geq 1 - \epsilon,    
\end{equation}
where $0 < \epsilon < 2/T^2$. Then,
\begin{equation}
    \langle \hat{x}_\ell^{(T)}, \hat{x}_\ell^{(0)} \rangle \geq 1 - T^2 \epsilon. \label{eqn:lowerBoundMultipleIterations}
\end{equation}
\end{proposition}

This proposition suggests that if the correlation between every two consecutive iterations is at least $1 - \epsilon$, then after $T$ iterations, the correlation between $\hat{x}_\ell^{(T)}$ and the initial estimate remains lower-bounded by $1 - T^2 \epsilon$.
As we will demonstrate in Section \ref{sec:numericalSimulationMulti}, this scaling law of $1-T^2 \epsilon$ holds empirically.

\paragraph{Closed-form convergence.} There are two particular regimes of interest, in which $K$-means and EM converge after a single iteration to their final value, which is biased towards the initial hypotheses. First, we analyze the case of two templates in the hard assignment algorithm.
\begin{proposition}
\label{thm:LequalsTwoMultiIteration}
Fix $L=2$, and assume the conditions of Theorem \ref{thm:hardAssignmentTwoHypoteses} are satisfied.  Denote the initialization by $\hat{x}_\ell^{(0)} = x_\ell$ for $\ell = 0,1$. Assume $M \to \infty$ and $T<\infty$. Then, the $K$-means estimates converge to their final values after a single iteration, as given by the right-hand side of \eqref{eqn:hardAssignTwoHypoteses1} and \eqref{eqn:hardAssignTwoHypoteses2}.
\end{proposition}

Proposition \ref{thm:LequalsTwoMultiIteration} implies that in the case of two templates $L=2$ and an asymptotically large number of observations $M \to \infty$, the $K$-means algorithm converges  after a single iteration. These estimators, given by \eqref{eqn:hardAssignTwoHypoteses1} and \eqref{eqn:hardAssignTwoHypoteses2}, remain biased toward the initial templates.
We now turn our attention to the regime where $L \to \infty$, focusing on the soft-assignment algorithm.
\begin{proposition}
\label{thm:infiniteLmultiIteration}
Denote by $\hat{x}_\ell^{(t)}$ the $t$-th iteration of soft-assignment EM algorithm (Algorithm \ref{alg:generalizedEfNSoftMultiIteration}) and denote the initialization by $\hat{x}_\ell^{(0)}$. Assume $M, d, L \to \infty$, satisfying the conditions of Theorem \ref{thm:softAssignmentAsymptoticL}. Then, for $1 \leq T < \infty$ iterations, we have 
    \begin{align}
       \hat{x}_\ell^{(T)} - \hat{x}_\ell^{(0)} \rightarrow 0.
    \end{align}
\end{proposition}

Proposition \ref{thm:infiniteLmultiIteration} suggests that when the number of templates is asymptotically large and the conditions of Theorem \ref{thm:softAssignmentAsymptoticL} are satisfied, the EM algorithm converges to its final value after a single iteration, determined by the initial hypotheses.

\subsection{Numerical simulations} \label{sec:numericalSimulationMulti}

Figure \ref{fig:5} and Figure \ref{fig:6} illustrate the confirmation bias phenomenon of the $K$-means algorithm over multiple iterations. Specifically, Figure \ref{fig:5} demonstrates this bias in the context of 2D images. In these experiments, we use $L=96$ hypotheses, with the first 12 corresponding to the notable mathematicians depicted in Figure \ref{fig:1}. As observed, although the bias diminishes as iterations progress, it remains strongly correlated with the initial templates.

\begin{figure}[t]
    \centering
    \includegraphics[width=1.0\linewidth]{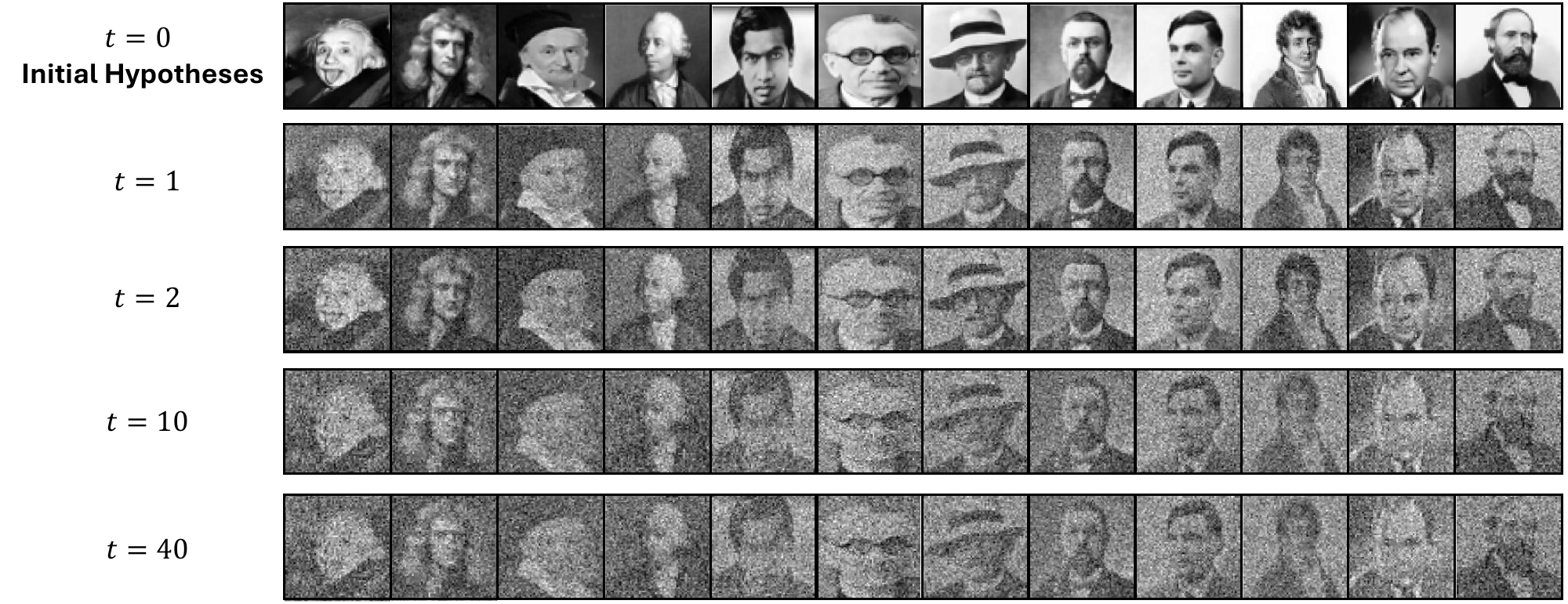}
    \caption{\textbf{Confirmation bias in the multi-iteration hard-assignment ($K$-means) algorithm.} The different rows illustrate the estimators as a function of the iteration number, $t$. The $K$-means algorithm converges to a local minimum after $t = 40$ iterations. The simulation parameters are as follows: the number of observations is $M = 2 \times 10^5$, the image size is $d = 60 \times 60$, and the number of templates is $L = 96$. Only the first 12 templates are shown, while the remaining templates were drawn from a uniform distribution over the sphere. As observed, confirmation bias reduces as the number of iterations increases, but it remains significant even after convergence to the local minimum ($t = 40$).}
    \label{fig:5} 
\end{figure}

Figure \ref{fig:6} illustrates the impact of the number of observations $M$ and the number of hypotheses $L$ on confirmation bias. In Figure \ref{fig:6}(a), we consider the case where $L = 2$ and examine how the number of observations $M$ influences the correlation between the estimator at the $t$-th iteration and the corresponding initial templates. As observed, larger $M$ leads to a stronger confirmation bias over a greater number of iterations, as predicted by Proposition \ref{thm:LequalsTwoMultiIteration}. In the initial iterations, the correlation follows the scaling law described in \eqref{eqn:lowerBoundMultipleIterations}. However, as $T \to \infty$, the estimators eventually converge to their local minima. Between these two extremes, there is a transition phase from the initial iteration to the asymptotic behavior.

In Figure \ref{fig:6}(b), we investigate the effect of the number of hypotheses $L$ on confirmation bias as a function of the number of iterations. The templates are generated based on the scheme presented in \eqref{eqn:generationOfXell}. As shown, in the initial iterations, there is a rapid transition toward the convergence region. As $L$ increases, the confirmation bias also increases and converges to larger correlation values, as expected from Proposition \ref{thm:infiniteLmultiIteration}.

\begin{figure}[t]
    \centering
    \includegraphics[width=1.0\linewidth]{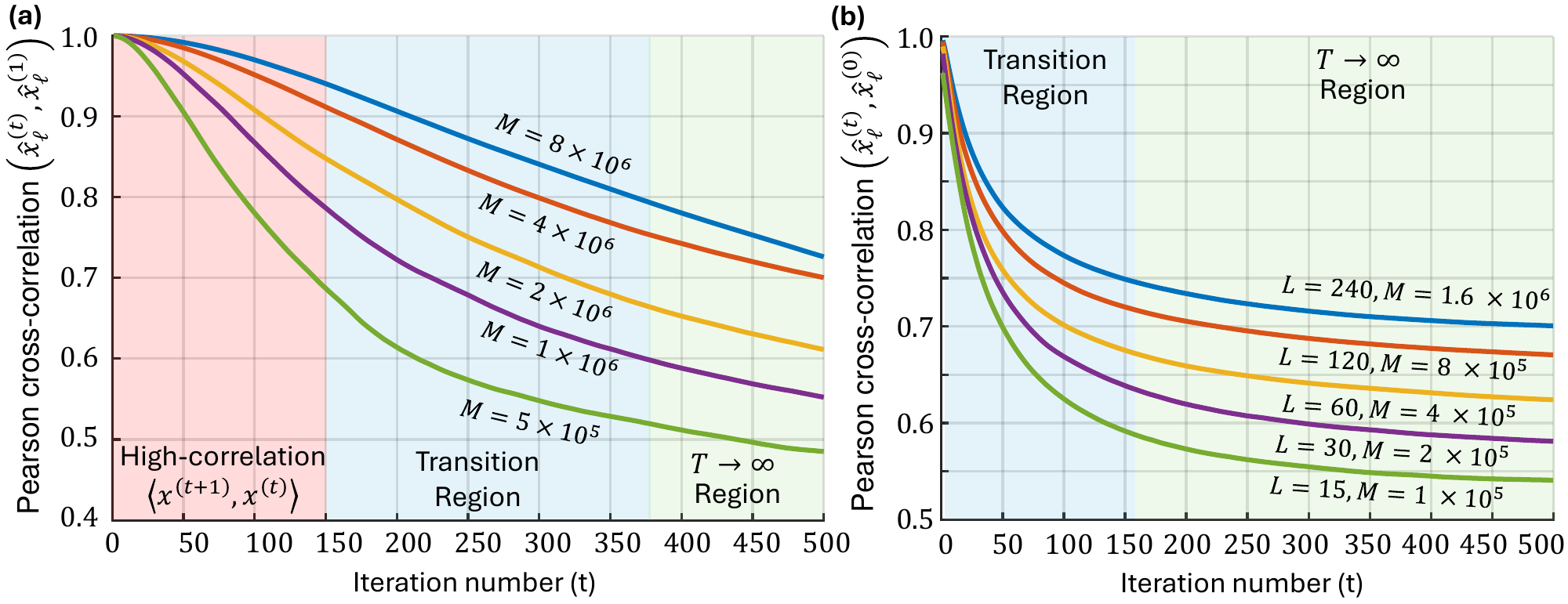}
    \caption{\textbf{Confirmation bias in the multi-iteration hard-assignment estimation process as a function of the number of observations $M $ and the number of templates $L$.} 
    \textbf{(a)} The correlation between $\hat{x}^{(t)}$, and $\hat{x}^{(1)}$ for two templates ($L = 2$) of the multi-iteration hard-assignment process is shown as a function of the iterations. In the initial iterations, the correlation follows the scaling law predicted by \eqref{eqn:lowerBoundMultipleIterations} as a function of $t$ and $M$. Following this phase, a transition region emerges, leading to the convergence of the algorithm toward local minima.  
    \textbf{(b)} The correlation between $\hat{x}^{(t)}$ and $\hat{x}^{(0)}$ of the multi-iteration hard-assignment process is shown as a function of the iteration number $t$ for different numbers of hypotheses. Each curve represents a different number of hypotheses, while ensuring that the ratio $M/L$ remains constant across curves so that each class, on average, contains the same number of observations. As observed, increasing the number of templates leads to a stronger correlation between the estimator and the initial templates.  
    The numerical experiments were conducted on a signal of dimension $d = 80$, with templates generated according to \eqref{eqn:generationOfXell}.}
    \label{fig:6} 
\end{figure}

\section{Perspective and future work}\label{sec:outlook}

In this paper, we have addressed the problem of confirmation bias in GMMs and examined its statistical properties. Our primary objective is to enhance our understanding of confirmation bias, particularly in scientific fields where observations exhibit a low SNR.
For instance, confirmation bias (also known as model bias) is a significant issue in single-particle cryo-electron microscopy—a leading technique for determining the spatial structure of biological molecules—where the data is often heavily contaminated by noise~\cite{henderson2013avoiding,bendory2020single}.

We next delineate important future extensions of our results.

\paragraph{Mixture models.} 
In this study, we have concentrated on a specific GMM with known variances and weights. While this approach is instructive, it does not fully capture the flexibility of GMMs. In future work, we plan to extend our analysis to more general GMMs to investigate whether the initial templates influence not only the means (as demonstrated in this study) but also the variances and other parameters.
Additionally, we aim to explore the presence of confirmation bias in other types of mixture models. While GMMs are associated with the Euclidean distance metric, other distributions correspond to a variety of metrics, necessitating further analysis.

\paragraph{Multi-iteration $K$-means and expectation-maximization.}  
This work focuses on the iterative algorithms $K$-means (hard-assignment) and EM (soft-assignment). While most of our analysis concerned a single iteration of these algorithms, Section \ref{sec:multiIteration} extended the discussion to the multi-iteration case, both theoretically and empirically, for specific scenarios. 
Several open questions remain regarding the extension to the multi-iteration setting, which we leave for future work. In particular, under the model considered in this study (with pure noise observations), what are the asymptotic values of the mean estimators as $T \to \infty$? How do these values depend on the number of observations $M$, the number of templates $L$, and the correlations between the templates?
Another key area of interest is the behavior of these algorithms in low SNR environments, where observations are generated as a combination of a signal and high levels of noise (as opposed to pure noise as in this work). In such environments, the algorithms tend to exhibit bias, with most observations being assigned to only a few clusters (Gaussian components), while other clusters remain nearly empty. This phenomenon is commonly referred to as the ``rich get richer" effect~\cite{sorzano2010clustering}.

\paragraph{Models with algebraic structure.}
In certain applications, such as multi-reference alignment \cite{bendory2017bispectrum,perry2019sample,bandeira2023estimation,bendory2024sample}, the templates follow an underlying algebraic structure, such as translations or rotations (see, e.g.,~\cite{bendory2020single,balanov2024einstein}). In this setting, all Gaussian centroids are derived from a single underlying template, transformed according to a specific algebraic structure (e.g., an orbit of a group).
A key application of such models is the computational problem of single-particle cryo-electron microscopy (cryo-EM)—a leading technology for reconstructing the 3D spatial structure of biological molecules. It is well established that the output of the EM and $K$-means algorithms in cryo-EM is significantly influenced by initialization~\cite{singer2020computational, bendory2020single}. This has led to extensive research into robust initialization methods, often referred to as ab initio algorithms, which employ various techniques such as the method of moments \cite{levin20183d}, common lines \cite{singer2011three}, and frequency marching \cite{barnett2017rapid}. However, none of these studies analyze the inherent bias of $K$-means and EM towards their initial points in the general case.

Our results for GMMs without structural constraints naturally extend to settings like multi-reference alignment and cryo-EM. However, we expect that the algebraic properties of the observations will introduce additional distinct characteristics in the induced conformation bias, analogous to how estimation rates differ between unconstrained GMMs and GMMs with algebraic constraints (see, e.g.,~\cite{bandeira2023estimation,bendory2017bispectrum,perry2019sample,bendory2024sample}). Investigating the role of algebraic structures, particularly group actions, in conformation bias is anticipated to provide fundamental insights into cryo-EM, where conformation bias remains a significant challenge, closely linked to the ``Einstein from noise" problem~\cite{balanov2024einstein}.

\paragraph{Confirmation bias in Bayesian settings.} A possible extension of this work involves analyzing confirmation bias within a Bayesian framework. For example, instead of assuming purely noisy observations, we could model the observations as being generated with probability $1-\epsilon$ from pure noise and with probability $\epsilon$ from the $L$ templates. In this setting, the key question is how the mean estimators would correlate with the initial templates, given that the true underlying statistics of the observations follow this mixture model rather than pure noise.

\paragraph{Connection to information-theoretic metrics and extension to other noise models.}
As discussed in Section~\ref{sec:2_2}, the geometric conditions in our framework, formulated in terms of inner products, can be equivalently expressed using the KL divergence. For instance, conditions of the form $\langle x_{\ell_1}, x_{\ell_2} \rangle \leq \langle y_{\ell_1}, y_{\ell_2} \rangle$ translate into reversed inequalities between the KL divergences of the corresponding Gaussian components. This equivalence highlights the potential to generalize our framework beyond isotropic Gaussian settings by replacing inner-product-based assumptions with divergence-based ones. A promising direction for future work is to explore whether this correspondence persists in the presence of non-isotropic Gaussian noise, where KL divergence or alternative information-theoretic measures may offer a more suitable basis for analyzing confirmation bias and estimation accuracy.

\paragraph{Finite-sample extensions.}
This work focuses on the asymptotic regime in which the number of observations $M \to \infty$. A compelling direction for future research is to extend the analysis to the finite-sample setting. In particular, it would be valuable to characterize the minimum number of observations required to detect confirmation bias with a prescribed confidence level $1 - \alpha$, drawing connections to hypothesis testing and information-theoretic sample complexity. Addressing this question would require more refined probabilistic tools, such as Berry--Esseen-type inequalities or modern non-asymptotic concentration techniques. We leave this as an open and promising direction for future investigation.

\section*{Acknowledgment}
T.B. is supported in part by BSF under Grant 2020159, in part by NSF-BSF under Grant 2019752, in part by ISF under Grant 1924/21, and in part by a grant from The Center for AI and Data Science at Tel Aviv University (TAD).
W.H. is supported by ISF under Grant 1734/21. 

\bibliographystyle{plain}


\begin{appendices}

{\centering{\section*{Appendix}}}

\section{Preliminaries}\label{sub:Preliminaries}
Before delving into the proofs, we start with a few definitions and auxiliary results which will aid our main derivations. We use $\xrightarrow[]{\calD}$, $\xrightarrow[]{\calP}$, and $\xrightarrow[]{\s{a.s.}}$, to denote the convergence of sequences of random variables in distribution, in probability, and almost surely, respectively. We use the indicator function $\mathbbm{1}_{A}\p{x}$ to indicate that $\mathbbm{1}_{A}\p{x} = 1$ if and only if $x \in A$. 

\subsection{The Gaussian vector induced by the templates}
Define the $L$-dimensional random vector $\s{S}_i^{(x)}$ as,
\begin{align}
    {S}_{i, \ell}^{(x)} \triangleq \frac{\langle{n_i, x_{\ell}\rangle}}{\norm{x_\ell}_2}\label{eqn:SxDef},
\end{align}
for $\ell \in \pp{L}$. For simplicity, we will omit the index $i$ in ${S}_{i, \ell}^{(x)}$ when the dependence is clear from the context. By our model assumptions, it is clear that $\s{S}^{(x)}$ is a zero-mean Gaussian random vector with covariance matrix $\Sigma_{\s{S}}$, whose entries are,
\begin{align}
 [\Sigma_{\s{S}}]_{\ell_1+1,\ell_2+1}  = \sigma_{\ell_1,\ell_2}^{(x)} = \frac{\langle{ x_{\ell_1}, x_{\ell_2}\rangle}}{\norm{x}_2^2},
\end{align}
for $\ell_1,\ell_2\in[L]$. Since we always assume that $x_{\ell_1} \neq x_{\ell_2}$, whenever $\ell_1\neq\ell_2$, it follows that $\Sigma_{\mathrm{S}}$ is positive definite. 
Next, let us denote the function $p_{\ell, \beta}:\mathbb{R}^L \to \mathbb{R}$, parameterized by $\beta\in\mathbb{R}_+$, as follows,
\begin{align}
   p_{\ell,\beta}\p{\s{Q}} \triangleq \frac{ \exp{\p{\beta Q_\ell}}}{\sum_{r=0}^{L-1}\exp{\p{\beta Q_r}}}, \label{eqn:softMaxFuncDef1}
\end{align}
for $\ell\in[L]$, where $\s{Q} = (Q_0,Q_1,\ldots,Q_{L-1})\in\mathbb{R}^L$. The function $p_{\ell,\beta}$ is known as the \emph{softmax function}. Finally, we define the log-sum-exp function  $F_\beta: \mathbb{R}^L \to \mathbb{R}$, as follows,
    \begin{align}
        F_\beta \p{\s{Q}} \triangleq \frac{1}{\beta}\log \p{\sum_{\ell=0}^{L-1} \exp\p{\beta Q_\ell}}. \label{eqn:logSumExpDef}
    \end{align}

\subsection{Asymptotic number of observations}
In this subsection, we find an asymptotic expression for the hard-assignment and soft-assignment estimators, as $M \to \infty$. We denote the set of all observations by $\calN \triangleq \ppp{n_i}_{i=0}^{M-1} \subset \mathbb{R}^d$. 

\paragraph{Hard-assignment.} For the hard-assignment estimator, it turns out that it is much more convenient to work with the following set of Voronoi regions, compared to $\{\calA_\ell\}_{\ell\geq0}$, as defined in Algorithm \ref{alg:generalizedEfNhard}.
Specifically, given the set of templates $\calX\triangleq\{x_\ell\}_{\ell\geq0}$, for every $\ell \in \pp{L}$, define,
    \begin{align}
        \nonumber \calV_\ell & \triangleq \ppp{
        v \in \mathbb{R}^d:  { \norm{v-x_\ell}_2^2 \leq \underset{k \neq \ell} {\min} \norm{v-x_k}_2^2 } } \\ & =  \ppp{
        v \in \mathbb{R}^d: {{\langle{v, x_\ell\rangle}} \geq \underset{k \neq \ell} \max {\langle{v, x_k\rangle}} } }\label{eqn:VlDef}.
    \end{align}
Note that $\{\calV_\ell\}_{\ell\in[L]}$ are deterministic, and they form a partition of $\mathbb{R}^d$. Furthermore, it is clear that $\calA_\ell \subset \calV_\ell$, for every $l \in \pp{L}$, $\cup_{\ell=0}^{L-1} \calA_\ell = \calN$, and $\calA_k \cap \calV_\ell = \emptyset$, for $k \neq \ell$. 
Following these properties, it follows that $\mathbbm{1}_{\ppp{n_i \in \calA_\ell}} = \mathbbm{1}_{\ppp{n_i \in \calV_\ell}}$, for every observation $n_i \in \calN$.

Let us see the implications of the above definitions. 
Recall the hard-assignment estimator in Algorithm~\ref{alg:generalizedEfNhard}. From \eqref{eqn:generalizedEfNEqn}, we have,
    \begin{align}
        {\langle{\hat{x}_\ell}, x_\ell\rangle} &= \frac{1}{\abs{\calA_\ell}} \sum_{n_i \in \calA_\ell} {\langle{n_i}, x_\ell\rangle} \\
        &= \frac{1}{\abs{\calA_\ell}} \sum_{n_i \in \calA_\ell} {\underset{0 \leq r \leq L-1} \max {\langle{n_i, x_r\rangle}}}\\
        &= \frac{1}{\abs{\calA_\ell}} \sum_{i=0}^{M-1} \pp{{\underset{0 \leq r \leq L-1} \max {\langle{n_i, x_r\rangle}}} \mathbbm{1}_{\ppp{n_i \in \calV_\ell}}}, \label{eqn:hardAssignmentAsymptotic1}
    \end{align}
where the second equality follows from the definition of $\calA_\ell$ in \eqref{eqn:OptShiftRealSpace}, and the third equality is from the definition of $\calV_\ell$. 
Then, we have the following Lemma:

\begin{lem}[Convergence of the hard-assignment estimator as $M \to \infty$]\label{lem:convergenceOfHardAssign}
As $M \to \infty$, the hard-assignment estimator converges almost surely to 
\begin{align}
    \hat{x}_\ell \xrightarrow[]{\text{a.s.}} \frac{\mathbb{E}\left[n_1 \, \mathbbm{1}_{\{n_1 \in \mathcal{V}_\ell\}}\right]}{\mathbb{P}\left(n_1 \in \mathcal{V}_\ell\right)}. \label{eqn:hardAsymptoticObservationsEstimatorMain}
\end{align}
Moreover, the confirmation bias between the hard-assignment estimator $\hat{x}_\ell$ and the corresponding template $x_\ell$ converges almost surely as $M \to \infty$ to
    \begin{align}
        {\langle{\hat{x}_\ell}, x_\ell\rangle} \xrightarrow[]{\s{a.s.}} \frac{\mathbb{E}\pp{\langle{n_1, x_\ell\rangle} \mathbbm{1}_{\ppp{n_1 \in \calV_\ell}}}}{\mathbb{P}\pp{n_1 \in \calV_\ell}} = \frac{\mathbb{E}\pp{{\underset{0 \leq r \leq L-1} \max {\langle{n_1, x_r\rangle}}} \mathbbm{1}_{\ppp{n_1 \in \calV_\ell}}}}{\mathbb{P}\pp{n_1 \in \calV_\ell}}. \label{eqn:hardAsymptoticObservaionsCorrelationMain}
    \end{align}
\end{lem}

\begin{proof}[Proof of Lemma~\ref{lem:convergenceOfHardAssign}]
Rearranging the last term in \eqref{eqn:hardAssignmentAsymptotic1}, we have,
    \begin{align}
        \nonumber {\langle{\hat{x}_\ell}, x_\ell\rangle} & = \frac{1}{\abs{\calA_\ell}} \sum_{i=0}^{M-1} \pp{{\underset{0 \leq r \leq L-1} \max {\langle{n_i, x_r\rangle}}} \mathbbm{1}_{\ppp{n_i \in \calV_\ell}}}\\
        &= \frac{\frac{1}{M} \sum_{i=0}^{M-1} \pp{{\underset{0 \leq r \leq L-1} \max {\langle{n_i, x_r\rangle}}} \mathbbm{1}_{\ppp{n_i \in \calV_\ell}}}}{\abs{\calA_\ell}/M}. \label{eqn:hardAssignmentAsymptotic2}
    \end{align}
Now, by the strong law of large numbers (SLLN) \cite{durrett2019probability}, we have,
    \begin{align}
        \frac{1}{M} \sum_{i=0}^{M-1} \pp{{\underset{0 \leq r \leq L-1} \max {\langle{n_i, x_r\rangle}}} \mathbbm{1}_{\ppp{n_i \in \calV_\ell}}} \xrightarrow[]{\s{a.s.}} \mathbb{E}\pp{{\underset{0 \leq r \leq L-1} \max {\langle{n_1, x_r\rangle}}} \mathbbm{1}_{\ppp{n_1 \in \calV_\ell}}},
    \end{align}
as $M\to\infty$. In addition, by the definition of $\calA_\ell$ in \eqref{eqn:OptShiftRealSpace}, we have,
    \begin{align}
        \frac{\abs{\calA_\ell}}{M} = \frac{1}{M} \sum_{i=0}^{M-1} \mathbbm{1}_{\ppp{n_i \in \calV_\ell}} \xrightarrow[]{\s{a.s.}} \mathbb{E}\pp{\mathbbm{1}_{\ppp{n_1 \in \calV_\ell}}} = \mathbb{P}\pp{n_1 \in \calV_\ell},
    \end{align}
as $M\to\infty$, where the almost sure convergence follows from the SLLN. Since both the numerator and denominator in \eqref{eqn:hardAssignmentAsymptotic2} converge almost surely, and the denominator converges to a positive number, it follows by the continuous mapping theorem \cite{mann1943stochastic, durrett2019probability} that,
    \begin{align}
        {\langle{\hat{x}_\ell}, x_\ell\rangle} \xrightarrow[]{\s{a.s.}} \frac{\mathbb{E}\pp{\langle{n_1, x_\ell\rangle} \mathbbm{1}_{\ppp{n_1 \in \calV_\ell}}}}{\mathbb{P}\pp{n_1 \in \calV_\ell}} = \frac{\mathbb{E}\pp{{\underset{0 \leq r \leq L-1} \max {\langle{n_1, x_r\rangle}}} \mathbbm{1}_{\ppp{n_1 \in \calV_\ell}}}}{\mathbb{P}\pp{n_1 \in \calV_\ell}} \label{eqn:hardAsymptoticObservaionsCorrelation}, 
    \end{align}
as $M\to\infty$. Using the same arguments as above, it can be shown that,
    \begin{align}
        \hat{x}_\ell \xrightarrow[]{\s{a.s.}} \frac{\mathbb{E}\pp{n_1\mathbbm{1}_{\ppp{n_1 \in \calV_\ell}}}}{\mathbb{P}\pp{n_1 \in \calV_\ell}},\label{eqn:hardAsymptoticObservaionsEstimator} 
    \end{align}
again, as $M\to\infty$. 
\end{proof}

\paragraph{Soft-assignment.} We move forward to the soft-assignment estimator in Algorithm~\ref{alg:generalizedEfNsoft}. Using \eqref{eqn:softAssignmentUpdateStep}, and following the same steps as for the hard-assignment estimator, we get,
\begin{align}
    \hat{x}_\ell = \frac{\sum_{i=0}^{M-1}{ p_{i}^{(\ell)}n_i}}{\sum_{i=0}^{M-1}{p_{i}^{(\ell)}}} \xrightarrow[]{\s{a.s.}} \frac{\mathbb{E}\pp{n_1 p_1^{(\ell)}}}{\mathbb{E}[p_1^{(\ell)}]} ,\label{eqn:softAsymptoticObservaionsEstimator}
\end{align}
as $M\to\infty$. Similarly, it follows that the correlation between the soft-assignment estimator and its corresponding template is,
\begin{align}
    {\langle{\hat{x}_\ell}, x_\ell\rangle} \xrightarrow[]{\s{a.s.}} \frac{\mathbb{E}\pp{{\langle{n_1}, x_\ell\rangle} p_1^{(\ell)}}}{\mathbb{E}[p_1^{(\ell)}]} \label{eqn:softAsymptoticObservaionsCorrelation},
\end{align}
as $M\to\infty$.

\subsection{The high-SNR regime} \label{sec:highSnrRegime}
The following result studies the convergence of $p_{\ell,\beta}$ in \eqref{eqn:softMaxFuncDef1}, as $\beta \to \infty$.
\begin{proposition} \label{prop:highSnrRegime}
    Let $p_{\ell,\beta}$ be as defined in \eqref{eqn:softMaxFuncDef1}, and $\s{S}^{(x)}$ in \eqref{eqn:SxDef}. Then,
    \begin{align}
        \lim_{\beta \to \infty} p_{\ell,\beta}(\s{S}^{(x)})  = \mathbbm{1}_{\ppp{n_i \in \calV_\ell}} \label{eqn:piDistForLargeNormalization},
    \end{align}
almost surely, where $\calV_\ell$ is defined in \eqref{eqn:VlDef}.
\end{proposition}
Before we prove the above result, it is useful prove the following auxiliary lemma.  
\begin{lem} \label{lem:asymptoticBetaRatio}
    Let $p_{\ell,\beta}$ be as defined in \eqref{eqn:softMaxFuncDef1}. Let $\s{Q} = (Q_0,Q_1,\ldots,Q_{L-1}) \in \mathbb{R}^L$ and assume that $\underset{0 \leq k \leq L-1} {\max}{Q_k}$ is unique. Then, 
    \begin{align}
        \lim_{\beta \to \infty} p_{\ell,\beta} \p{\s{Q}} = \mathbbm{1}{\ppp{Q_\ell = \underset{0 \leq k \leq L-1} {\max} {Q_k}}}. \label{eqn:convergenceForAsymptoticBetaToIndicator}
    \end{align}
    
\end{lem}
\begin{proof}[Proof of Lemma~\ref{lem:asymptoticBetaRatio}]
    Recall the definition of $F_\beta \p{\s{Q}}$ in \eqref{eqn:logSumExpDef}. It is well known that \cite{calafiore2014optimization},
        \begin{align}
        \underset{0 \leq k \leq L-1} {\max} Q_k \leq F_\beta \p{\s{Q}} \leq \underset{0 \leq k \leq L-1} {\max} Q_k + \frac{\log L}{\beta}. \label{eqn:FbetaBounds}
    \end{align}
    Next, note that,
    \begin{align}
        p_{\ell,\beta} \p{\s{Q}} &= \exp \pp{\beta Q_\ell - \log \p{\sum_{r=0}^{L-1} \exp\p{\beta Q_r}}}. \label{eqn:representationOfSoftMaxByLogSumExp}
    \end{align}  
    Then, combining \eqref{eqn:FbetaBounds} into \eqref{eqn:representationOfSoftMaxByLogSumExp} leads to,
    \begin{align}
        { Q_\ell - \underset{0 \leq k \leq L-1} {\max} {Q_k} - \frac{\log L}{\beta}} \leq \frac{1}{\beta} \log {p_{\ell,\beta}\p{\s{Q}}}
        \leq { Q_\ell - \underset{0 \leq k \leq L-1} {\max} {Q_k}}.\label{eqn:FbetaBounds2}
    \end{align}
    Taking $\beta \to \infty$, we get,
    \begin{align}
         \lim_{\beta \to \infty} \frac{1}{\beta} \log {p_{\ell,\beta}\p{\s{Q}}} = Q_\ell - \underset{0 \leq k \leq L-1} {\max} {Q_k}. \label{eqn:convergenceForAsymptoticBetaAsDifference}
    \end{align}
    Now, if $Q_\ell \neq \underset{0 \leq k \leq L-1} {\max} {Q_k} $, then clearly,
        \begin{align}
         \lim_{\beta \to \infty}  p_{\ell,\beta}\p{\s{Q}} = 0. \label{eqn:convergenceToZeroOfnotMaximum}
    \end{align}
   Conversely, if $Q_\ell = \underset{0 \leq k \leq L-1} {\max} {Q_k}$, and is unique, then must have that,
   \begin{align}
         \lim_{\beta \to \infty}  p_{\ell,\beta}\p{\s{Q}} = 1,
    \end{align}
    which concludes the proof.
\end{proof}

\begin{proof} [Proof of Proposition~\ref{prop:highSnrRegime}]
Using Lemma \ref{lem:asymptoticBetaRatio}, we can infer that for every realization of $n_i$, such that $\underset{0 \leq k \leq L-1} {\max} {\langle{n_i}, x_k \rangle} / \norm{x}_2$ is unique, we get,
    \begin{align}
        \lim_{\beta \to \infty}  \frac{\exp \p{\beta \frac{{\,\langle{n_i}, x_\ell \rangle}}{\norm{x}_2}}}{\sum_{r=0}^{L-1} \exp \p{\beta\frac{{\,\langle{n_i}, x_r \rangle}}{\norm{x}_2}}} 
        = \mathbbm{1}_{\ppp{n_i \in \calV_\ell}}, \label{eqn:piDistForLargeNormalization2}
    \end{align}
where $\calV_\ell$ is defined in \eqref{eqn:VlDef}. Since the maximum of a Gaussian vector with positive definite covariance matrix is unique with probability one, it follows that \eqref{eqn:piDistForLargeNormalization2} holds almost surely, which proves the desired result.

\end{proof}

\subsection{Soft-assignment positive correlation}
The following result will be used for proving that the correlation between the soft-assignment estimator and its corresponding template is positive (Theorem \ref{thm:hardAssignmentPositiveCorrelation}, second part, for Algorithm \ref{alg:generalizedEfNsoft}).
\begin{proposition}\label{prop:softAssignemtPositivePreli}
Let ${\s{X}} = \p{X_0, X_1,\ldots, X_{L-1}}^T$ be a zero-mean Gaussian vector, with ${\s{X}} \sim \calN \p{0, \Sigma_{\s{X}}}$, where $\mathbb{E}\p{X_i^2} = 1$, for $0 \leq i \leq L-1$. Then, for every $0 \leq l \leq L-1$ and $\beta>0$,
    \begin{align}
       \mathbb{E} \pp{X_\ell \cdot p_{\ell,\beta}\p{\s{X}}} \geq 0 \label{eqn:prop3}.
    \end{align}
If the off-diagonal covariance entries satisfy $\sigma_{ij}\triangleq[\Sigma_{\s{X}}]_{ij}< 1$, for every $i \neq j$, then the inequality in \eqref{eqn:prop3} is strict.
\end{proposition}

To prove proposition \ref{prop:softAssignemtPositivePreli}, we use the following known Gaussian integration by parts lemma \cite[Exercise 13.3]{boucheron2013concentration},\cite[Lemma 2.1]{ross2011fundamentals}.

\begin{lem} \label{lemma:3} Let $F: \mathbb{R}^n \to \mathbb{R}$ be a $C^1$ function, and ${\s{X}} = \p{X_0, X_1,\ldots, X_{n-1}}^T$ be a zero-mean Gaussian random vector, such that for any $a > 0$, $\lim_{\|\s{X}\| \to \infty} F(\s{X}) \exp \p{-a \|\s{X}\|^2} =0$. Then, for every $0 \leq i \leq n-1$,
    \begin{align}
       \mathbb{E} \pp{X_i F\p{\s{X}}} = \sum_{j=0}^{n-1}\mathbb{E}\p{X_iX_j} \mathbb{E}\p{\frac{\partial F}{\partial x_j} \p{\s{X}}}.
    \end{align}
\end{lem}

\begin{proof}[Proof of Proposition~\ref{prop:softAssignemtPositivePreli}]
By Lemma~\ref{lemma:3}, we have
\begin{align}
   \mathbb{E}\left[X_\ell \cdot p_{\ell,\beta}(\s{X})\right] = \sum_{k=0}^{L-1} \mathbb{E}\left[X_\ell X_k\right] \mathbb{E}\left[\frac{\partial p_{\ell,\beta}}{\partial X_k}(\s{X})\right], \label{eqn:gaussianIntegByPartsForPl}
\end{align}
where we note that $p_{\ell,\beta}$ satisfies the conditions required by Lemma~\ref{lemma:3}. Specifically, since $0 \leq p_{\ell,\beta}(\s{X}) \leq 1$ for all \( \s{X} \in \mathbb{R}^L \), and $p_{\ell,\beta}$ is continuously differentiable, it follows that
\begin{align}
    \lim_{\s{X} \to \infty} p_{\ell,\beta}(\s{X}) \exp(-a \|\s{X}\|^2) = 0,
\end{align}
for any $a > 0$. Therefore, the conditions of Lemma~\ref{lemma:3} are satisfied, and the integration by parts identity applies. 
It is clear that,
    \begin{align}
       \frac{\partial p_{\ell,\beta}}{\partial x_k}(\s{X}) = \beta\p{-p_{\ell,\beta}(\s{X}) p_{k,\beta}(\s{X}) + \delta_{\ell k}p_{\ell,\beta}(\s{X})} \label{eqn:derivativeOfPl},
    \end{align}
where $\delta_{kn}=1$ if $k=n$, and zero otherwise. Thus, substituting \eqref{eqn:derivativeOfPl} in \eqref{eqn:gaussianIntegByPartsForPl} leads to,
    \begin{align}
       \nonumber \frac{1}{\beta} \mathbb{E}\pp{X_\ell \cdot p_{\ell,\beta}\p{X}} & = \sum_{k=0}^{L-1}\mathbb{E}\p{X_\ell X_k}\mathbb{E}\p{-p_{\ell,\beta}(\s{X}) p_{k,\beta}(\s{X}) + \delta_{\ell k}p_{\ell,\beta}(\s{X})} \\ \nonumber &  = \mathbb{E}\p{X_\ell^2}\mathbb{E}\p{p_{\ell,\beta}(\s{X})} - \sum_{k=0}^{L-1}\mathbb{E}\p{X_\ell X_k}\mathbb{E}\p{p_{\ell,\beta}(\s{X}) p_{k,\beta}(\s{X})} \\ & = \mathbb{E}\p{p_{\ell,\beta}(\s{X})} - \sum_{k=0}^{L-1}\mathbb{E}\p{X_\ell X_k}\mathbb{E}\p{p_{\ell,\beta}(\s{X}) p_{k,\beta}(\s{X})} \label{eqn:gaussianIntegByPartsExplicit}.
    \end{align}
Since $\sum_{k=0}^{L-1}p_{k,\beta}(\s{X}) = 1$, $p_{k,\beta}(\s{X}) \geq 0$ for every $0 \leq k \leq L-1$, and $\mathbb{E}\p{X_lX_k} \leq 1$ for every $0 \leq \ell,k \leq L-1$, it follows from \eqref{eqn:gaussianIntegByPartsExplicit} that,
    \begin{align}
       \mathbb{E}\pp{X_\ell\cdot  p_{\ell,\beta} \p{\s{X}}} \geq 0,
    \end{align}
as claimed. 

When all off-diagonal entries of $\Sigma_{\s{X}}$ are strictly less than unity, i.e., $\mathbb{E}\p{X_\ell X_k} < 1$, for $\ell \neq k$, then for every $k \in \pp{L}$, we have $0 < p_{k,\beta}(\s{X}) < 1$ almost surely, implying that
    \begin{align}
        \sum_{k=0}^{L-1}\mathbb{E}\p{X_\ell X_k}\mathbb{E}\p{p_{\ell,\beta}(\s{X}) p_{k,\beta}(\s{X})} &\leq \underset{0 \leq k \leq L-1} {\max} \mathbb{E}\p{X_\ell X_k}\mathbb{E}\pp{p_{\ell,\beta}(\s{X})\sum_{k=0}^{L-1}p_{k,\beta}(\s{X})}\\
        & = \underset{0 \leq k \leq L-1} {\max} \mathbb{E}\p{X_\ell X_k}\mathbb{E}\pp{p_{\ell,\beta}(\s{X})}\\
        &< \mathbb{E}\p{p_{\ell,\beta}(\s{X})}.\label{eqn:maxIfSmallerThan}
    \end{align}
Thus, combining \eqref{eqn:gaussianIntegByPartsExplicit} and \eqref{eqn:maxIfSmallerThan} we get that $\mathbb{E}\pp{X_\ell\cdot  p_{\ell,\beta}\p{X}}>0$, as claimed.
\end{proof}

\subsection{Hard-assignment positive correlation} \label{sec:A5}
The following result will be used to prove that the correlation between the hard-assignment estimator and its corresponding template is positive (Theorem \ref{thm:hardAssignmentPositiveCorrelation}, second part, for Algorithm \ref{alg:generalizedEfNhard}).
\begin{proposition}\label{prop:hardAssignemtPositivePreli}
Let ${\s{X}} = \p{X_0, X_1,\ldots, X_{L-1}}^T$ be a zero-mean Gaussian vector, with  ${\s{X}} \sim \calN \p{0, \Sigma_{\s{X}}}$, where $\mathbb{E}\p{X_i^2} = 1$, for $0 \leq i \leq L-1$, and $\mathbb{E}\p{X_iX_j} < 1$, for $i \neq j$. Let,
    \begin{align}
        \s{\hat{R}} = \underset{0 \leq \ell \leq L-1}{\argmax} \ppp{X_\ell}.
    \end{align}
Then, for every $0 \leq \ell \leq L-1$,
    \begin{align}
       \mathbb{E} \pp{X_\ell \cdot \mathbbm{1}_{\ppp{\s{\hat{R}} = l}}} > 0 \label{eqn:hardAssignemtPositivePreli}.
    \end{align}
\end{proposition}
\begin{proof}[Proof of Proposition~\ref{prop:hardAssignemtPositivePreli}]
From proposition \ref{prop:softAssignemtPositivePreli}, we have,
    \begin{align}
       \lim_{\beta \to \infty} \mathbb{E} \pp{X_\ell \cdot \frac{ \exp{\p{\beta X_\ell}}}{\sum_{r=0}^{L-1}\exp{\p{\beta X_r}}}} \geq 0,\label{eqn:hardAssignentPositiveA37}
    \end{align}
while from Lemma \ref{lem:asymptoticBetaRatio}, we have,
    \begin{align}
       \lim_{\beta \to \infty} \frac{ \exp{\p{\beta X_\ell}}}{\sum_{r=0}^{L-1}\exp{\p{\beta X_r}}} = \mathbbm{1}_{\ppp{\s{\hat{R}} = l}}.
       \label{eqn:hardAssignentPositiveA38}
    \end{align}
Therefore, combining \eqref{eqn:hardAssignentPositiveA37} and \eqref{eqn:hardAssignentPositiveA38} leads to,
    \begin{align}
       \mathbb{E} \pp{X_\ell \cdot \mathbbm{1}_{\ppp{\s{\hat{R}} = l}}} \geq 0.
    \end{align}
Next, we show that the above inequality is in fact strict. From \eqref{eqn:gaussianIntegByPartsExplicit}, we have,
    \begin{align}
       \nonumber \mathbb{E}\pp{X_\ell p_{\ell,\beta}\p{\s{X}}} & = \beta \cdot \sum_{k=0}^{L-1}\mathbb{E}\p{X_\ell X_k}\mathbb{E}\p{-p_{\ell,\beta}\p{\s{X}} p_{k,\beta}\p{\s{X}} + \delta_{\ell k}p_{\ell,\beta}\p{\s{X}}} \\  &  = \beta \cdot \sum_{k=0}^{L-1} \pp{\mathbb{E}\p{X_\ell^2} - \mathbb{E}\p{X_\ell X_k}}\mathbb{E}\pp{p_{\ell,\beta}\p{\s{X}} p_{k,\beta}\p{\s{X}}},
    \end{align}
where the last equality follows from $\sum_{k=0}^{L-1}p_{k,\beta}\p{\s{X}} = 1$. Taking $\beta \to \infty$, applying the dominated convergence theorem, and Lemma \ref{lem:asymptoticBetaRatio} leads to, 
    \begin{align}
       \mathbb{E} \pp{X_\ell \cdot \mathbbm{1}_{\ppp{\s{\hat{R}} = l}}} &= \lim_{\beta \to \infty} \mathbb{E}\pp{X_\ell p_{\ell,\beta}\p{\s{X}}}\\
       &= \lim_{\beta \to \infty} \beta \cdot \sum_{k=0}^{L-1} \pp{\mathbb{E}\p{X_\ell^2} - \mathbb{E}\p{X_\ell X_k}}\mathbb{E}\pp{p_{\ell,\beta}\p{\s{X}} p_{k,\beta}\p{\s{X}}}. \label{eqn:hardAssignPositiceCorreA41}
    \end{align}
Lemma \ref{lemma:secondDerivativeIsPositive}, as stated and proved below, shows that $ \lim_{\beta \to \infty}  \beta \cdot \mathbb{E} \pp{p_{\ell,\beta}\p{\s{X}} p_{k,\beta}\p{\s{X}}} > 0$, for every $k \in \pp{L}$. Since we assume that $\mathbb{E}\p{X_\ell X_k} < 1$, for every $k \neq \ell$, it follows that the right-hand-side (r.h.s.) of \eqref{eqn:hardAssignPositiceCorreA41} is positive, which concludes the proof.
\end{proof}

\begin{lem} \label{lemma:secondDerivativeIsPositive}
Let ${\s{X}} = \p{X_0, X_1,\ldots, X_{L-1}}^T$ be a zero-mean Gaussian vector, with  ${\s{X}} \sim \calN \p{0, \Sigma_{\s{X}}}$, where $\mathbb{E}\p{X_i^2} = 1$, for $0 \leq i \leq L-1$, and $\mathbb{E}\p{X_iX_j} < 1$, for $i \neq j$. Then, for any $k,\ell\in[L]$,
\begin{align}
    \beta \cdot \mathbb{E}\pp{ p_{\ell,\beta} \p{\s{X}} {p_{k,\beta} \p{\s{X}}}} > 0, \label{eqn:secondDerivativeIsPositive}
\end{align}
for every finite $\beta > 0$ and $\beta \to \infty$.
\end{lem}

\begin{proof}[Proof of Lemma~\ref{lemma:secondDerivativeIsPositive}]
When $\beta > 0$ is fixed, we note that $p_{\ell,\beta}(\s{X}) > 0$, and $p_{k,\beta}(\s{X}) > 0$, almost surely, thus \eqref{eqn:secondDerivativeIsPositive} follows. Next, we deal with the case where $\beta \to \infty$. Fix $\epsilon > 0$, $\ell,k\in[L]$, and define $\calW$ as the following set of events,
\begin{align}
    \calW \triangleq \ppp{\abs{X_\ell - X_k} < \frac{\epsilon}{\beta}} \bigcap \ppp{X_\ell = \underset{k \in \pp{L}} \max \;{X_k}}.
\end{align}
Over $\calW$, using \eqref{eqn:FbetaBounds2}, we get,
\begin{align}
    p_{\ell,\beta}(\s{X}) p_{k,\beta}(\s{X}) \geq e^{-2\epsilon-2 \log{L}} > 0. \label{eqn:lowerBoundOfPlOneMinus}
\end{align}
Note that the term at the r.h.s. of \eqref{eqn:lowerBoundOfPlOneMinus} is independent of $\beta$, and so, we denote,
\begin{align}
    \s{C} \triangleq e^{-2\epsilon-2 \log{L}}> 0.
\end{align}
Thus, since $p_{\ell,\beta}(\s{X}) p_{k,\beta}(\s{X})$ is non-negative, we have,
\begin{align}
    \beta \cdot \mathbb{E}\pp{ p_{\ell,\beta}(\s{X}) p_{k,\beta}(\s{X})} \geq \beta \cdot \mathbb{E}\pp{p_{\ell,\beta}(\s{X}) p_{k,\beta}(\s{X}) \mathbbm{1}_\calW} \geq \beta \cdot \s{C}\cdot \mathbb{P}(\calW). \label{eqn:lemmaPositiveCorrelationA53}
\end{align}
Next, we show that $\lim_{\beta \to \infty} \beta \cdot \mathbb{P} \p{\calW} \geq \s{D}$, for a constant $\s{D} > 0$. 
Denote by $f_{\s{X}} (\cdot)$ the probability density function of $\s{X}$. Clearly, because the covariance matrix of $\s{X}$ is positive-definite, then $f_{\s{X}} (x) > 0$ and is continuous for all $x \in \mathbb{R}^{L}$. By definition, we have,
\begin{align}
    \lim_{\beta \to \infty} \beta \cdot \mathbb{P}\p{\mathbb{\calW}}  = \lim_{\beta \to \infty} \beta \int_{x\in\calW}  f_{\s{X}}\p{x_0^{L-1}} \mathrm{d}x_0^{L-1}.  
    \label{eqn:asymptoticBetaProbabilityDifference}
\end{align}
Subsequently, due to the continuity of $f_{\s{X}}\p{}$, when integrating over $x_k$, we get that,
\begin{align}
    \lim_{\beta \to \infty} \beta \cdot \mathbb{P}\p{\mathbb{\calW}}&= 2\epsilon\cdot \int_{\substack{x \in \mathbb{R}^{L-1} \\ \underset{r \neq \ell, k} {\max} \ppp{x_r} < x_\ell}} f_{\s{X}}\p{x_0^{k-1},x_\ell,x_{k+1}^{L-1}} \mathrm{d}x_0^{k-1}\mathrm{d}x_{k+1}^{L-1}\\
    & \triangleq \s{D}_k>0,
    \label{eqn:asymptoticBetaProbabilityDifference3}
\end{align}
where the inequality follows from the fact that $f_{\s{X}} (\cdot) > 0$ for all $\s{X} \in \mathbb{R}^{L}$, and thus the integral in \eqref{eqn:asymptoticBetaProbabilityDifference3} is positive as well. Finally, substituting \eqref{eqn:asymptoticBetaProbabilityDifference3} in \eqref{eqn:lemmaPositiveCorrelationA53} leads to,
\begin{align}
    \lim_{\beta \to \infty}  \beta \cdot \mathbb{E}\pp{ p_{\ell,\beta} \p{\s{X}} p_{k,\beta} \p{\s{X}}} \geq \s{C} \cdot \s{D}_k>0,  
\end{align}
which concludes the proof.
\end{proof}

\subsection{Properties of Gaussian random vectors}

We state and prove two results about certain properties of Gaussian random vectors.
\begin{lem} \label{lem:cycloStationary}
    Let ${\s{X}} = \p{X_0, X_1,\ldots, X_{L-1}}^T$ be a zero-mean cyclo-stationary Gaussian vector, with  ${\s{X}} \sim \calN \p{0, \Sigma_{\s{X}}}$, such that, $\mathbb{E}\p{X_i^2} = 1$ for $0 \leq i \leq L-1$, and, 
        \begin{align}
            \mathbb{E}\p{X_{\ell_1} X_{\ell_2}} = \rho_{\abs{{\ell_1}-{\ell_2}}\s{mod}L}. \label{eqn:cycloStationartyRelation}
        \end{align}
    Recall the definition of $p_{\ell,\beta}$ in \eqref{eqn:softMaxFuncDef1}. 
        \begin{enumerate}
        \item For every $\beta>0$ and $\ell \in \pp{L}$, 
        \begin{align}
            \mathbb{E}\pp{p_{\ell,\beta}\p{\s{X}}}= \frac{1}{L}. \label{eqn:expectedProbabilityUniform}
        \end{align}
        \item For every $\beta>0$ and $\ell_1, \ell_2 \in \pp{L}$,
            \begin{align}
                \mathbb{E}\pp{X_{\ell_1}\cdot p_{\ell_1,\beta}\p{\s{X}}} = \mathbb{E}\pp{X_{\ell_2}\cdot p_{\ell_2,\beta}\p{\s{X}}}. \label{eqn:equalNumeratorCyclic}
            \end{align}
    \end{enumerate}
\end{lem}

\begin{proof}[Proof of Lemma~\ref{lem:cycloStationary}]    

By definition, due to \eqref{eqn:cycloStationartyRelation}, the Gaussian vector $\s{X}$ is cyclo-stationary. Therefore, by the definition of cyclo-stationary Gaussian vectors, its cumulative distribution function $F_{\s{X}}$ is invariant under cyclic shifts \cite{durrett2019probability, balanov2024einstein}, i.e.,
    \begin{align}   
         F_{\s{X}}\left(z_0, z_1,\ldots,z_{L-1} \right) = F_{\s{X}}\left( z_\tau, z_{\tau+1},\ldots,z_{\tau+L-1} \right),
         \label{eqn:stationarityTimeShift}
    \end{align}
for any $\tau \in \mathbb{Z}$, where the indices are taken modulo $L$. Therefore, the following holds for any $\tau \in \pp{L}$,
    \begin{align}
        \mathbb{E}\pp{p_{\ell,\beta} \p{\s{X}}} & = \mathbb{E} \pp{\frac{\exp\p{\beta X_\ell}}{\sum_{r = 0}^{L-1} \exp\p{\beta X_r}}} 
        \\ & = \mathbb{E} \pp{\frac{\exp\p{\beta X_{\ell+\tau}}}{\sum_{r = 0}^{L-1} \exp\p{\beta X_{r+\tau}}}} 
        \\ & = \mathbb{E} \pp{\frac{\exp\p{\beta X_{\ell+\tau}}}{\sum_{r = 0}^{L-1} \exp\p{\beta X_{r}}}} \\
        &= \mathbb{E}\pp{p_{\ell+\tau}\p{\s{X}}},
    \end{align}
where the second equality is due to the cyclo-stationary invariance property of $\s{X}$, and the third equality is due to the fact that sum in the denominator is over all the entries of $\s{X}$. This proves \eqref{eqn:equalNumeratorCyclic}. For \eqref{eqn:expectedProbabilityUniform}, we note that  since $\mathbb{E}\pp{p_{\ell,\beta} \p{\s{X}}} = \mathbb{E}\pp{p_{\ell+\tau} \p{\s{X}}}$, for every $\tau \in \pp{L}$, as well as due to the property that $\sum_{\ell=0}^{L-1} p_{\ell,\beta} \p{\s{X}} = 1$, we get,
\begin{align}
    1 = \mathbb{E}\pp{\sum_{\ell=0}^{L-1} p_{\ell,\beta} \p{\s{X}}} = L \cdot \mathbb{E}\pp{p_{\ell,\beta} \p{\s{X}}},
\end{align}
for every $\ell\in[L]$, as claimed.
\end{proof}

The following result gives an expression for the expected value of the maximum of two Gaussian random variables.
\begin{lem}\label{lemma:twoTemplatesExpectedValueOfMax}
Let $\s{X} = \p{{X}_0, {X}_1}^T $ be a zero-mean Gaussian random vector $\s{X}  \sim \calN \p{0, \s{\Sigma}}$, such that $\mathbb{E}\pp{{X}_0^2} = \mathbb{E}\pp{{X}_1^2 } = 1$ and $\rho = {\mathbb{E}\pp{{X}_0 {X}_1}} < 1$. Then, 
    \begin{align}
        \mathbb{E} \pp{\max \{{X}_0, {X}_1\}} = \sqrt{\frac{1-\rho}{\pi}} > 0 \label{eqn:expectationValueOfMaximumOfTwoNormal}.
    \end{align}
\end{lem}

\begin{proof}[Proof of Lemma~\ref{lemma:twoTemplatesExpectedValueOfMax}]
Since,
\begin{align}
        \max\ppp{x,y} = \frac{1}{2}\p{x+y+\abs{{x-y}}},
    \end{align}
it follows that,
    \begin{align}
        \mathbb{E} \pp{\max \{{X}_0, {X}_1\}} &=  \frac{1}{2} \mathbb{E} \pp{ {X}_0 + {X}_1 + \abs{{X}_0 - {X}_1}}\\
        & = \frac{1}{2} \mathbb{E}\abs{{X}_0 - {X}_1}\label{eqn:expectedValueOfMaximum},
    \end{align}
where the second equality is because $\mathbb{E} \pp{ {X}_0} = \mathbb{E} \pp{ {X}_1} = 0$. The random variable $\abs{{X}_0 - {X}_1}$ is a folded Normal distribution with parameters $\mu_F = 0$ and $\sigma_F^2 = 2-2\rho$. Thus,
    \begin{align}
         \mathbb{E} \abs{{X}_0 - {X}_1} = \sigma_F \sqrt{\frac{2}{\pi}}\exp\p{-\frac{\mu_F^2}{2\sigma_F^2}} + \mu_F \pp{1-2\Phi\p{-\frac{\mu_F}{\sigma_F}}} \label{eqn:foldedNormalMean}.
    \end{align}
In our case, since $\mu_F = 0$ and $\sigma_F^2 = 2-2\rho$, it follows that,
    \begin{align}
         \mathbb{E} \abs{{X}_0 - {X}_1} = 2 \sqrt{\frac{1-\rho}{\pi}} \label{eqn:foldedNormalMeanReduced},
    \end{align}
which upon substitution in \eqref{eqn:expectedValueOfMaximum}, completes the proof.
\end{proof}

\subsection{Hard-assignment asymptotic number of templates}
We state and prove two results about properties the maximum of Gaussian vectors, which would aid in proving Theorem \ref{thm:hardAssignmentAsymptoticLandAsymptoticD}.

\begin{proposition}
\label{prop:expectedValueOfTwoSoftMaxAsymptoticBetaAsymptoticL2}
Let ${\s{X}} = \p{X_0, X_1,\ldots, X_{L-1}}^T$ be a zero-mean Gaussian vector, with ${\s{X}} \sim \calN \p{0, \Sigma}$. Assume that $\abs{\s{\Sigma}_{ij}} \leq \rho_{\abs{{i-j}}}$, where $\{\rho_\ell\}_{\ell\in\mathbb{N}}$ is a sequence of real-valued numbers such that $\rho_0 = 1$, $\rho_d < 1$, and $\rho_\ell \log \ell\rightarrow 0$, as $\ell\to\infty$. 
\begin{enumerate}
    \item As $L \to \infty$, we have,
        \begin{align}
            \lim_{L \to \infty} \frac{\mathbb{E} \pp{\max \p{X_0, X_1, \dots, X_{L-1}}}}{\sqrt{2 \log {L}}}  = 1. \label{eqn:A66}
        \end{align}
    \item  In addition, assume that $\s{X}$ is a cyclo-stationary Gaussian vector, i.e., $\abs{\s{\Sigma}_{ij}} = \rho_{\abs{{i-j}}\s{mod}L}$, and 
    let,
    \begin{align}
        \s{\hat{R}} = \underset{0 \leq \ell \leq L-1}{\argmax} \ppp{X_\ell}.
    \end{align}
    Then, we have,
    \begin{align}
         \lim_{L \to \infty} \frac{1}{\sqrt{2 \log L}} \cdot \frac{\mathbb{E}\pp{X_\ell \mathbbm{1}_{\ppp{\s{\hat{R}} = \ell}}}}{\mathbb{P}\pp{\s{\hat{R} = \ell}}} = 1,
    \end{align}
    for every $\ell \in \pp{L}$.
\end{enumerate}
\end{proposition}

\begin{proof}[Proof of Proposition~\ref{prop:expectedValueOfTwoSoftMaxAsymptoticBetaAsymptoticL2}]
It is known that for an i.i.d. sequence of normally distributed random variables $\{\xi_n\}$, the asymptotic distribution of the maximum $\s{M}_n \triangleq \max\{\xi_1, \xi_2, ..., \xi_n\}$ is the Gumbel distribution, i.e., for any $x\in\mathbb{R}$,
    \begin{align}
         \mathbb{P}\pp{a_n(\s{M}_n - b_n) \leq x} \to \exp(-e^{-x}),
    \end{align}
as $n\to\infty$, where,
    \begin{align}
         a_n \triangleq \sqrt{2\log n}, \label{eqn:a_n}
    \end{align}
and,
    \begin{align}
         b_n \triangleq \sqrt{2\log n} - \frac{1}{2} \frac{\log{\log{n}} + \log {4\pi}}{\sqrt{2\log n}}.
         \label{eqn:b_n}
    \end{align}
It turns out that the above convergence result remains valid even if the sequence $\ppp{{\xi_n}}_n$ is not independent and normally distributed. Specifically, as shown in \cite[Theorem 6.2.1]{leadbetter2012extremes}, a similar result holds for Gaussian random variables $\ppp{{\xi_n}}_n$ with a covariance matrix that decays such that $\lim_{n\to\infty} \rho_n \cdot \log{n} = 0$. In addition, the asymptotic expected value of the maximum $\s{M}_n$ satisfies,
\begin{align}
    \lim_{n \to \infty} \frac{\mathbb{E}\pp{\s{M}_n}}{\sqrt{2 \log n}} = 1.
\end{align}

Let us denote by $\s{M}_L$ the maximum of the vector $\s{X}$,
\begin{align}
    \s{M}_L \triangleq \max\{X_0, X_1,\ldots, X_{L-1}\}.
\end{align}
Under the assumptions of this proposition, and the discussion above, we get,
\begin{align}
    \lim_{L \to \infty} \frac{\mathbb{E}\pp{\s{M}_L}}{\sqrt{2 \log L}} = 1. \label{eqn:hardAssignmentCyclicLargeLA197},
\end{align}
which proves \eqref{eqn:A66}.

As $\sum_{r=0}^{L-1} \mathbbm{1}_{\ppp{\s{\hat{R}} = r}} = 1$, and $\s{M}_L \cdot\mathbbm{1}_{\ppp{\s{\hat{R}} = r}} = X_r \cdot \mathbbm{1}_{\ppp{\s{\hat{R}} = r}}$, we have,
\begin{align}
    \mathbb{E}\pp{\s{M}_L} = \sum_{r=0}^{L-1} \mathbb{E}\pp{\s{M}_L \cdot \mathbbm{1}_{\ppp{\s{\hat{R}} = r}}} = \sum_{r=0}^{L-1} \mathbb{E}\pp{X_r \cdot \mathbbm{1}_{\ppp{\s{\hat{R}} = r}}}. \label{eqn:hardAssignmentCyclicLargeLA198}
\end{align}
By the assumption of this proposition, and Lemma \ref{lem:cycloStationary}, 
\begin{align}
    \mathbb{E}\pp{X_r \cdot \mathbbm{1}_{\ppp{\s{\hat{R}} = r}}} = \mathbb{E}\pp{X_\ell \cdot \mathbbm{1}_{\ppp{\s{\hat{R}} = \ell}}}, \label{eqn:hardAssignmentCyclicLargeLA199}
\end{align}
for every $r,\ell \in \pp{L}$. Therefore, substituting \eqref{eqn:hardAssignmentCyclicLargeLA199} into \eqref{eqn:hardAssignmentCyclicLargeLA198}, leads to,
\begin{align}
    \mathbb{E}\pp{X_\ell \cdot \mathbbm{1}_{\ppp{\s{\hat{R}} = \ell}}} = \frac{\mathbb{E}\pp{\s{M}_L}}{L}.
\end{align}
In addition, by Lemma \ref{lem:cycloStationary}, we have $\mathbb{P}[\s{\hat{R}} = \ell] = 1/L$, for every $\ell \in \pp{L}$. Thus, substituting \eqref{eqn:hardAssignmentCyclicLargeLA199} into \eqref{eqn:hardAssignmentCyclicLargeLA197}, gives,
\begin{align}
     \lim_{L \to \infty} \frac{1}{\sqrt{2 \log L}} \cdot \frac{\mathbb{E}\pp{X_\ell \mathbbm{1}_{\ppp{\s{\hat{R}} = \ell}}}}{\mathbb{P}\pp{\s{\hat{R} = \ell}}} = 
     \lim_{L \to \infty} \frac{1}{\sqrt{2 \log L}} \cdot \frac{\mathbb{E}\pp{\s{M}_L} / L}{1/L} = 1,
\end{align}
which concludes the proof.
\end{proof}

\section{Derivation of the soft-assignment estimator} \label{sec:softAssignmentUpdateStep}

Our soft-assignment process is based on this EM algorithm. EM is one of the popular algorithms for GMMs \cite{dempster1977maximum}. The EM iteration update is given by,
    \begin{align}
       \bm{\theta}^{(t+1)} = \underset{\bm{\theta}} {\argmax} \; {\mathbb{E}_{\bm{\s{Z}} \sim p\p{\cdot \vert \bm{\s{N}}, \bm{\theta}^{(t)}}}} \pp{\log p\p{\bm{\s{N}}, \bm{\s{Z}} \vert \bm{\theta}}}\label{eqn:EMiterations}
    \end{align}
where $\bm{\s{N}}$ are the observations, $\bm{\s{Z}}$ is the missing value, $ \bm{\theta}$ are the parameters to be estimated, and $\bm{\theta}^{(t)}$ is the current estimate~\cite{dempster1977maximum}. 

The observations in our case are $\bm{\s{N}} \triangleq \ppp{n_i}_{i=0}^{M-1}$, (falsely) assumed to be generated from the GMM $\mathscr{Q}$ \eqref{eqn:calQModel}. The latent variables $\bm{\s{Z}} \triangleq \ppp{z_i}_{i=0}^{M-1}$ control the underlying component in the mixture. Namely, we have $\pr(z_i=\ell) =w_\ell = 1/L$, for $i\in[M]$. The parameters to be estimated are the GMM components means $\bm{\s{\theta}} \triangleq \ppp{\mu_\ell}_{\ell=0}^{L-1}$. As for the initialization, as described in Subsection~\ref{subsec:GMMandour}, we have, $        \bm{\s{\theta}}_\ell^{(0)} \triangleq \{x_\ell\}_{\ell=0}^{L-1}$. 

 We prove below \eqref{eqn:softAssignmentUpdateStep}--\eqref{eqn:softMaxProbabilityDist}. We are interested in a single iteration $\bm{\s{\theta}}_\ell^{(1)} = \{\hat{x}_\ell\}_{\ell=0}^{L-1}$ of the EM algorithm. Recall \eqref{eqn:EMiterations}, and we are interested in deriving a closed-form expression for a single iteration of the EM estimator $ \bm{\theta}^{(1)} = \hat{x}$. From our model definition in Subsection~\ref{subsec:GMMandour}, it is rather straightforward to see that,
\begin{align}
        \bm{\theta}^{(1)} & \triangleq \underset{\bm{\theta}} {\argmax} \; {\mathbb{E}_{\bm{\s{Z}} \sim p\p{\cdot \vert \bm{\s{N}}, \bm{\theta}^{(0)}}}} \pp{\log p\p{\bm{\s{N}}, \bm{\s{Z}}\vert \bm{\theta}}} 
       \\ & \nonumber = \underset{\bm{\theta}} {\argmin} \pp{\sum_{i=0}^{M-1} \sum_{\ell=0}^{L-1} \frac{\norm{n_i - \theta_\ell}_2^2}{2} \frac{\exp\pp{-\frac{\norm{n_i-x_\ell}_2^2}{2}}}{\sum_{r=0}^{L-1} \exp\pp{-\frac{\norm{n_i-x_r}_2^2}{2}}}} 
       \\ & =  \underset{\bm{\theta}} {\argmin} \pp{\sum_{\ell=0}^{L-1} \sum_{i=0}^{M-1}  {\norm{n_i - \theta_\ell}_2^2} p_i^{(\ell)}} \label{eqn:softAssignmentUpdateStepInNonExplicitForm},
    \end{align}
where we have used the definition of $p_{i}^{(\ell)}$ in \eqref{eqn:softMaxProbabilityDist}, and in the last step the fact that all the templates $\ppp{x_\ell}_{\ell=0}^{L-1}$ are normalized to the same value $\norm{x}$. 
As the objective is separable, then we can optimize w.r.t. each $\theta_\ell$ separately. Specifically, we denote,
    \begin{align}
       \calL \p{\bm{\theta}} \triangleq {\sum_{\ell=0}^{L-1} \sum_{i=0}^{M-1}  {\norm{n_i - \theta_\ell}_2^2} p_i^{(\ell)}}.
    \end{align}
Then, the derivative of $\calL \p{\bm{\theta}}$ w.r.t. $\theta_\ell$ is,
    \begin{align}
       \frac{1}{2}\frac{\partial \calL \p{\bm{\theta}}}{\partial {\theta_{\ell}}} = \sum_{i=0}^{M-1} \p{{\theta_{\ell}p_i^{(\ell)} - n_ip_i^{(\ell)}}}.
    \end{align}
Setting this derivative to zero, yields the following minimum of $\calL \p{\bm{\theta}}$, 
    \begin{align}
       {\theta_{\ell}} = \hat{x}_\ell=\frac{\sum_{i=0}^{M-1}{ n_i p_{i}^{(\ell)}}}{\sum_{i=0}^{M-1}{p_{i}^{(\ell)}}},
    \end{align}
which proves \eqref{eqn:softAssignmentUpdateStep}.

\section{Proof of Proposition \ref{thm:softAssignmentExtremeNormalization}}
Recall the definition of the $\beta$-soft-assignment estimator,
    \begin{align}
       \hat{x}_\ell^{(\beta)} = \frac{\sum_{i=0}^{M-1}{n_i p_{i,\beta}^{(\ell)}}}{\sum_{i=0}^{M-1}{p_{i,\beta}^{(\ell)}}} \label{eqn:softAssignmentUpdateStepProof},
    \end{align}
where $p_{i,\beta}^{(\ell)}$ is defined in \eqref{eqn:softMaxProbabilityDistWithBeta}. 
We start by proving the first part of Theorem \ref{thm:softAssignmentExtremeNormalization} for $\beta \to \infty$. Using Proposition \ref{prop:highSnrRegime}, we have,
    \begin{align}
        \lim_{\beta \to \infty} p_{i,\beta}^{(\ell)} 
        =\mathbbm{1}_{\ppp{n_i \in \calV_\ell}} \label{eqn:piDistForLargeNormalization3},
    \end{align}
almost surely. Accordingly,
    \begin{align}
        \lim_{\beta \to \infty} \frac{1}{M} \sum_{i=0}^{M-1}{p_{i,\beta}^{(\ell)}} = \frac{1}{M} \sum_{i=0}^{M-1}\mathbbm{1}_{\ppp{n_i \in \calV_\ell}} = \frac{\abs{\calA_\ell}}{M} \label{eqn:averagePiDistForLargeNormalization},
    \end{align}
and
    \begin{align}
        \lim_{\beta \to \infty} \frac{1}{M} \sum_{i=0}^{M-1}{n_i p_{i,\beta}^{(\ell)}} = \frac{1}{M} \sum_{i=0}^{M-1} n_i\mathbbm{1}_{n_i \in \calV_\ell} = \frac{1}{M}\sum_{n_i \in \calA_\ell} n_i\label{eqn:averageNiPiDistForLargeNormalization},
    \end{align}
almost surely. Therefore, substituting \eqref{eqn:averagePiDistForLargeNormalization} and \eqref{eqn:averageNiPiDistForLargeNormalization} into~\eqref{eqn:softAssignmentUpdateStepProof} proves the first part of the theorem. 

Next, we consider the case where $\beta \to 0$. To that end, we use the following lemma, proved at the end of this subsection.
\begin{lem}\label{lemma:betaToZero} 
    Recall the definition of $p_{i,\beta}^{(\ell)}$, as defined in \eqref{eqn:softMaxProbabilityDistWithBeta}. Then,
    \begin{align}
         \lim_{\beta \to 0} \lim_{M \to \infty} \frac{1}{M} \sum_{i=0}^{M-1} p_{i,\beta}^{(\ell)} = \lim_{\beta \to 0} {\mathbb{E}\pp{p_{i,\beta}^{(\ell)}}} = \frac{1}{L}, \label{eqn:lemmaBetaToZero1}
    \end{align}
and,
    \begin{align}
         \lim_{\beta \to 0} \lim_{M \to \infty} \frac{1}{M} \sum_{i=0}^{M-1} n_i p_{i,\beta}^{(\ell)} / \beta = \lim_{\beta \to 0} \frac{1}{\beta}{{\mathbb{E}\pp{n_i p_{i,\beta}^{(\ell)}}}}{} = \frac{1}{L} \pp{\p{1-\frac{1}{L}}x_\ell - \frac{1}{L}\sum_{r \neq \ell}x_r},
         \label{eqn:lemmaBetaToZero2}
    \end{align}
almost surely.
\end{lem}
Combining \eqref{eqn:lemmaBetaToZero1} and \eqref{eqn:lemmaBetaToZero2}, we have,
    \begin{align}
         \lim_{\beta \to 0} \frac{\left.\mathbb{E}\pp{n_i p_{i,\beta}^{(\ell)}}\middle/ \mathbb{E}\pp{p_{i, \beta}^{(\ell)}}\right.}{\beta \cdot \pp{\p{1-\frac{1}{L}}x_\ell - \frac{1}{L}\sum_{r \neq \ell}x_r}} = 1\label{eqn:betaGoesToZeroA104},
    \end{align}
almost surely. By \eqref{eqn:softAsymptoticObservaionsEstimator}, we have,
    \begin{align}
        \lim_{M \to \infty} \hat{x}_\ell^{(\beta)} = \frac{\mathbb{E}\pp{n_i p_{i,\beta}^{(\ell)}}}{\mathbb{E}\pp{p_{i,\beta}^{(\ell)}}} \label{eqn:strongLawOfLargeNumbersSoftEstimator},
    \end{align}
almost surely. Therefore, substituting \eqref{eqn:betaGoesToZeroA104} into \eqref{eqn:strongLawOfLargeNumbersSoftEstimator}, we get,
    \begin{align}
        \lim_{\beta \to 0} \lim_{M \to \infty} \frac{ \hat{x}_\ell^{(\beta)}}{\beta \cdot \pp{\p{1-\frac{1}{L}}x_\ell - \frac{1}{L}\sum_{r \neq \ell}x_r}} 
        = \lim_{\beta \to 0} \frac{\left.\mathbb{E}\pp{n_i p_{i,\beta}^{(\ell)}}\middle/ \mathbb{E}\pp{p_{i, \beta}^{(\ell)}}\right.}{\beta \cdot \pp{\p{1-\frac{1}{L}}x_\ell - \frac{1}{L}\sum_{r \neq \ell}x_r}} =  1,
    \end{align}
which proves the second part of the theorem. It is left to prove Lemma \ref{lemma:betaToZero}.

\begin{proof}[Proof of Lemma \ref{lemma:betaToZero}]
By the SLLN, we have,
    \begin{align}
        \lim_{M \to \infty} \frac{1}{M} \sum_{i=0}^{M-1}{n_i p_{i,\beta}^{(\ell)}} = \mathbb{E}\pp{n_i p_{i,\beta}^{(\ell)}}, \label{eqn:betaGoesToZeroA97}
    \end{align}
and
     \begin{align}
        \lim_{M \to \infty} \frac{1}{M} \sum_{i=0}^{M-1}{p_{i,\beta}^{(\ell)}} = \mathbb{E}\pp{p_{i,\beta}^{(\ell)}}, \label{eqn:betaGoesToZeroA98}
    \end{align}   
almost surely. By definition, we have,
\begin{align}
       \nonumber \lim_{\beta \to 0} p_{i,\beta}^{(\ell)}  & =  \lim_{\beta\to 0} {\frac{e^{\beta {\,\langle{n_i}, x_\ell\rangle}}}{\sum_{r=0}^{L-1} e^{\beta {\,\langle{n_i}, x_r\rangle}}}} 
       \\ \nonumber &  = \lim_{\beta \to 0} {\frac{1}{\sum_{r=0}^{L-1} e^{\beta {\,\langle{n_i}, x_r - x_\ell\rangle} }}} 
       \\ & = \lim_{\beta \to 0} {\frac{1}{1 + \sum_{r\neq \ell} e^{\beta {\,\langle{n_i}, x_r - x_\ell\rangle}}}}. \label{eqn:A37}
    \end{align}
Using Taylor series expansion around $\beta=0$, we get,
    \begin{align}
        \frac{1}{M} & \sum_{i=0}^{M-1} p_{i,\beta}^{(\ell)} = \frac{1}{M}  \sum_{i=0}^{M-1} \frac{1}{1 + \sum_{r\neq \ell} e^{\beta {\,\langle{n_i}, x_r - x_\ell\rangle}}} \\ & = \frac{1}{L} + \frac{\beta}{L^2} \sum_{r\neq\ell} \pp{\frac{1}{M} \sum_{i=0}^{M-1} \langle{n_i}, x_\ell - x_r \rangle} +  \frac{1}{M}\sum_{i=0}^{M-1} \sum_{k=2}^{\infty}\beta^k \cdot f_k\p{\ppp{\langle{n_i}, x_r - x_\ell\rangle}_{r \neq \ell}}, \label{eqn:betaToZeroA121}
    \end{align}
where $f_k$ is the $k$th order Taylor expansion of $p_{i,\beta}^{(\ell)}$. Similarly, we have,
    \begin{align}
        \frac{1}{M}\sum_{i=0}^{M-1} n_i p_{i,\beta}^{(\ell)} &= \frac{1}{M}  \sum_{i=0}^{M-1} \frac{n_i}{1 + \sum_{r\neq \ell} e^{\beta {\,\langle{n_i}, x_r - x_\ell\rangle}}} \\ & = \frac{1}{L \cdot M} \sum_{i=0}^{M-1} n_i + \frac{\beta}{L^2} \sum_{r\neq\ell} \pp{\frac{1}{M} \sum_{i=0}^{M-1} n_i \langle{n_i}, x_\ell - x_r \rangle}  + \\ &\quad\quad\quad\quad +  \frac{1}{M}\sum_{i=0}^{M-1} \sum_{k=2}^{\infty}\beta^k n_i \cdot f_k\p{\ppp{\langle{n_i}, x_r - x_\ell\rangle}_{r \neq \ell}}.\label{eqn:betaToZeroA122}
    \end{align}
Now, taking $M \to \infty$ in \eqref{eqn:betaToZeroA121} and applying the SLLN once again, we obtain,
    \begin{align}
        \mathbb{E} \pp{p_{i,\beta}^{(\ell)}} = \frac{1}{L} + \frac{\beta}{L^2} \sum_{r\neq\ell} {\mathbb{E} \pp{ \langle{n_i}, x_\ell - x_r \rangle}} +  \mathbb{E} \pp{\sum_{k=2}^{\infty}\beta^k \cdot f_k\p{\ppp{\langle{n_i}, x_r - x_\ell\rangle}_{r \neq \ell}}}. \label{eqn:betaToZeroA123}
    \end{align}
Similarly, taking $M \to \infty$ in \eqref{eqn:betaToZeroA122}, and applying the SLLN, we get,
\begin{align}
        \mathbb{E} \pp{n_i p_{i,\beta}^{(\ell)}} = \frac{1}{L} \mathbb{E}\pp{n_i} + \frac{\beta}{L^2} \sum_{r\neq\ell} {\mathbb{E} \pp{n_i \langle{n_i}, x_\ell - x_r \rangle}} +  \mathbb{E} \pp{\sum_{k=2}^{\infty}\beta^k \cdot n_i f_k\p{\ppp{\langle{n_i}, x_r - x_\ell\rangle}_{r \neq \ell}}}. \label{eqn:betaToZeroA124}
    \end{align}
Since $\mathbb{E}\pp{p_{i, \beta}^{(\ell)}} < \infty$, and $\mathbb{E}\pp{\abs{n_i p_{i, \beta}^{(\ell)}}} \leq \mathbb{E}\pp{\abs{n_i}} < \infty$, then the terms at the r.h.s. of \eqref{eqn:betaToZeroA123} and \eqref{eqn:betaToZeroA124} are finite, i.e.,
\begin{align}
     \mathbb{E}\pp{\sum_{k=2}^{\infty}\beta^{k-2} \cdot f_k\p{\ppp{\langle{n_i}, x_r - x_\ell\rangle}_{r \neq \ell}}} &= \s{C}_0 < \infty, \label{eqn:betaToZeroA125}\\
     \mathbb{E}\pp{\sum_{k=2}^{\infty}\beta^{k-2} n_i \cdot f_k\p{\ppp{\langle{n_i}, x_r - x_\ell\rangle}_{r \neq \ell}}} &= \s{C}_1 < \infty.
     \label{eqn:betaToZeroA126}
\end{align}
Therefore, combining \eqref{eqn:betaToZeroA123}, and \eqref{eqn:betaToZeroA125}, we have,
    \begin{align}
         \lim_{\beta \to 0} \frac{\mathbb{E}\pp{p_{i,\beta}^{(\ell)}}}{\frac{1}{L} + \frac{\beta}{L^2} \sum_{r\neq\ell} {\mathbb{E} \pp{ \langle{n_i}, x_\ell - x_r \rangle}} + C_0 \cdot \beta^2} = \lim_{\beta \to 0} \frac{\mathbb{E}\pp{p_{i,\beta}^{(\ell)}}}{\frac{1}{L} + \frac{\beta}{L^2} \sum_{r\neq\ell} {\mathbb{E} \pp{ \langle{n_i}, x_\ell - x_r \rangle}}} = 1. \label{eqn:betaToZeroA127}
    \end{align}
Similarly, combining \eqref{eqn:betaToZeroA124}, and \eqref{eqn:betaToZeroA126}, we get,
    \begin{align}
         \lim_{\beta \to 0} & \frac{\mathbb{E}\pp{n_i p_{i,\beta}^{(\ell)}}}{\frac{1}{L}\mathbb{E}\pp{n_i} + \frac{\beta}{L^2} \sum_{r\neq\ell} {\mathbb{E} \pp{n_i \langle{n_i}, x_\ell - x_r \rangle}} + C_1 \cdot \beta^2} \\ & \quad\quad= \lim_{\beta \to 0} \frac{\mathbb{E}\pp{n_i p_{i,\beta}^{(\ell)}}}{\frac{1}{L} \mathbb{E} \pp{n_i} + \frac{\beta}{L^2} \sum_{r\neq\ell} {\mathbb{E} \pp{ n_i \langle{n_i}, x_\ell - x_r \rangle}}} = 1. \label{eqn:betaToZeroA128}
    \end{align}
Since $\mathbb{E}\pp{{\,\langle{n_i}, x_\ell - x_r\rangle}} = 0$ and $\mathbb{E}\pp{n_i{\,\langle{n_1},  x_\ell - x_r\rangle}} = x_\ell - x_r$, then the denominator of \eqref{eqn:betaToZeroA128} is,
    \begin{align}
          \mathbb{E}\pp{{\frac{n_i}{L}+\frac{\beta n_i}{L^2} \sum_{r\neq \ell}{\,\langle{n_i}, x_\ell-x_r\rangle} }} & = \frac{\beta}{L^2} \sum_{r\neq \ell} x_\ell - x_r \\ & = \frac{\beta}{L} \cdot \pp{\p{1-\frac{1}{L}} x_\ell - \frac{1}{L}\sum_{r \neq \ell} x_r}, \label{eqn:betaToZeroA129}
    \end{align}
and the denominator of \eqref{eqn:betaToZeroA127} is,
    \begin{align}
          \mathbb{E}\pp{{\frac{1}{L}+\frac{\beta}{L^2} \sum_{r\neq \ell}{\,\langle{n_i}, x_\ell-x_r\rangle} }} = \frac{1}{L}.
          \label{eqn:betaToZeroA130}
    \end{align}
Finally, substituting \eqref{eqn:betaToZeroA130} in \eqref{eqn:betaToZeroA127}, we obtain,
    \begin{align}
         \lim_{\beta \to 0} \lim_{M \to \infty} \frac{1}{M} \sum_{i=0}^{M-1} p_{i,\beta}^{(\ell)} = \lim_{\beta \to 0} {\mathbb{E}\pp{p_{i,\beta}^{(\ell)}}} = \frac{1}{L},
    \end{align}
and substituting \eqref{eqn:betaToZeroA129} in \eqref{eqn:betaToZeroA128}, gives,
    \begin{align}
         \lim_{\beta \to 0} \frac{1}{\beta}{{\mathbb{E}\pp{n_i p_{i,\beta}^{(\ell)}}}}{} = \frac{1}{L} \pp{\p{1-\frac{1}{L}}x_\ell - \frac{1}{L}\sum_{r \neq \ell}x_r},
    \end{align}
which completes the proof.

\end{proof}

\section{Proof of Theorem \ref{thm:hardAssignmentPositiveCorrelation}} \label{sec:appD}

\subsection{Hard-assignment}
We start by proving \eqref{eqn:closetstTemplateIsCorresponding}. From Lemma \ref{lem:convergenceOfHardAssign} and \eqref{eqn:hardAsymptoticObservaionsCorrelationMain}, we have that,
    \begin{align}
        {\langle{\hat{x}_\ell}, x_\ell\rangle} \xrightarrow[]{\s{a.s.}} \frac{\mathbb{E}\pp{\langle{n_1, x_\ell\rangle} \mathbbm{1}_{\ppp{n_1 \in \calV_\ell}}}}{\mathbb{P}\pp{n_1 \in \calV_\ell}},\label{eqn:xlxlCorrelation}
    \end{align}
and similarly,
    \begin{align}
        {\langle{\hat{x}_\ell}, x_k\rangle} \xrightarrow[]{\s{a.s.}} \frac{\mathbb{E}\pp{\langle{n_1, x_k\rangle} \mathbbm{1}_{\ppp{n_1 \in \calV_\ell}}}}{\mathbb{P}\pp{n_1 \in \calV_\ell}}. \label{eqn:xlxkCorrelation}
    \end{align}
By the definition of the set $\calV_\ell$, we have $n_1 \in \calV_\ell$ if and only if $\langle{n_1, x_\ell\rangle} \geq \langle{n_1, x_k \rangle}$, for every $k \neq \ell$. Therefore,
    \begin{align}
        \mathbb{E}\pp{\langle{n_1, x_\ell\rangle} \mathbbm{1}_{\ppp{n_1 \in \calV_\ell}}} - \mathbb{E}\pp{\langle{n_1, x_k\rangle} \mathbbm{1}_{\ppp{n_1 \in \calV_\ell}}} > 0, \label{eqn:closetTemplateIsCorrespondingTemplate2}
    \end{align}
where the strict inequality follows from the fact that the covariance matrix of the underlying Gaussian process is positive definite, and the maximum of such a Gaussian process is almost sure unique. Thus, \eqref{eqn:closetTemplateIsCorrespondingTemplate2} combined with \eqref{eqn:xlxlCorrelation} and \eqref{eqn:xlxkCorrelation} leads to \eqref{eqn:closetstTemplateIsCorresponding}. Next, \eqref{eqn:nonVanishingEstimator} follows immediately since $\hat{x}_\ell$ satisfies \eqref{eqn:closetstTemplateIsCorresponding}, and thus it cannot vanish. 

Next, we prove \eqref{eqn:positiveCorrelation}. To that end, we apply Proposition \ref{prop:hardAssignemtPositivePreli}, where $\s{S}^{(x)}$ , as defined in \eqref{eqn:SxDef}, plays the role of $\s{X}$ in Proposition \ref{prop:hardAssignemtPositivePreli}. The entries of the covariance matrix of $\s{S}$ are given by $\sigma_{ij} = \langle{x_i, x_j \rangle} / \norm{x_\ell}_2^2$, and by the assumptions in Theorem \ref{thm:hardAssignmentPositiveCorrelation}, they satisfy the conditions of Proposition~\ref{prop:hardAssignemtPositivePreli}. Finally, note that the event $\{\s{\hat{R}} = \ell\}$ in Proposition \ref{prop:hardAssignemtPositivePreli} is equivalent to the event $\ppp{n_i \in \calV_\ell}$. Therefore, it follows from \eqref{eqn:hardAssignemtPositivePreli} that,
    \begin{align}
       \mathbb{E} \pp{\langle{n_i, x_\ell\rangle} \mathbbm{1}_{\ppp{n_i \in \calV_\ell}}} > 0. \label{eqn:positiveCorrelationA109}
    \end{align}
Substituting \eqref{eqn:positiveCorrelationA109} in \eqref{eqn:xlxlCorrelation} concludes the proof.

Finally, we prove \eqref{eqn:positiveCorrelationSum}.
Fix $L\geq 2$. Consider normalized templates $\ppp{{x}_\ell}_{\ell=0}^{L-1}$.  Denote by $\ppp{\hat{x}_\ell}_{\ell=0}^{L-1}$ the output of Algorithm~\ref{alg:generalizedEfNhard}. 
According to Lemma \ref{lem:convergenceOfHardAssign} and \eqref{eqn:hardAsymptoticObservaionsCorrelationMain}, it follows that,
\begin{align}
    {\langle{\hat{x}_\ell}, x_\ell\rangle} \xrightarrow[]{\s{a.s.}} \frac{\mathbb{E}\pp{\langle{n_1, x_\ell\rangle} \mathbbm{1}_{\ppp{n_1 \in \calV_\ell}}}}{\mathbb{P}\pp{n_1 \in \calV_\ell}} = \frac{\mathbb{E}\pp{{\underset{0 \leq r \leq L-1} \max {\langle{n_1, x_r\rangle}}} \mathbbm{1}_{\ppp{n_1 \in \calV_\ell}}}}{\pr [\s{\hat{R}}^{(x)} = \ell]}, \label{eqn:E1}
\end{align}
as $M \to \infty$. Thus, substituting \eqref{eqn:E1} into the left-hand-side of \eqref{eqn:positiveCorrelationSum} results,
\begin{align}
     \sum_{\ell=0}^{L-1} {\langle{\hat{x}_\ell}, x_\ell\rangle} \mathbb{P}[\s{\hat{R}}^{(x)} = \ell] & \xrightarrow[]{\s{a.s.}}
     \sum_{\ell=0}^{L-1} \frac{\mathbb{E}\pp{{\underset{0 \leq r \leq L-1} \max {\langle{n_1, x_r\rangle}}} \mathbbm{1}_{\ppp{n_1 \in \calV_\ell}}}}{\pr [\s{\hat{R}}^{(x)} = \ell]} \pr [\s{\hat{R}}^{(x)} = \ell] 
     \\ & = \mathbb{E}\pp{{\underset{0 \leq r \leq L-1} \max {\langle{n_1, x_r\rangle}}}}.\label{eqn:E3}
\end{align}

Denote by $\rho = \underset{i \neq j} \min \pp{{\langle{{x}_{i}}, x_{j}\rangle}}$, and by $x_{i_1}, x_{i_2}$, the templates that satisfy $\rho = \langle{{x}_{i_1}}, x_{i_2}\rangle$. Then, by the convexity property of the maximum function, we have,
\begin{align}
    \mathbb{E}\pp{{\underset{0 \leq r \leq L-1} \max {\langle{n_1, x_r\rangle}}}} \geq \mathbb{E}{\pp{\underset{r \in \ppp{i_1,i_2}} \max {\langle{n_1, x_r\rangle}}}}. \label{eqn:E4}
\end{align}
By Lemma \ref{lemma:twoTemplatesExpectedValueOfMax}, it follows that,
\begin{align}
    \mathbb{E}{\pp{\underset{r \in \ppp{i_1,i_2}} \max {\langle{n_1, x_r\rangle}}}} = \sqrt{\frac{1-\rho}{\pi}}. \label{eqn:E5}
\end{align}
Thus, combining \eqref{eqn:E3}, \eqref{eqn:E4}, and \eqref{eqn:E5}, results,
\begin{align}
     \sum_{\ell=0}^{L-1} {\langle{\hat{x}_\ell}, x_\ell\rangle} \mathbb{P}[\s{\hat{R}}^{(x)} = \ell] & \xrightarrow[]{\s{a.s.}}
     \mathbb{E}\pp{{\underset{0 \leq r \leq L-1} \max {\langle{n_1, x_r\rangle}}}} 
     \\  & \geq  \mathbb{E}{\pp{\underset{r \in \ppp{i_1,i_2}} \max {\langle{n_1, x_r\rangle}}}}
     = \sqrt{\frac{1-\rho}{\pi}} > 0, \label{eqn:mainResultSectionE1}
\end{align}
which proves \eqref{eqn:positiveCorrelationSum}.

\subsection{Soft-assignment}

We start by proving \eqref{eqn:positiveCorrelation}. To that end, we apply Proposition \ref{prop:softAssignemtPositivePreli}, where $\s{S}^{(x)}$ , as defined in \eqref{eqn:SxDef}, plays the role of $\s{X}$ in Proposition \ref{prop:softAssignemtPositivePreli}. The entries of the covariance matrix of $\s{S}$ are given by $\sigma_{ij} = \langle{x_i, x_j \rangle} / \norm{x_\ell}_2^2$, and by the assumptions in Theorem \ref{thm:hardAssignmentPositiveCorrelation}, they satisfy the conditions of Proposition~\ref{prop:softAssignemtPositivePreli}. Therefore, it follows from \eqref{eqn:prop3} that,
    \begin{align}
       0 < \mathbb{E} \pp{S_\ell^{(x)} \cdot \frac{ \exp{\p{\beta S_\ell^{(x)}}}}{\sum_{r=0}^{L-1}\exp{\p{\beta S_r^{(x)}}}}} = \mathbb{E} \pp{\frac{{\langle{n_i, x_\ell\rangle}}}{\norm{x_\ell}_2} \cdot \frac{ \exp{\p{\beta \frac{{\langle{n_i, x_\ell\rangle}}}{\norm{x_\ell}_2}}}}{\sum_{r=0}^{L-1}\exp{\p{\beta \frac{{\langle{n_i, x_r\rangle}}}{\norm{x_r}_2}}}}} \label{eqn:prop3ForTheorem}.
    \end{align}
Since \eqref{eqn:prop3ForTheorem} is true for every $\beta > 0$, then choosing $\beta = \norm{x_\ell}_2$, and recalling that $\norm{x_\ell}_2 = \norm{x_r}_2$ for every $0 \leq r,l \leq L-1$, we get,
    \begin{align}
       0 < \mathbb{E} \pp{{{\langle{n_i, x_\ell\rangle}}} \cdot \frac{ \exp{\p{{{\langle{n_i, x_\ell\rangle}}}}}}{\sum_{r=0}^{L-1}\exp{\p{{{\langle{n_i, x_r\rangle}}}}}}} =  {\langle{\hat{x}_\ell, x_\ell\rangle}} \mathbb{E}\pp{p_i^{(\ell)}}.
    \end{align}
Since $\mathbb{E}[p_i^{(\ell)}] > 0$, it follows that ${\langle{\hat{x}_\ell, x_\ell\rangle}} > 0$, as claimed. Finally, we prove \eqref{eqn:nonVanishingEstimator}. Since the estimator $\hat{x}_\ell$ satisfies \eqref{eqn:positiveCorrelation}, it is clear it cannot vanish.

\section{Proof of Proposition \ref{thm:positiveCorrForCyclic}}

Fix $L\geq 2$. Let $\s{\hat{R}}$ be defined as in \eqref{eqn:OptShiftRealSpace}, with the normalized templates $\ppp{x_\ell}_{\ell=0}^{L-1}$, and by $\ppp{\hat{x}_\ell}_{\ell=0}^{L-1}$ the output of Algorithm~\ref{alg:generalizedEfNhard}. Define the Gaussian vector $\s{S}^{(x)}$ as,
\begin{align}
    {S}_\ell^{(x)} \triangleq {\langle{n_1, x_{\ell}\rangle}} / \norm{x_\ell}_2 \label{eqn:E8}.
\end{align}
Then, for templates $\ppp{{x}_\ell}_{\ell=0}^{L-1}$ that satisfy Assumption \ref{assump:0}, the corresponding vector ${S}_\ell^{(x)}$ satisfy the conditions of Lemma \ref{lem:cycloStationary}. Then, applying the lemma on $S_\ell^{(x)}$, we have,
\begin{align}
    \mathbb{P}[\s{\hat{R}} = \ell] = 1/L, \label{eqn:E9}
\end{align}
for every $\ell \in \pp{L}$, and,
\begin{align}
    \langle \hat{x}_{\ell_1}, x_{\ell_1} \rangle = \langle \hat{x}_{\ell_2}, x_{\ell_2} \rangle, \label{eqn:E10}
\end{align}
for every $\ell_2, \ell_2 \in \pp{L}$. Therefore, combining \eqref{eqn:E9} and \eqref{eqn:E10} and substituting into the left-hand-side of \eqref{eqn:positiveCorrelationSum}, for $M \to \infty$, results,
    \begin{align}
        \mathbb{E} \pp{{\langle{\hat{x}_\ell}, x_\ell\rangle}} = \sum_{\ell=0}^{L-1} {\langle{\hat{x}_\ell}, x_\ell\rangle} \mathbb{P}[\s{\hat{R}} = \ell] \geq \sqrt{\frac{1-\rho}{\pi}} > 0,
    \end{align}
which completes the proof of the proposition.

\section{Proof of Proposition \ref{thm:meanInverseDependency}}
Using Lemma \ref{lem:convergenceOfHardAssign} and \eqref{eqn:hardAsymptoticObservaionsCorrelationMain}, we have,
      \begin{align}
        \sum_{\ell=0}^{L-1} {\langle{\hat{x}_\ell}, x_{\ell}\rangle}\pr [\s{\hat{R}}^{(y)} = \ell] = \sum_{\ell=0}^{L-1} {\langle{\hat{x}_\ell}, x_{\ell}\rangle} \mathbb{E}\pp{\mathbbm{1}_{\ppp{n_i \in \calV_\ell}}} \xrightarrow[]{\s{a.s.}} \sum_{\ell=0}^{L-1} \mathbb{E}\pp{{\langle{n_i}, x_\ell\rangle} 
        \mathbbm{1}_{\ppp{n_i \in \calV_\ell}}},
    \end{align}
as $M\to\infty$. Thus, \eqref{eqn:hardAssignMeanInverse} is equivalent to proving the following,
    \begin{align}
        \sum_{\ell=0}^{L-1} \mathbb{E}\pp{{\langle{n_i}, x_\ell\rangle} \mathbbm{1}_{\ppp{n_i \in \calV_\ell^{(x)}}}} 
        \geq \sum_{\ell=0}^{L-1} \mathbb{E}\pp{{\langle{n_i}, y_\ell\rangle} \mathbbm{1}_{\ppp{n_i \in \calV_\ell^{(y)}}}} , \label{eqn:hradTargetFunctionMeanInverse}
    \end{align}
as $M\to\infty$, where $\{\calV_\ell^{(x)}\}_\ell$ and $\{\calV_\ell^{(y)}\}_\ell$ are defined as $\{\calV_\ell\}_\ell$ in \eqref{eqn:VlDef}, for two sets of templates $\{x_\ell\}_\ell$ and $\{y_\ell\}_\ell$, respectively. By Lemma \ref{lem:convergenceOfHardAssign} and \eqref{eqn:hardAsymptoticObservaionsCorrelationMain}, the following holds,
    \begin{align}
        \sum_{\ell=0}^{L-1} \mathbb{E}\pp{{\langle{n_i}, x_\ell\rangle} \mathbbm{1}_{\ppp{n_i \in \calV_\ell^{(x)}}}} & = \sum_{\ell = 0}^{L-1} \mathbb{E} \pp{\underset{0 \leq r \leq L-1} \max {\langle{n_i, x_r\rangle}}\mathbbm{1}_{\ppp{n_i \in \calV_\ell^{(x)}}} } 
        \\ & =  \mathbb{E} \pp{\underset{0 \leq r \leq L-1} \max {\langle{n_i, x_r\rangle}} \sum_{\ell = 0}^{L-1} \mathbbm{1}_{\ppp{n_i \in \calV_\ell^{(x)}}} } 
        \\ & = \mathbb{E} \pp{\underset{0 \leq r \leq L-1} \max {\langle{n_i, x_r\rangle}} }. 
        \label{eqn:sumOfTermIsMaximum}
    \end{align}
Combining \eqref{eqn:hradTargetFunctionMeanInverse} and \eqref{eqn:sumOfTermIsMaximum}, it follows that we need to prove that,
\begin{align}
    \mathbb{E} \pp{\underset{0 \leq \ell \leq L-1} \max {\langle{n_i, x_\ell\rangle}} } \geq \mathbb{E} \pp{\underset{0 \leq \ell \leq L-1} \max {\langle{n_i, y_\ell\rangle}} }. \label{eqn:equivalentConditionByMaximums}
\end{align}
To that end, we use the Sudakov-Fernique inequality \cite{vershynin2018high}.

\begin{lem}\label{lem:SudakovFerniqueInequality}
(Sudakov-Fernique inequality) Let ${\s{X}} = \p{{X}_0, {X}_1, ..., {X}_{L-1}}^T$ and ${\s{Y}} = \p{{Y}_0, {Y}_1, ..., {Y}_{L-1}}^T$ be two zero-mean Gaussian vectors, with  ${\s{X}} \sim \calN \p{0, \Sigma_{\s{X}}}$ and ${\s{Y}} \sim \calN \p{0, \Sigma_{\s{Y}}}$, satisfying, $\mathbb{E}\pp{\p{{X}_i-X_j}^2} \geq \mathbb{E}\pp{\p{{Y}_i-Y_j}^2}$ for all $i,j \in \pp{L}$. Then,
    \begin{align}
        \mathbb{E} \pp{\max \{{X}_0, {X}_1, ..., {X}_{L-1}\}} \geq \mathbb{E} \pp{\max \{{Y}_0, {Y}_1, ..., {Y}_{L-1}\}}.
    \end{align}
\end{lem}

Define the Gaussian vectors $\s{S}^{(x)}$ and $\s{S}^{(y)}$ as,
\begin{align}
    {S}_\ell^{(x)} \triangleq {\langle{n_1, x_{\ell}\rangle}} / \norm{x_\ell}_2, {S}_\ell^{(y)} = {\langle{n_1, y_{\ell}\rangle}}/\norm{y_\ell}_2 \label{eqn:SxSyDef},
\end{align}
for $\ell \in \pp{L}$. Note that $\s{S}^{(x)}$ and $\s{S}^{(y)}$ are zero-mean Gaussian vectors with the covariance matrices $\sigma_{\ell_1, \ell_2}^{(x)} = {\langle{x_{\ell_1}, x_{\ell_2}\rangle}} / \norm{x}_2^2$ and $\sigma_{\ell_1, \ell_2}^{(y)} = {\langle{y_{\ell_1}, y_{\ell_2}\rangle}} / \norm{y}_2^2$, for $\ell_1,\ell_2\in[L]$. By assumption, we have $\sigma_{\ell_1 \ell_2}^{(x)} \leq \sigma_{\ell_1 \ell_2}^{(y)}$, for every $\ell_1, \ell_2 \in \pp{L}$. Therefore, the Gaussian vectors $\s{S}^{(x)}$ and $\s{S}^{(y)}$ satisfy the conditions of the Sudakov-Fernique inequality, and it follows that,
    \begin{align}
         \mathbb{E} \pp{\underset{0 \leq \ell \leq L-1} \max {\langle{n_i, x_\ell\rangle}} } \geq \mathbb{E} \pp{\underset{0 \leq \ell \leq L-1} \max {\langle{n_i, y_\ell\rangle}} } \label{eqn:applyingSudakovFerniqueToThm2},
    \end{align}
which proves \eqref{eqn:equivalentConditionByMaximums}.

\section{Necessary conditions for inverse dependence}
In this subsection, we give some necessary conditions for the inverse dependence property to hold. Recall the definition of the function $p_\ell \p{\s{Q}}$ in \eqref{eqn:softMaxFuncDef1}, where $\s{Q} = (Q_0,Q_1,\ldots,Q_{L-1})$. In addition, recall the definitions of the  Gaussian vector $\s{S}^{(x)}$ and $\s{S}^{(y)}$ in \eqref{eqn:SxSyDef}. Define,
    \begin{align}
        w_{\ell,j}^{(x)} \triangleq \frac{\mathbb{E}\pp{p_\ell \p{\s{S}^{(x)}} p_j \p{\s{S}^{(x)}}}}{\mathbb{E}\pp{p_\ell \p{\s{S}^{(x)}}}}\label{eqn:weightWjlSx},
    \end{align}
and,
    \begin{align}
        w_{\ell,j}^{(y)} \triangleq \frac{\mathbb{E}\pp{p_\ell \p{\s{S}^{(y)}} p_j \p{\s{S}^{(y)}}}}{\mathbb{E}\pp{p_\ell \p{\s{S}^{(y)}}}}\label{eqn:weightWjlSy}.
    \end{align}
Note that $\sum_{j=0}^{L-1}w_{\ell,j}^{(x)} = 1$ and  $\sum_{j=0}^{L-1}w_{\ell,j}^{(y)} = 1$, for every $0 \leq \ell \leq L-1$. Furthermore, it holds $w_{\ell,j}^{(x)} \geq 0$ and $w_{\ell,j}^{(y)} \geq 0$, for every $0 \leq \ell,j \leq L-1$. We have the following result.
\begin{proposition}[Soft-assignment inverse dependency]\label{prop:softAssignmentGreaterCorrelation}
Fix $L\geq 2$. Let $\ppp{{x}_\ell}_{\ell=0}^{L-1}$ and $\ppp{{y}_\ell}_{\ell=0}^{L-1}$ be two sets of templates. Assume that,
    \begin{align}
        \sum_{j=0}^{L-1}{\langle{{x}_{\ell}}, x_{j}\rangle}w_{\ell,j}^{(x)} \leq \sum_{j=0}^{L-1}{\langle{{y}_{\ell}}, y_{j}\rangle}w_{\ell,j}^{(y)} \label{eqn:assumSoftAssignmentGreaterCorrelation},
    \end{align}
for every $0\leq\ell\leq L-1$. Denote by $\ppp{{\hat{x}}_\ell}_{\ell=0}^{L-1}$ and $\ppp{{\hat{y}}_\ell}_{\ell=0}^{L-1}$ the output of the soft-assignment estimators in Algorithm \ref{alg:generalizedEfNsoft}. Then, for every $0 \leq \ell \leq L-1$, we have,
    \begin{align}
        \mathbb{E} \pp{{\langle{\hat{x}_\ell}, x_\ell\rangle}} \geq  \mathbb{E} \pp{{\langle{\hat{y}_\ell}, y_\ell\rangle}} \label{eqn:softAssignmentGreaterCorrelation}.
    \end{align}
\end{proposition}

We see that the condition in \eqref{eqn:assumSoftAssignmentGreaterCorrelation} is a weighted monotonicity assumption, where the weights $\{w_{\ell,j}^{(x)},w_{\ell,j}^{(y)}\}$ are proportional to certain relative probabilities in \eqref{eqn:softMaxFuncDef1}, for the choice of a given template. An implication of Proposition \ref{prop:softAssignmentGreaterCorrelation} is that when an additional template is added, whose correlation with the other templates is less than unity, then the correlation between the estimator and the corresponding template increases. 


\begin{proof}[Proof of Proposition \ref{prop:softAssignmentGreaterCorrelation}]
Define $X_\ell \triangleq {\langle{n_i, x_\ell\rangle}}$ and  $Y_\ell \triangleq {\langle{n_i, y_\ell\rangle}}$. Then, similarly to \eqref{eqn:gaussianIntegByPartsExplicit}, by using Lemma \ref{lemma:3}, we get,
    \begin{align}
        \frac{\mathbb{E}\pp{X_\ell p_\ell\p{\s{X}}}}{\mathbb{E} \pp{p_\ell\p{\s{X}}}} = \norm{x_\ell}_2^2-\sum_{k=0}^{L-1}{\langle{x_\ell, x_k\rangle}} w_{\ell,k}^{(x)},
    \end{align}
and
    \begin{align}
        \frac{\mathbb{E}\pp{Y_\ell p_\ell\p{\s{Y}}}}{\mathbb{E} \pp{p_\ell \p{{\s{Y}}}}} = \norm{y_\ell}_2^2-\sum_{k=0}^{L-1}{\langle{y_\ell, y_k\rangle}} w_{\ell,k}^{(y)}.
    \end{align}
Therefore, since $\norm{x_\ell}_2 = \norm{y_\ell}_2$, the assumption in \eqref{eqn:assumSoftAssignmentGreaterCorrelation}, implies that,
    \begin{align}
        \nonumber \mathbb{E} \pp{{\langle{\hat{x}_\ell}, x_\ell \rangle}} =  \frac{\mathbb{E}\pp{X_\ell p_\ell \p{\s{X}}}}{\mathbb{E} \pp{p_\ell \p{{\s{X}}}}} & = \norm{x_\ell}_2^2-\sum_{k=0}^{L-1}{\langle{x_\ell, x_k\rangle}} w_{\ell,k}^{(x)} \\ \nonumber & \geq \norm{y_\ell}_2^2-\sum_{k=0}^{L-1}{\langle{y_\ell, y_k\rangle}} w_{\ell,k}^{(y)} \\ & = \mathbb{E} \pp{{\langle{\hat{y}_\ell}, y_\ell\rangle}},
    \end{align}
which concludes the proof.    
\end{proof}

\section{Proof of Proposition \ref{thm:largerCorrelationForCycloCorrelations}}
Let $\{x_\ell\}_{\ell=0}^{L-1}$ and $\{y_\ell\}_{\ell=0}^{L-1}$ be sets of normalized template hypotheses, such that ${\langle{{x}_{\ell_1}}, x_{\ell_2}\rangle} \leq {\langle{{y}_{\ell_1}}, y_{\ell_2}\rangle}$, for every $0 \leq \ell_1,\ell_2 \leq L-1$. 
Denote by $\{\hat{x}_\ell\}_{\ell=0}^{L-1}$ and $\{\hat{y}_\ell\}_{\ell=0}^{L-1}$ the outputs of Algorithm~\ref{alg:generalizedEfNhard} when applied using the templates $\{x_\ell\}_{\ell=0}^{L-1}$ and $\{y_\ell\}_{\ell=0}^{L-1}$, respectively. Then, by Proposition \ref{thm:meanInverseDependency}, for $M \to \infty$, we have,
\begin{align}
    \sum_{\ell=0}^{L-1} {\langle{\hat{x}_\ell}, x_{\ell}\rangle} \mathbb{P}[\s{\hat{R}}^{(x)} = \ell] \geq  \sum_{\ell=0}^{L-1} {\langle{\hat{y}_\ell}, y_{\ell}\rangle} \mathbb{P}[\s{\hat{R}}^{(y)} = \ell],
    \label{eqn:H0}
\end{align}
almost surely.

Recall the definitions of $\s{S}^{(x)}$ and $\s{S}^{(y)}$ in \eqref{eqn:SxSyDef}.  By Assumption \ref{assump:0}, $\s{S}^{(x)}$ and $\s{S}^{(y)}$ are cyclo-stationary Gaussian vectors.  Then, we apply Lemma \ref{lem:cycloStationary}, whose conditions are satisfied for $\s{S}^{(x)}$ and $\s{S}^{(y)}$, and we have,
\begin{align}
    \mathbb{P}[\s{\hat{R}}^{(x)} = \ell] = \frac{1}{L}, \quad \mathbb{P}[\s{\hat{R}}^{(y)} = \ell] = \frac{1}{L}, \label{eqn:H5}
\end{align}
for every $\ell \in \pp{L}$, and, 
\begin{align}
    \langle \hat{x}_{\ell_1}, x_{\ell_1} \rangle = \langle \hat{x}_{\ell_2}, x_{\ell_2} \rangle, \quad \langle \hat{y}_{\ell_1}, y_{\ell_1} \rangle = \langle \hat{y}_{\ell_2}, y_{\ell_2} \rangle,
    \label{eqn:H6}
\end{align}
for every $\ell_1, \ell_2 \in \pp{L}$, as $M \to \infty$. Then, substituting \eqref{eqn:H5}, and \eqref{eqn:H6} into \eqref{eqn:H0} yields,
\begin{align}
    \mathbb{E} \pp{{\langle{\hat{x}_\ell}, x_\ell\rangle}} \geq  \mathbb{E} \pp{{\langle{\hat{y}_\ell}, y_\ell\rangle}} \label{eqn:H7},
\end{align}
which completes the proof.

\section{Proof of Proposition \ref{thm:monotonicity}}

Let $\mathcal{X} = \{x_\ell\}_{\ell=0}^{L_X-1}$ and $\mathcal{Y} = \{y_\ell\}_{\ell=0}^{L_Y-1}$ be sets of normalized template hypotheses. Denote by $\{\hat{x}_\ell\}_{\ell=0}^{L_X-1}$ and $\{\hat{y}_\ell\}_{\ell=0}^{L_Y-1}$ the outputs of Algorithm~\ref{alg:generalizedEfNhard} when applied using the templates $\mathcal{X}$ and $\mathcal{Y}$, respectively.
According to Lemma \ref{lem:convergenceOfHardAssign} and \eqref{eqn:hardAsymptoticObservaionsCorrelationMain}, as $M \to \infty$, we have
\begin{align}
    \langle \hat{x}_\ell, x_\ell \rangle \xrightarrow[]{\text{a.s.}} \frac{\mathbb{E}\left[\langle n_1, x_\ell \rangle \, \mathbbm{1}_{\{n_1 \in \mathcal{V}_\ell^{(x)}\}}\right]}{\mathbb{P}\left[n_1 \in \mathcal{V}_\ell^{(x)}\right]} 
    = \frac{\mathbb{E}\left[\underset{0 \leq r \leq L_X-1} \max \langle n_1, x_r \rangle \cdot \mathbbm{1}_{\{n_1 \in \mathcal{V}_\ell^{(x)}\}}\right]}{\mathbb{P}\left[\hat{\s{R}}^{(x)} = \ell\right]}, \label{eqn:I0}
\end{align}
and similarly,
\begin{align}
    \langle \hat{y}_\ell, y_\ell \rangle \xrightarrow[]{\text{a.s.}} \frac{\mathbb{E}\left[\langle n_1, y_\ell \rangle \, \mathbbm{1}_{\{n_1 \in \mathcal{V}_\ell^{(y)}\}}\right]}{\mathbb{P}\left[n_1 \in \mathcal{V}_\ell^{(y)}\right]} 
    = \frac{\mathbb{E}\left[\underset{0 \leq r \leq L_Y-1} \max \langle n_1, y_r \rangle \cdot \mathbbm{1}_{\{n_1 \in \mathcal{V}_\ell^{(y)}\}}\right]}{\mathbb{P}\left[\hat{\s{R}}^{(y)} = \ell\right]}. \label{eqn:I00}
\end{align}

Substituting \eqref{eqn:I0} into the left-hand side of \eqref{eqn:hardAssignMeanMonotone}, we obtain
\begin{align}
     \sum_{\ell=0}^{L_X-1} {\langle{\hat{x}_\ell}, x_\ell\rangle} \mathbb{P}[\s{\hat{R}}^{(x)} = \ell] & \xrightarrow[]{\s{a.s.}}
     \sum_{\ell=0}^{L_X-1} \frac{\mathbb{E}\pp{{\underset{0 \leq r \leq L_X-1} \max {\langle{n_1, x_r\rangle}}} \mathbbm{1}_{\{n_1 \in \calV_\ell^{(x)}\}}}}{\pr [\s{\hat{R}}^{(x)} = \ell]} \pr [\s{\hat{R}}^{(x)} = \ell] 
     \\ & = \mathbb{E}\pp{{\underset{0 \leq r \leq L_X-1} \max {\langle{n_1, x_r\rangle}}}}.\label{eqn:I1}
\end{align}
Similarly, using \eqref{eqn:I00}, we have
\begin{align}
     \sum_{\ell=0}^{L_Y-1} {\langle{\hat{y}_\ell}, y_\ell\rangle} \mathbb{P}[\s{\hat{R}}^{(y)} = \ell] & \xrightarrow[]{\s{a.s.}}
     \sum_{\ell=0}^{L_Y-1} \frac{\mathbb{E}\pp{{\underset{0 \leq r \leq L_Y-1} \max {\langle{n_1, x_r\rangle}}} \mathbbm{1}_{\{n_1 \in \calV_\ell^{(y)}\}}}}{\pr [\s{\hat{R}}^{(y)} = \ell]} \pr [\s{\hat{R}}^{(y)} = \ell] 
     \\ & = \mathbb{E}\pp{{\underset{0 \leq r \leq L_Y-1} \max {\langle{n_1, y_r\rangle}}}}.\label{eqn:I2}
\end{align}

Now, since by assumption $\{x_\ell\}_{\ell=0}^{L_X-1} \supseteq \{y_\ell\}_{\ell=0}^{L_Y-1}$, it follows that for any realization of $n_1$,
\begin{align}
    \max_{0 \leq r \leq L_X-1} \langle n_1, x_r \rangle \geq \max_{0 \leq r \leq L_Y-1} \langle n_1, y_r \rangle. \label{eqn:I3}
\end{align}
Taking expectations on both sides yields
\begin{align}
    \mathbb{E}\left[\max_{0 \leq r \leq L_X-1} \langle n_1, x_r \rangle\right] 
    \geq 
    \mathbb{E}\left[\max_{0 \leq r \leq L_Y-1} \langle n_1, y_r \rangle\right]. \label{eqn:I4}
\end{align}

Combining \eqref{eqn:I1}, \eqref{eqn:I2}, and \eqref{eqn:I4}, we conclude that
\begin{align}
    \sum_{\ell=0}^{L_X-1} \langle \hat{x}_\ell, x_\ell \rangle \, \mathbb{P}\left[\hat{\s{R}}^{(x)} = \ell\right] 
    \geq 
    \sum_{\ell=0}^{L_Y-1} \langle \hat{y}_\ell, y_\ell \rangle \, \mathbb{P}\left[\hat{\s{R}}^{(y)} = \ell\right],
\end{align}
which completes the proof.

\section{Proof of Theorem \ref{thm:hardAssignmentTwoHypoteses}}
To prove Theorem \ref{thm:hardAssignmentTwoHypoteses} we will combine a few facts. First, by applying Lemma \ref{lemma:twoTemplatesExpectedValueOfMax}, we show that for $M \to \infty$, we have,
        \begin{align}
            {{\langle{\hat{x}_\ell}, x_\ell\rangle}} \xrightarrow[]{\s{a.s.}} \sqrt{\frac{1-\rho}{\pi}} \norm{x}_2 \label{eqn:recapOfTheorem1},   
        \end{align}
for $\ell=0,1$. To that end, recall from Lemma \ref{lem:convergenceOfHardAssign} and \eqref{eqn:hardAsymptoticObservaionsCorrelationMain} that,
    \begin{align}
        {\langle{\hat{x}_\ell}, x_\ell\rangle} \xrightarrow[]{\s{a.s.}} \frac{\mathbb{E}\pp{\langle{n_i, x_\ell\rangle} \mathbbm{1}_{\ppp{n_1 \in \calV_\ell}}}}{\mathbb{P}\pp{n_1 \in \calV_\ell}}, \label{eqn:recapOfHardAsymptoticCorrelation}
    \end{align}
for $\ell=0,1$. For $L=2$ we clearly have that
$\mathbb{P}\pp{n_1 \in \calV_0} = \mathbb{P}\pp{n_1 \in \calV_1} = 1/2$. Now, note that,
\begin{align}
    \mathbb{E}\pp{{\underset{r\in\{0,1\}} \max {\langle{n_1, x_r\rangle}}}} = \mathbb{E}\pp{{\underset{r\in\{0,1\}} \max {\langle{n_i, x_r\rangle}}} \mathbbm{1}_{\ppp{n_i \in \calV_0}}} + \mathbb{E}\pp{{\underset{r\in\{0,1\}} \max {\langle{n_1, x_r\rangle}}} \mathbbm{1}_{\ppp{n_1 \in \calV_1}}}. \label{eqn:splitToTwoEvents}
\end{align}
Thus, by symmetry, since the first and second terms at the r.h.s. of \eqref{eqn:splitToTwoEvents} are equal, we have,
    \begin{align}
        \mathbb{E}\pp{\langle{n_1, x_\ell\rangle} \mathbbm{1}_{\ppp{n_1 \in \calA_\ell}}} = \frac{1}{2} \mathbb{E}\pp{{\underset{r\in\{0,1\}} \max {\langle{n_1, x_r\rangle}}}} \label{eqn:numeratorHalfOfMaximum},
    \end{align}
for $\ell=0,1$. Substituting \eqref{eqn:numeratorHalfOfMaximum} in \eqref{eqn:recapOfHardAsymptoticCorrelation} leads to,
    \begin{align}
        {\langle{\hat{x}_\ell}, x_\ell\rangle} \xrightarrow[]{\s{a.s.}} \mathbb{E}\pp{{\underset{r\in\{0,1\}} \max {\langle{n_1, x_r\rangle}}}} \label{eqn:relationOfCorrelationToMaximum}.  
    \end{align}
Recall the definitions of $\s{S}^{(x)}$ in \eqref{eqn:SxDef}. Applying Lemma \ref{lemma:twoTemplatesExpectedValueOfMax} on ${\underset{r=0,1} \max \ppp{S_r^{(x)}}}$, which satisfies the conditions of the lemma, we get that,
    \begin{align}
        \mathbb{E}\pp{{\underset{r\in\{0,1\}} \max {S_r^{(x)}}}} = \mathbb{E}\pp{{\underset{r\in\{0,1\}} \max {\langle{n_1, x_r\rangle}}}} / \norm{x}_2  = \sqrt{\frac{1-\rho}{\pi}} > 0,
    \end{align}
which proves \eqref{eqn:recapOfTheorem1}. 

To prove Theorem \ref{thm:hardAssignmentTwoHypoteses}, we will use Theorem \ref{thm:hardAssignmentLinearCombination} (which holds for $L=2$ as well, and which we will prove in the following section), which states that,
\begin{equation}
\begin{aligned}\label{eqn:estimatorsLinearCombination}
    \hat{x}_0 &\xrightarrow[]{\s{a.s.}} \alpha_{11} x_0 + \alpha_{12} x_1, \\
    \hat{x}_1 &\xrightarrow[]{\s{a.s.}} \alpha_{21} x_0 + \alpha_{22} x_1,
\end{aligned}
\end{equation}
as $M\to\infty$. From \eqref{eqn:hardAsymptoticObservaionsEstimator}, we have,
    \begin{align}
        \hat{x}_0 + \hat{x}_1 \xrightarrow[]{\s{a.s.}} \frac{\mathbb{E}\pp{n_1\mathbbm{1}_{\ppp{n_1 \in \calV_0}}}}{\mathbb{P}\pp{n_1 \in \calV_0}} + \frac{\mathbb{E}\pp{n_i\mathbbm{1}_{\ppp{n_1 \in \calV_1}}}}{\mathbb{P}\pp{n_1 \in \calV_1}} = 2\cdot\mathbb{E}\pp{n_i} = 0, \label{eqn:sumOfEstimatorIsZero}
    \end{align}
where the last step is due to that fact that $\mathbb{P}\pp{n_i \in \calV_0} = \mathbb{P}\pp{n_i \in \calV_1} = 1/2$. Since we assume that $\rho={\langle{x_0}, x_1\rangle}/\norm{x}_2<1$, combining the set of linear equations in \eqref{eqn:recapOfTheorem1}, \eqref{eqn:estimatorsLinearCombination}, and \eqref{eqn:sumOfEstimatorIsZero}, we can apply the continuous mapping theorem to extract the coefficients $\ppp{\alpha_{ij}}_{i,j=1}^2$; A straightforward algebra reveals that,
     \begin{align}
         \alpha_{11} = \alpha_{22} = -\alpha_{12} = -\alpha_{21} = \sqrt{\frac{1}{\pi\p{1-\rho}\norm{x}_2^2}} \label{eqn:linearCoefForL2}.
    \end{align}
The proof is concluded by combining \eqref{eqn:estimatorsLinearCombination} and \eqref{eqn:linearCoefForL2}.

\section{Proof of Theorem \ref{thm:softAssignmentTwoTemplates}}\label{eqn:softTwoAppro}
Recall from Algorithm \ref{alg:generalizedEfNsoft} that, 
    \begin{align}
       \hat{x}_0 = \frac{\frac{1}{M} \sum_{i=0}^{M-1}{n_i p_{i}^{(0)}}}{\frac{1}{M} \sum_{i=0}^{M-1}{p_{i}^{(0)}}} \label{eqn:softAssignmentEstimatorForTwoTemplates}.
    \end{align}
For the numerator, we have,
    \begin{align}
       \nonumber \frac{1}{M} \sum_{i=0}^{M-1}{n_i p_{i}^{(0)}} & = \frac{1}{M} \sum_{i=0}^{M-1}{ \frac{n_i\exp\p{{\langle{n_i, x_0\rangle}}}}{\exp\p{{\langle{n_i, x_0\rangle}}} + \exp\p{{\langle{n_i, x_1\rangle}}}}} \\ & = \frac{1}{M} \sum_{i=0}^{M-1}{ \frac{n_i}{1 + \exp\p{{\langle{n_i, x_1 - x_0\rangle}}}}}.
    \end{align}
By the SLLN, we thus get,
    \begin{align}
       \frac{1}{M} \sum_{i=0}^{M-1}{ \frac{n_i}{1 + \exp\ppp{{\langle{n_i, x_1 - x_0\rangle}}}}} \xrightarrow[]{\s{a.s.}} \mathbb{E}\pp{\frac{n_1}{1+\exp\p{{\langle{n_1, x_1 - x_0\rangle}}}}} \label{eqn:SLLNnumeratorTwoTemplates}.
    \end{align}
Similarly, for the denominator in \eqref{eqn:softAssignmentEstimatorForTwoTemplates}, by the SLLN, we have,
    \begin{align}
       \nonumber \frac{1}{M} \sum_{i=0}^{M-1}{p_{i}^{(0)}} & = \frac{1}{M} \sum_{i=0}^{M-1}{ \frac{\exp\p{{\langle{n_i, x_0\rangle}}}}{\exp\p{{\langle{n_i, x_0\rangle}}} + \exp\p{{\langle{n_i, x_1\rangle}}}}} \\ \nonumber & = \frac{1}{M} \sum_{i=0}^{M-1}{ \frac{1}{1 + \exp\p{{\langle{n_i, x_1 - x_0\rangle}}}}} \\ & \xrightarrow[]{\s{a.s.}} \mathbb{E}\pp{ \frac{1}{1 + \exp\p{{\langle{n_1, x_1 - x_0\rangle}}}}}.
    \end{align}
Since ${\langle{n_1, x_1 - x_0\rangle}} \sim \calN \p{0, \norm{x_1-x_0}_2^2}$, then, 
    \begin{align}
       \mathbb{E}\pp{ \frac{1}{1 + \exp\p{{\langle{n_1, x_1 - x_0\rangle}}}}} = \frac{1}{2}.\label{eqn:denominatorOfTwoTemplatesExpectationValue}
    \end{align}
Thus, combining \eqref{eqn:softAssignmentEstimatorForTwoTemplates}, \eqref{eqn:SLLNnumeratorTwoTemplates}, and \eqref{eqn:denominatorOfTwoTemplatesExpectationValue}, and applying the continuous mapping theorem we get \eqref{eqn:softBefApp1}. To prove \eqref{eqn:softBefApp2}, we notice that, by definition, $\hat{x}_0 + \hat{x}_1 = \frac{1}{M}\sum_{i=0}^{M-1}n_i \xrightarrow[]{\s{a.s.}} 0$, and thus, the result is obtained by using \eqref{eqn:denominatorOfTwoTemplatesExpectationValue}.

Finally, we establish the approximation in \eqref{eqn:x0hatTwoTemplates}. The logistic function can be approximated by the error function as follows \cite{crooks2009logistic},
    \begin{align}
        \frac{1}{1+\exp \p{-x}} \approx \frac{1}{2} + \frac{1}{2}\s{erf} \p{\frac{\sqrt{\pi}}{4}x}.
    \end{align}
Therefore, the expected value in \eqref{eqn:SLLNnumeratorTwoTemplates} can be approximated by,
    \begin{align}
       \mathbb{E}\pp{\frac{n_1}{1+\exp\ppp{{\langle{n_1, x_1 - x_0\rangle}}}}} &\approx \frac{1}{2}\mathbb{E}\pp{n_1} + \frac{1}{2}\mathbb{E}\pp{{n_1} \s{erf}\p{\frac{\sqrt{\pi}}{4}{\langle{n_1, x_0 - x_1\rangle}}}}\\
       & = \frac{1}{2}\mathbb{E}\pp{{n_1} \s{erf}\p{\frac{\sqrt{\pi}}{4}{\langle{n_1, x_0 - x_1\rangle}}}}\label{eqn:apprxoimationOfLogisticByErf}.
    \end{align}
It is known that \cite{ng1969table},
    \begin{align}
       \int_{-\infty}^{\infty} {\s{erf}\p{ax+b} \frac{1}{\sqrt{2\pi \sigma^2}}e^{-\frac{\p{x-\mu}^2}{2\sigma^2}} \mathrm{d}x} &= \s{erf}\p{\frac{a\mu+b}{\sqrt{1+2a^2\sigma^2}}} \label{eqn:erfIntegralIdentity1}\\
       \int_{-\infty}^{\infty} x \cdot {\s{erf}\p{ax} \frac{1}{\sqrt{2\pi \sigma^2}}e^{-\frac{{x}^2}{2\sigma^2}} \mathrm{d}x} &= \frac{2a{\sigma}^{2}}{\sqrt{\pi} \sqrt{1+2a^{2} {\sigma}^{2}}} \label{eqn:erfIntegralIdentity2}.
    \end{align} 
Using these identities in \eqref{eqn:apprxoimationOfLogisticByErf}, leads to,
    \begin{align}
       \frac{1}{2}\mathbb{E}\pp{{n_1} \s{erf}\p{\frac{\sqrt{\pi}}{4}{\langle{n_1, x_0 - x_1\rangle}}}} = \frac{1}{4}\frac{x_0 - x_1}{\sqrt{1+\frac{\pi}{4}\p{1-\rho}\norm{x}_2^2}} \label{eqn:nominatorOfTwoTemplatesExpectationValue},
    \end{align}
which together with \eqref{eqn:apprxoimationOfLogisticByErf} proves \eqref{eqn:x0hatTwoTemplates}. The approximation in \eqref{eqn:x1hatTwoTemplates} follows from the fact that $\hat{x}_1 \xrightarrow[]{\s{a.s.}} -\hat{x}_0$. 

\section{Proof of Theorem \ref{thm:hardAssignmentLinearCombination}}

We start with the hard-assignment procedure, with the understanding that the proof for the soft-assignment procedure is similar. We prove that for every $0\leq\ell\leq L-1$, the estimator $\hat{x}_\ell$ converges to a linear combination of the templates almost surely. From \eqref{eqn:hardAsymptoticObservaionsEstimator}, we have,
\begin{align}
        \nonumber \hat{x}_\ell & \triangleq \frac{1}{\abs{\calA_\ell}} \sum_{n_i \in \calA_\ell} {n_i} 
        \\ & = \frac{\frac{1}{M} \sum_{i=0}^{M-1} n_i \mathbbm{1}_{\ppp{n_i \in \calV_\ell}}}{\abs{\calA_\ell}/M}
        \xrightarrow[]{\s{a.s.}} \frac{\mathbb{E} \pp{n_1 \mathbbm{1}_{\ppp{n_1 \in \calV_\ell}}}}{\mathbb{P}\pp{n_1 \in \calV_\ell}} \label{eqn:asymptoticForGeneralizedEfNvector},
    \end{align}
as $M\to\infty$. The numerator term at the  r.h.s. of \eqref{eqn:asymptoticForGeneralizedEfNvector} is simply the expected value of all vectors which are closest to the vector $x_\ell$, relative to the other vectors in $\mathbb{R}^{d}$. Our aim is to show that $\mathbb{E} \pp{n_1 \mathbbm{1}_{\ppp{n_1 \in \calV_\ell}}}$ is a linear combination of the templates $\ppp{x_\ell}_{\ell=0}^{L-1}$, which would lead to the claimed result. 
We define $y_\ell \triangleq \mathbb{E}\pp{n_1 \mathbbm{1}_{\ppp{n_1 \in \calV_\ell}}}$. Combining \eqref{eqn:asymptoticForGeneralizedEfNvector} and \eqref{eqn:VlDef} leads to,
    \begin{align}
         y_\ell = \int_{v \in \calV_\ell} v f_{n_1}\p{v}\mathrm{d}v,\label{eqn:asymptoticEstimatorIntegralForm}
    \end{align}
where $f_{n_1}$ is the probability density function of $n_1$.

We consider two possible cases. If $\s{dim}\ppp{\s{span}\ppp{x_0, x_1,\ldots, x_{L-1}}} = d$, then it is clear that ${\hat{x}}_l \in \s{span}\ppp{x_0, x_1,\ldots, x_{L-1}}$, and thus the result follows trivially. Therefore, we assume that $r \triangleq \s{dim}\ppp{\s{span}\ppp{x_0, x_1,\ldots, x_{L-1}}} < d$. Therefore, $y_\ell$ can be represented as,
    \begin{align}
        y_\ell = \sum_{k=0}^{L-1}\alpha_{kl}x_\ell + \sum_{i=1}^{d-r}\beta_i u_i \label{eqn:efnEstimatorDecomposition},
    \end{align}
where $\s{span} \{\ppp{u_i}_{i=0}^{d-r}\}$ is the complement space of 
$\s{span}\ppp{x_0, x_1,\ldots, x_{L-1}}$, satisfying ${\langle{u_i, x_k\rangle}} = 0$ and ${\langle{u_i, u_j\rangle}} = 0$, for every $1 \leq i < j < d-r$ and $0 \leq k \leq L-1$. We next show that it must be the case that $\beta_i = 0$ for every $1 \leq i < j < d-r$. Indeed, following \eqref{eqn:efnEstimatorDecomposition}, and by orthonormality, 
    \begin{align}
        \beta_i = {\langle{y_\ell, u_i\rangle}}.
    \end{align}
From \eqref{eqn:asymptoticEstimatorIntegralForm}, we have,
    \begin{align}
        \beta_i = {\langle{y_\ell, u_i\rangle}} \xrightarrow[]{\s{a.s.}} \int_{v \in \calV_\ell} {\langle{v, u_i\rangle}} f_{n_1}\p{v}\mathrm{d}v \label{eqn:betaCoefficentAsymptotic}.
    \end{align}
We show that $\int_{v \in \calV_\ell} {\langle{v, u_i\rangle}} f_{n_1}\p{v} \cdot \mathrm{d}v = 0$. To that end, we note that each $v \in \calV_\ell$ can be decomposed into a parallel and orthogonal components as follows,
    \begin{align}
        v = \gamma_v u_i + w_v \label{eqn:vDecomposition},
    \end{align}
where $w_v \in \mathbb{R}^d$ is such that ${\langle{w_v, u_i\rangle}} = 0$, and $\gamma_v \in \mathbb{R}$. Since ${\langle{u_i, x_k\rangle}} = 0$, for every $0 \leq k \leq L-1$, if the vector $v\in\calV_\ell$, then it must be that $v' \triangleq -\gamma_v u_i + w_v\in\calV_\ell$. Indeed, this follows because,
    \begin{align}
        \nonumber v = \gamma_v u_i + w_v \in \calV_\ell & \Longleftrightarrow {\langle{\gamma_v u_i + w_v, x_\ell\rangle}} \geq \underset{k\neq \ell} {\max} {{\langle{\gamma_v u_i + w_v, x_k\rangle}}} \\ \nonumber &  \Longleftrightarrow {\langle{w_v, x_\ell\rangle}} \geq \underset{k\neq \ell} {\max} {{\langle{ w_v, x_k\rangle}}} \\ \nonumber & \Longleftrightarrow {\langle{-\gamma_v u_i + w_v, x_\ell\rangle}} \geq \underset{k\neq \ell} {\max} {{\langle{-\gamma_v u_i + w_v, x_k\rangle}}} \\ & \Longleftrightarrow v' = -\gamma_v u_i + w_v \in \calV_\ell.
    \end{align}
Thus, for each vector $v$ as in $\eqref{eqn:vDecomposition}$, there is a corresponding vector $v' = -\gamma_v u_i + w_v$ which is also in $\calV_\ell$. In addition, $v$ and $v'$ have the same norm, thus $f_{n_1}\p{v} = f_{n_1}\p{v'}$. Now, note that \eqref{eqn:betaCoefficentAsymptotic} can be rewritten as,
    \begin{align}
        \int_{v \in \calV_\ell} {\langle{v, u_i\rangle}} f_{n_1}\p{v} \mathrm{d}v =  \int_{v \in \calV_\ell: {\langle{v, u_i\rangle}} > 0} {\langle{v, u_i\rangle}} f_{n_1}\p{v} \mathrm{d}v + \int_{v \in \calV_\ell: {\langle{v, u_i\rangle}} < 0} {\langle{v, u_i\rangle}} f_{n_1}\p{v} \mathrm{d}v \label{eqn:pluggingGammaVintoTheIntegral}.
    \end{align}
Thus, since for every $v \in \calV_\ell$ there is a unique corresponding $v' \in \calV_\ell$ such that ${\langle{v, u_i\rangle}} = -{\langle{v', u_i\rangle}}$ and $f_{n_1}\p{v} = f_{n_1}\p{v'}$, it is clear that the r.h.s. of \eqref{eqn:pluggingGammaVintoTheIntegral} is zero. Therefore, combining \eqref{eqn:betaCoefficentAsymptotic} and \eqref{eqn:pluggingGammaVintoTheIntegral} leads to $\beta_i \xrightarrow[]{\s{a.s.}} 0$, for every $i$. Finally, plugging this in \eqref{eqn:efnEstimatorDecomposition}, and then in \eqref{eqn:asymptoticForGeneralizedEfNvector}, leads to,
    \begin{align}
        \hat{x}_\ell \xrightarrow[]{\s{a.s.}} \left.\p{\sum_{k=0}^{L-1}\alpha_{kl}x_\ell}\middle /\mathbb{P}\pp{n_1 \in \calV_\ell}\right.,
    \end{align}
which completes the proof.

\section{Expression for the linear coefficients} \label{sec:expressionForLinearCoefficents}
We derive the analytical expressions for the linear coefficients in Theorem \ref{thm:hardAssignmentLinearCombination}. Recalling the definition of $\s{S}^{(x)}$ from \eqref{eqn:SxDef}, explicitly given by,
\begin{align}
    {S}_{\ell}^{(x)} \triangleq \langle{n, x_{\ell}\rangle}, 
\end{align}
where $n \sim \mathcal{N} \p{0, I_{d \times d}}$ is a $d$-dimensional isotropic Gaussian vector. We will denote by $n_i$ the $i$-th entry of the vector $n$, for $0 \leq i \leq d-1$, and recall we assume that $\norm{x_\ell} = 1$, for every $\ell \in \pp{L}$. In addition, recall the definition of  $p_{\ell,\beta}$, \eqref{eqn:softMaxFuncDef1},
\begin{align}
   p_{\ell,\beta}\p{\s{S}^{(x)}} = \frac{ \exp{\p{\beta S_\ell^{(x)}}}}{\sum_{r=0}^{L-1}\exp{\p{\beta S_r^{(x)}}}} 
   = \frac{ \exp{\p{\beta \langle{n, x_{\ell}\rangle}}}}{\sum_{r=0}^{L-1}\exp{\p{\beta \langle{n, x_{r}\rangle}}}}.\label{eqn:K2}
\end{align}
From \eqref{eqn:softAsymptoticObservaionsEstimator}, we note that,
\begin{align}
    \hat{x}_\ell^{(\beta)} \xrightarrow[]{\s{a.s.}} 
    \frac{\mathbb{E}\pp{n \cdot p_{\ell, \beta} \p{\s{S}^{(x)}}}}{\mathbb{E}\pp{p_{\ell, \beta} \p{\s{S}^{(x)}}}}.\label{eqn:K3}
\end{align}
Using the Gaussian integration by parts lemma (see, e.g., \cite{ross2011fundamentals}), we have,
\begin{align}
    \mathbb{E}\pp{n_i \cdot p_{\ell, \beta} \p{\s{S}^{(x)}}} =  \sum_{j=0}^{d-1}\mathbb{E}\p{n_in_j} \mathbb{E}\p{\frac{\partial p_{\ell, \beta}}{\partial n_j} \p{\s{S}^{(x)}}}. \label{eqn:K4}
\end{align}
As $n$ has an identity covariance matrix, we have $\mathbb{E}\pp{n_i n_j} = \delta_{ij}$, giving the vector representation:
\begin{align}
    \mathbb{E}\pp{n \cdot p_{\ell, \beta} \p{\s{S}^{(x)}}} =   \mathbb{E}\pp{\frac{\partial p_{\ell, \beta}}{\partial n} \p{\s{S}^{(x)}}}. \label{eqn:K5}
\end{align}
The derivative in the right-hand-side of \eqref{eqn:K4} is given by
\begin{align}
    \frac{1}{\beta} \frac{\partial p_{\ell, \beta}}{\partial n} \p{\s{S}^{(x)}} = x_\ell \cdot p_{\ell, \beta} \p{\s{S}^{(x)}} - \sum_{r=0}^{L-1} x_r \cdot p_{\ell, \beta} \p{\s{S}^{(x)}} p_{r, \beta} \p{\s{S}^{(x)}}. \label{eqn:K6}
\end{align}
Combining \eqref{eqn:K5} and \eqref{eqn:K6}, and substituting into \eqref{eqn:K3} results,
\begin{align}
    {\hat{x}_\ell^{(\beta)}} \xrightarrow[]{\s{a.s.}} 
    \beta\p{x_\ell - \sum_{r=0}^{L-1} x_r \frac{\mathbb{E}\pp{p_{r, \beta}\p{\s{S}^{(x)}}p_{\ell, \beta}\p{\s{S}^{(x)}}}}{\mathbb{E}\pp{p_{\ell, \beta}\p{\s{S}^{(x)}}}}}\label{eqn:K7}.
\end{align}
Thus, for the linear coefficients in Theorem \ref{thm:hardAssignmentLinearCombination}, we obtain
\begin{align}
    \alpha_{k \ell} = \beta \p{\delta_{k \ell} - \frac{\mathbb{E}\pp{p_{r, \beta}\p{\s{S}^{(x)}}p_{\ell, \beta}\p{\s{S}^{(x)}}}}{\mathbb{E}\pp{p_{\ell, \beta}\p{\s{S}^{(x)}}}} }.
\end{align}
Specifically, the soft-assignment estimates correspond to $\beta = 1$, as defined in \eqref{eqn:softAssignmentUpdateStep}, while the hard-assignment case arises in the limit $\beta \to \infty$, as stated in Proposition \ref{thm:softAssignmentExtremeNormalization}.

\section{Approximation for a finite number of templates} \label{sec:approximationForFiniteNumberOfTemplates}
We derive an approximation for the soft-assignment estimator. First, we already know that,
    \begin{align}
       \hat{x}_\ell \xrightarrow[]{\s{a.s.}} \frac{\mathbb{E}\pp{n_1 p_1^{(\ell)}}}{\mathbb{E}\pp{p_1^{(\ell)}}} \label{eqn:sllnSoftEstimator},
    \end{align}
where $p_1^{(\ell)}$ is defined in \eqref{eqn:softMaxProbabilityDist}. We find an approximations for the numerator and denominator of \eqref{eqn:sllnSoftEstimator}. Our approximation is based on the following approximation of the expected value of the ratio between two random variables $A$ and $B$ \cite{stuart2010kendall},
    \begin{align}
       \mathbb{E} \pp{\frac{A}{B}} \approx \frac{\mu_A}{\mu_B} \p{1 - \frac{\s{cov} \p{A,B}}{\mu_A \mu_B} + \frac{\sigma_B^2}{\mu_B^2}} = \frac{\mu_A}{\mu_B} - \frac{\s{cov}\p{A,B}}{\mu_B^2} + \frac{\sigma_B^2 \mu_A}{\mu_B^3} \label{eqn:ratioExpectationApproximation},
    \end{align}
where $\mu_A$, $\mu_B$, $\sigma_A^2$, and $\sigma_B^2$, are the means and variances of $A$ and $B$, respectively, and $\s{cov}(A,B)$ is the covariance between $A$ and $B$. 

We first approximate the numerator in \eqref{eqn:sllnSoftEstimator}. Accordingly, we define, 
    \begin{align}
       A \triangleq n_1 e^{{\,\langle{n_1}, x_\ell \rangle}},
    \end{align}
and
    \begin{align}
       B \triangleq \sum_{r=0}^{L-1} e^{{\,\langle{n_1}, x_r \rangle}}.
    \end{align}
Note that here $A$ is a vector. We have,
    \begin{align}
       \mu_A = \mathbb{E}\pp{ n_1 e^{{\,\langle{n_1}, x_\ell \rangle}}} = x_\ell e^{\frac{\norm{x_\ell}_2^2}{2}},
    \end{align}
and
    \begin{align}
       \mu_B = \mathbb{E}\pp{\sum_{r=0}^{L-1} e^{{\,\langle{n_1}, x_r \rangle}}} = L\cdot e^{\frac{\norm{x_\ell}_2^2}{2}}.
    \end{align}
Also,
    \begin{align}
       \nonumber \mathbb{E}\pp{AB} & = \mathbb{E}\pp{n_1 e^{{\,\langle{n_1}, x_\ell \rangle}} \sum_{r=0}^{L-1} e^{{\,\langle{n_1}, x_r \rangle}}} \\ & =  \nonumber \sum_{r=0}^{L-1} \mathbb{E}\pp{n_1 e^{{\,\langle{n_1}, x_\ell+x_r \rangle}}} \\ & = \nonumber \sum_{r=0}^{L-1} \p{x_\ell+x_r} e^{\frac{\norm{x_\ell+x_r}_2^2}{2}} \\ & = \sum_{r=0}^{L-1} \p{x_\ell+x_r} e^{\norm{x_\ell}_2^2 + {\,\langle{x_\ell}, x_r \rangle}}.
    \end{align}
Therefore,
    \begin{align}
       \s{cov}\p{A,B} = \sum_{r=0}^{L-1} \p{x_\ell+x_r} e^{\norm{x}_2^2 + {\,\langle{x_\ell}, x_r \rangle}} - L \cdot x_\ell e^{\norm{x_\ell}_2^2}.
    \end{align}
Finally,
        \begin{align}
        \nonumber \mathbb{E}\pp{B^2} & = \mathbb{E}\pp{\p{\sum_{r=0}^{L-1} e^{{\,\langle{n_1}, x_r \rangle}}}^2}\\
        &=  \sum_{r_1,r_2=0}^{L-1} \mathbb{E}\pp{ e^{{\,\langle{n_1}, x_{r_1} + x_{r_2} \rangle}}}  \\ 
        & = \sum_{r_1,r_2=0}^{L-1} e^{\frac{\norm{x_{r_1} + x_{r_2}}_2^2}{2}}\\
        &= \sum_{r_1,r_2=0}^{L-1} e^{\norm{x_\ell}_2^2 + {\,\langle{x_{r_1}}, x_{r_2} \rangle}}.
    \end{align}
Thus,
    \begin{align}
       \sigma_B^2 = \sum_{r_1,r_2=0}^{L-1} e^{\norm{x_\ell}_2^2 + {\,\langle{x_{r_1}}, x_{r_2} \rangle}} - L^2 \cdot e^{\norm{x_\ell}_2^2}.
    \end{align}
Combining the above results with \eqref{eqn:ratioExpectationApproximation} leads to,
    \begin{align}
        \mathbb{E}\pp{n_1 p_1^{(\ell)}} & \approx \frac{\mu_A}{\mu_B} - \frac{\s{cov}\p{A,B}}{\mu_B^2} + \frac{\sigma_B^2 \mu_A}{\mu_B^3} \\ 
       &  = \frac{1}{L} \pp{x_\ell \p{1 - \frac{1}{L} \sum_{r=0}^{L-1} e^{{\,\langle{x_\ell}, x_r \rangle}} + \frac{1}{L^2} \sum_{r_1,r_2=0}^{L-1}e^{{\,\langle{x_{r_1}}, x_{r_2} \rangle}}}- \frac{1}{L} \sum_{r=0}^{L-1} x_re^{{\,\langle{x_\ell}, x_r \rangle}}}.\label{eqn:appro2}
    \end{align}
Using the same arguments, we can analyze the denominator in \eqref{eqn:sllnSoftEstimator}, and get,
    \begin{align}
        \mathbb{E}\pp{p_1^{(l)}} \approx \frac{1}{L} \pp{1 - \frac{1}{L} \sum_{r=0}^{L-1}e^{{\,\langle{x_\ell}, x_r \rangle}} + \frac{1}{L^2}\sum_{r_1,r_2=0}^{L-1} e^{{\,\langle{x_{r_1}}, x_{r_2} \rangle}}}.\label{eqn:appro3}
    \end{align}
Combining \eqref{eqn:sllnSoftEstimator}, \eqref{eqn:appro2}, and \eqref{eqn:appro3}, we obtain \eqref{eqn:approximationSoftForFiniteL}, as required.

\section{Proof of Theorem \ref{thm:hardAssignmentAsymptoticLandAsymptoticD}}
Recall the definition of $\s{S}^{(x)}$ as defined in \eqref{eqn:SxDef}. Note that the entries of the covariance matrix of $\s{S}^{(x)}$ are given by $\sigma_{ij} = \langle{x_i, x_j \rangle} / \norm{x_\ell}_2^2$. By Assumption \ref{assump:1}, $\s{S}^{(x)}$ satisfies the conditions of the first part of Proposition~\ref{prop:expectedValueOfTwoSoftMaxAsymptoticBetaAsymptoticL2}. Then, by \eqref{eqn:A66},
\begin{align}
    \lim_{L \to \infty} \frac{\mathbb{E}\left[\underset {0 \leq \ell \leq L-1} \max \langle n_1, x_\ell \rangle\right] }{\sqrt{2 \log {L}}}  = 1. \label{eqn:O01}
\end{align}

Similarly to \eqref{eqn:I2}, we have for $M \to \infty$,
\begin{align}
     \sum_{\ell=0}^{L-1} {\langle{\hat{x}_\ell}, x_\ell\rangle} \mathbb{P}[\s{\hat{R}}^{(x)} = \ell] & \xrightarrow[]{\s{a.s.}}
     \sum_{\ell=0}^{L-1} \frac{\mathbb{E}\pp{{\underset{0 \leq r \leq L-1} \max {\langle{n_1, x_r\rangle}}} \mathbbm{1}_{\{n_1 \in \calV_\ell^{(x)}\}}}}{\pr [\s{\hat{R}}^{(x)} = \ell]} \pr [\s{\hat{R}}^{(x)} = \ell] 
     \\ & = \mathbb{E}\pp{{\underset{0 \leq r \leq L-1} \max {\langle{n_1, x_r\rangle}}}}.\label{eqn:O02}
\end{align}
Thus, combining \eqref{eqn:O01}, and \eqref{eqn:O02}, yields,
    \begin{align}
        \lim_{d,L \to\infty} \lim_{M \to\infty} \frac{1}{\sqrt{2 \log L}}  \sum_{\ell=0}^{L-1}\langle{\hat{x}_\ell},{x_\ell}\rangle \mathbb{P}[\hat{\s{R}} = \ell]  = 1, \label{eqn:eqnO03}
    \end{align}
which proves \eqref{eqn:hardAssignmentAsymptoticLandAsymptoticD1}.

Next, we prove \eqref{eqn:hardAssignmentAsymptoticLandAsymptoticD2}. To that end, we apply the second part of Proposition \ref{prop:expectedValueOfTwoSoftMaxAsymptoticBetaAsymptoticL2}, and we note that its conditions are satisfied under Assumptions \ref{assump:0} and \ref{assump:1}, for the Gaussian vector $\s{S}^{(x)}$, defined above. Therefore, we have,
\begin{align}
     \lim_{L \to \infty} \frac{1}{\sqrt{2 \log L}} \cdot \frac{\mathbb{E}\pp{\langle{n_1, x_\ell\rangle} \mathbbm{1}_{\ppp{n_1 \in \calV_\ell}}}}{\mathbb{P}\pp{n_1 \in \calV_\ell}} = 1, \label{eqn:closetTemplateIsCorrespondingTemplateA219}
\end{align}
for every $\ell \in \pp{L}$. 
Combining \eqref{eqn:hardAsymptoticObservaionsEstimator}, and \eqref{eqn:closetTemplateIsCorrespondingTemplateA219}, we obtain,
\begin{align}
     \lim_{d,L \to \infty} \lim_{M \to \infty} \frac{1}{\sqrt{2 \log L}}  \cdot \langle \hat{x}_\ell, x_\ell \rangle 
     & = \lim_{d,L \to \infty} \frac{1}{\sqrt{2 \log L}} \cdot \frac{\mathbb{E}\pp{\langle{n_1, x_\ell\rangle} \mathbbm{1}_{\ppp{n_1 \in \calV_\ell}}}}{\mathbb{P}\pp{n_1 \in \calV_\ell}} = 1, \label{eqn:closetTemplateIsCorrespondingTemplateA220}
\end{align}
which concludes the proof.

\section{Proof of Theorem \ref{thm:softAssignmentAsymptoticL}}
To prove Theorem \ref{thm:softAssignmentAsymptoticL}, we will use Bernstein's law of large numbers \cite{durrett2019probability}. 
\begin{proposition} \label{prop:4}(Bernstein's LLN)
Let $ Y_1, Y_2, \ldots $ be a sequence of random variables with finite expectation $\mathbb{E}\p{Y_j} = \mu < \infty$, and uniformly bounded variance $\text{Var} \p{Y_j} < K < \infty$ for every $j \geq 1$, and $\text{Cov} \p{Y_i, Y_j} \to 0$, as $\abs{i-j} \to \infty$. Then,
    \begin{align}
       \frac{1}{n}\sum_{i=1}^{n} Y_i \xrightarrow[]{\calP} \mu,
    \end{align}
as $n \to \infty$.
\end{proposition}
We also use the following result, which we prove later on in this subsection.
\begin{lem} \label{lemma:5}
Assume ${\,\langle{x_{\ell_1}}, x_{\ell_2} \rangle} \to 0$, as $\abs{\ell_1-\ell_2} \to \infty$. Then,
    \begin{align}
       \frac{1}{L}\sum_{r=0}^{L-1} e^{{\,\langle{n_1}, x_r \rangle}}\xrightarrow[]{\calP}\mathbb{E}\pp{e^{\,\langle{n_1}, x_1 \rangle}} = {e^{\frac{\norm{x}_2^2}{2}}},
       \label{eqn:lemma5}
    \end{align}
as $L \to \infty$.
\end{lem}
Define the sequence of random variables,
    \begin{align}
      Z_L \triangleq n_1p_1^{(\ell)}= { \frac{n_1 e^{{\,\langle{n_1}, x_\ell \rangle}}}{\frac{1}{L}\sum_{r=0}^{L-1} e^{{\,\langle{n_1}, x_r \rangle}}}}, \label{eqn:asymptoticLsoftAssignA164}
    \end{align}
indexed by $L$. Lemma \ref{lemma:5} implies that the denominator in \eqref{eqn:asymptoticLsoftAssignA164} converges, as $L\to\infty$, to ${e^{\frac{\norm{x}_2^2}{2}}} > 0$ in probability. Thus, applying the continuous mapping theorem, 
    \begin{align}
      Z_L \xrightarrow[]{\calP} n_1 e^{{\,\langle{n_1}, x_\ell \rangle}-\frac{\norm{x}_2^2}{2}},
    \end{align}
as $L \to \infty$. Now, since the sequence $\{Z_L\}_L$ is uniformly integrable for every $L$ (note that $\abs{Z_L} \leq \abs{n_1}$ and $\mathbb{E}\abs{n_1} < \infty$), we also have that $\{Z_L\}_L$ converges in expectation, namely, 
    \begin{align}
      \lim_{L \to \infty} \mathbb{E} \pp{Z_L} = \mathbb{E} \pp{n_1 e^{{\,\langle{n_1}, x_\ell \rangle}-\frac{\norm{x}_2^2}{2}}} = x_\ell \label{eqn:nominatorConvergenceForLargeL}.
    \end{align}
Using similar arguments, we get,
    \begin{align}
      \lim_{L\to\infty} \mathbb{E} \pp{p_1^{(\ell)}}= \lim_{L \to \infty} \mathbb{E} \pp{{ \frac{e^{{\,\langle{n_1}, x_\ell \rangle}}}{\frac{1}{L}\sum_{r=0}^{L-1} e^{{\,\langle{n_1}, x_r \rangle}}}}} = \mathbb{E} \pp{e^{{\,\langle{n_1}, x_\ell \rangle}-\frac{\norm{x}_2^2}{2}}}  = 1 \label{eqn:denominatorConvergenceForLargeL}.
    \end{align}
Thus, combining \eqref{eqn:nominatorConvergenceForLargeL} and \eqref{eqn:denominatorConvergenceForLargeL}, by the SLLN, we get,
    \begin{align}
       \lim_{M, L \to \infty} \hat{x}_\ell = \lim_{M,L \to \infty} \frac{Z_L}{p_1^{(\ell)}} = x_\ell.
    \end{align}
as claimed. It is left to prove Lemma \ref{lemma:5}.

\begin{proof}[Proof of Lemma \ref{lemma:5}]
Let us denote $Y_i = e^{{\,\langle{n_1}, x_{i} \rangle}}$. In order to apply Proposition \ref{prop:4} in our case, we need to show that the expectation of $Y_i$ is finite, its variance is uniformly bounded, and that the covariance decays to zero, as $L \to \infty$. For the expectation, we have,
        \begin{align}
           \mathbb{E}\pp{e^{{\,\langle{n_1}, x_{i} \rangle}}} = e^{\frac{\norm{x}_2^2}{2}} < \infty.
        \end{align}
The variance is given by,
    \begin{align}
       \s{Var}\pp{e^{{\,\langle{n_1}, x_{i} \rangle}}} = e^{2\norm{x}_2^2} - e^{\norm{x}_2^2} < \infty.
    \end{align}
Finally, for the covariance we have,
    \begin{align}
       \nonumber \s{cov} \p{e^{{\,\langle{n_1}, x_{\ell_1} \rangle}},e^{{\,\langle{n_1}, x_{\ell_2} \rangle}}} & = \mathbb{E}\pp{e^{{\,\langle{n_1}, x_{\ell_1} + x_{\ell_2} \rangle}}} - \mathbb{E}\pp{e^{{\,\langle{n_1}, x_{\ell_1} \rangle}}}\mathbb{E}\pp{e^{{\,\langle{n_1}, x_{\ell_2} \rangle}}} \\ 
       & = e^{{{\norm{x}_2^2 + {\,\langle{x_{\ell_1}, x_{\ell_2} \rangle}}}}} - e^{{{\norm{x}_2^2}}}.
    \end{align}
By assumption, we have ${\,\langle{x_{\ell_1}, x_{\ell_2} \rangle}} \to 0$, as $\abs{\ell_1 - \ell_2} \to \infty$. Thus,
    \begin{align}
       \nonumber \s{cov} \p{e^{{\,\langle{n_1}, x_{\ell_1} \rangle}},e^{{\,\langle{n_1}, x_{\ell_2} \rangle}}} \to 0.
    \end{align}
Therefore, all the assumptions of Proposition \ref{prop:4} are satisfied, which proves \eqref{eqn:lemma5}, as required.
\end{proof}

\section{Proof of Proposition~\ref{thm:lowerBoundAfterTiterations}}
First, we demonstrate that, under the assumptions of Proposition~\ref{thm:lowerBoundAfterTiterations}, the following holds,
\begin{align}
    \norm{\hat{x}_\ell^{(T)} - \hat{x}_\ell^{(0)}} \leq T \sqrt{2\epsilon}. \label{eqn:O1}
\end{align}
As $\norm{ \hat{x}_\ell^{(t)}} = 1$, for every $t \in \pp{T}$, it follows that
\begin{align}
    \norm{\hat{x}_\ell^{(t+1)} - \hat{x}_\ell^{(t)}}^2 = 2 - 2 \langle \hat{x}_\ell^{(t+1)}, \hat{x}_\ell^{(t)} \rangle. \label{eqn:O2}
\end{align}
By the assumption $\langle \hat{x}_\ell^{(t+1)}, \hat{x}_\ell^{(t)} \rangle \geq 1 - \epsilon$, it follows that
\begin{align}
    \norm{\hat{x}_\ell^{(t+1)} - \hat{x}_\ell^{(t)}}^2 \leq 2 - 2 \p{1-\epsilon} = 2\epsilon. \label{eqn:O3}
\end{align}
Applying a telescopic series, followed by the triangle inequality on the left-hand-size of \eqref{eqn:O1} results in,
\begin{align}
    \norm{\hat{x}_\ell^{(T)} - \hat{x}_\ell^{(0)}} = \norm{\sum_{t=0}^{T-1} \p{\hat{x}_\ell^{(t+1)} - \hat{x}_\ell^{(t)}}} \leq \sum_{t=0}^{T-1} \norm{\p{\hat{x}_\ell^{(t+1)} - \hat{x}_\ell^{(t)}}} \leq \sqrt{2\epsilon} T, \label{eqn:O4}
\end{align}
which proves \eqref{eqn:O1}. Therefore,
\begin{align}
    \langle \hat{x}_\ell^{(T)}, \hat{x}_\ell^{(0)} \rangle = 1 - \frac{\norm{\hat{x}_\ell^{(T)} - \hat{x}_\ell^{(0)}}^2}{2} \geq 1- \frac{(\sqrt{2 \epsilon}T)^2}{2} = 1-T^2 \epsilon, \label{eqn:O5}
\end{align}
completing the proof.

\section{Proof of Proposition \ref{thm:LequalsTwoMultiIteration}}
Fix $L=2$, and denote the initialization by $\hat{x}_\ell^{(0)} = x_\ell$ for $\ell = 0,1$, and $\rho = \langle x_0, x_1 \rangle$.
According to Theorem \ref{thm:hardAssignmentTwoHypoteses} and under the assumptions of Proposition \ref{thm:LequalsTwoMultiIteration}, after a single iteration, as $M \to \infty$, we obtain,
\begin{align}
    \hat{x}_0^{(1)} \rightarrow \sqrt{\frac{1}{\pi\p{1-\rho} \norm{\hat{x}^{(0)}}_2^2}} \p{\hat{x}_0^{(0)} - \hat{x}_1^{(0)}}, \label{eqn:P1}
\end{align}
and,
\begin{align}
    \hat{x}_1^{(1)} \rightarrow \sqrt{\frac{1}{\pi\p{1-\rho} \norm{\hat{x}^{(0)}}_2^2}} \p{\hat{x}_1^{(0)} - \hat{x}_0^{(0)}}. \label{eqn:P2}
\end{align}
In particular, we note that
\begin{align}
    \norm{\hat{x}_0^{(1)}} = \norm{\hat{x}_1^{(1)}} = \sqrt{\frac{2}{\pi}}, \label{eqn:P3}
\end{align}
and
\begin{align}
    \rho^{(1)} = \langle \hat{x}_0^{(1)}, \hat{x}_0^{(1)} \rangle = -1. \label{eqn:P4}
\end{align}
Applying an additional iteration of the hard-assignment procedure on $\hat{x}_0^{(1)}, \hat{x}_1^{(1)}$, by using  
Theorem \ref{thm:hardAssignmentTwoHypoteses} it follows that,
\begin{align}
    \hat{x}_0^{(2)} \rightarrow \sqrt{\frac{1}{\pi\p{1-\rho^{(1)}} \norm{\hat{x}^{(1)}}_2^2}} \p{\hat{x}_0^{(1)} - \hat{x}_1^{(1)}}.\label{eqn:P6}
\end{align}
Substituting \eqref{eqn:P3} and \eqref{eqn:P4} into the right-hand-side of \eqref{eqn:P6} results in,
\begin{align}
    \hat{x}_0^{(2)} \rightarrow  \hat{x}_0^{(1)}, \quad \hat{x}_1^{(2)} \rightarrow  \hat{x}_1^{(1)}.
\end{align}
Therefore, for $t \geq 1$, $\hat{x}_\ell^{(t+1)} \rightarrow \hat{x}_\ell^{(t)}$ a.s., for $\ell = 0,1$, which in light of \eqref{eqn:P1} and \eqref{eqn:P2}, completes the proof.

\section{Proof of Proposition \ref{thm:infiniteLmultiIteration}}
Denote by $\hat{x}_\ell^{(t)}$ the $t$-th iteration of soft-assignment procedure in Algorithm \ref{alg:generalizedEfNSoftMultiIteration}, and denote the initialization by $\hat{x}_\ell^{(0)}$. According to Theorem \ref{thm:softAssignmentAsymptoticL}, and under the assumptions of the proposition, after a single iteration, as $M, d, L \to \infty$,
    \begin{align}
       \hat{x}_\ell^{(1)} - \hat{x}_\ell^{(0)} \rightarrow 0,
    \end{align}
in probability, for every $\ell \in \mathbb{N}$. The same holds for the consecutive iteration, i.e., for $t \geq 0$, $\hat{x}_\ell^{(t+1)} \rightarrow \hat{x}_\ell^{(t)}$ in probability, for $\ell \in \mathbb{N}$. Therefore, by induction, we have,
\begin{align}
       \hat{x}_\ell^{(T)} - \hat{x}_\ell^{(0)} \rightarrow 0,
\end{align}
for every $T<\infty$, which completes the proof.

\end{appendices}

\end{document}